\documentclass{article}
\usepackage{mathrsfs}
\usepackage[utf8]{inputenc}
\usepackage[letterpaper,margin=1.0in]{geometry}

\usepackage{amsmath}\allowdisplaybreaks
\usepackage{amsfonts,bm}
\usepackage{amssymb}
\usepackage{algorithm}
\usepackage{algorithmic}
\usepackage{multirow}
\usepackage{booktabs}
\usepackage{xcolor}

\def\eqref#1{equation~(\ref{#1})}

\def\floor#1{\left\lfloor #1 \right\rfloor}
\def\1{\bf{1}}

\newcommand{\Norm}[1]{\left\| #1 \right\|}
\newcommand{\norm}[1]{\left\| #1 \right\|_2}

\def\fA{{\mathcal{A}}}
\def\fB{{\mathcal{B}}}

\def\fD{{\mathcal{D}}}

\def\fF{{\mathcal{F}}}

\def\fL{{\mathcal{L}}}
\def\fM{{\mathcal{M}}}
\def\fN{{\mathcal{N}}}
\def\fO{{\mathcal{O}}}

\def\fR{{\mathcal{R}}}

\def\BE{{\mathbb{E}}}

\def\BI{{\mathbb{I}}}

\def\BN{{\mathbb{N}}}
\def\BO{{\mathbb{O}}}
\def\BP{{\mathbb{P}}}

\def\BR{{\mathbb{R}}}

\DeclareMathOperator*{\argmax}{arg\,max}
\DeclareMathOperator*{\argmin}{arg\,min}

\usepackage{amsthm}
\theoremstyle{plain}
\newtheorem{thm}{Theorem}[section]
\newtheorem{dfn}{Definition}[section]

\newtheorem{lem}{Lemma}[section]
\newtheorem{asm}{Assumption}[section]

\makeatletter
\def\Ddots{\mathinner{\mkern1mu\raise\p@
\vbox{\kern7\p@\hbox{.}}\mkern2mu
\raise4\p@\hbox{.}\mkern2mu\raise7\p@\hbox{.}\mkern1mu}}
\makeatother

\makeatletter
\newcommand*{\rom}[1]{\expandafter\@slowromancap\romannumeral #1@}
\makeatother

\usepackage{color}
\usepackage{hyperref}
\usepackage{enumerate}
\usepackage{enumitem}
\usepackage{algorithm}
\usepackage{algorithmic}
\usepackage[numbers,sort,compress]{natbib}
\usepackage{url}
\usepackage{microtype}
\usepackage{hyperref}
\usepackage{url}
\usepackage{microtype}
\usepackage{enumitem}
\usepackage{graphicx}
\usepackage{bbm}

\usepackage{thmtools}
\usepackage{mathtools}
\usepackage{thm-restate}
\usepackage{nicefrac}
\usepackage{enumitem}
\usepackage{threeparttable}
\usepackage{makecell}

\title{Stochastic Bilevel Optimization with Heavy-Tailed Noise}

\date{}
\author{Zhuanghua Liu \quad\quad\quad Luo Luo}

\begin{document}

\maketitle

\begin{abstract}
This paper considers the smooth bilevel optimization in which the lower-level problem is strongly convex and the upper-level problem is possibly nonconvex.
We focus on the stochastic setting where the algorithm can access the unbiased stochastic gradient evaluation with heavy-tailed noise, which is prevalent in many machine learning applications, such as training large language models and reinforcement learning.
We propose a nested-loop normalized stochastic bilevel approximation (N$^2$SBA) for finding an $\epsilon$-stationary point with the stochastic first-order oracle (SFO) complexity of $\tilde{\mathcal{O}}\big(\kappa^{\frac{7p-3}{p-1}} \sigma^{\frac{p}{p-1}} \epsilon^{-\frac{4 p - 2}{p-1}}\big)$, where $\kappa$ is the condition number, $p\in(1,2]$ is the order of central moment for the noise, and $\sigma$ is the noise level.
Furthermore, we specialize our idea to solve the nonconvex-strongly-concave minimax optimization problem, achieving an $\epsilon$-stationary point with the SFO complexity of~$\tilde{\mathcal O}\big(\kappa^{\frac{2p-1}{p-1}} \sigma^{\frac{p}{p-1}} \epsilon^{-\frac{3p-2}{p-1}}\big)$.
All the above upper bounds match the best-known results under the special case of the bounded variance setting, i.e., $p=2$.
We also conduct the numerical experiments to show the empirical superiority of proposed methods.
\end{abstract}

\section{Introduction}

We consider the following smooth bilevel optimization problem
\begin{equation}\label{bilevel_obj}
\begin{split}
    \min_{x \in \BR^{d_x}} &\varphi(x) \coloneqq f(x, y^*(x)) ,\\
     {\rm s.t.}~&y^*(x) \coloneqq \argmin_{y \in \BR^{d_y}} g(x,y),
\end{split}
\end{equation}
where both the upper-level function and the lower-level function are jointly continuously differentiable and have the form of 
\begin{align}\label{bilevel_stochastic}
  f(x,y)  \coloneqq \BE[F(x, y; \xi)]  \qquad\text{and}\qquad g(x,y)  \coloneqq \BE[G(x, y; \zeta)]  
\end{align}
where $\xi\sim\fD_f$ and $\zeta\sim D_g$ are the stochastic indices. In particular, we focus on the nonconvex-strongly-convex setting, i.e., 
the lower-level function $g(x,y)$ is smooth and strongly convex with respect to $y$, while the upper-level function $f$ is smooth but possibly nonconvex.
This formulation is ubiquitous in a wide range of machine learning applications, including reinforcement learning~\citep{sutton1998introduction,li2017deep,hong2020two}, hyperparameter optimization~\citep{franceschi2018bilevel,pedregosa2016hyperparameter,feurer2019hyperparameter}, meta learning~\citep{finn2017model,fallah2020convergence,rajeswaran2019meta,ji2022theoretical}, and neural architecture search~\citep{liu2018darts,zoph2016neural,zhang2021idarts}.

There has been increasing interest in developing stochastic algorithms for large-scale bilevel optimization under the bounded variance assumption~\citep{hao2024bilevel,gong2024accelerated,gong2024nearly,ghadimi2018approximation,chen2021closing,chen2022single,hong2020two,khanduri2021near,ji2021bilevel,yang2021provably,yang2022decentralized,shen2023penalty,kwon2023fully,chen2025near,chen2025condition,liang2023lower,zhanggeneralized,huangoptimal,yaoovercoming}.
The early methods construct the estimator for the hyper-gradient of the objective~{\citep{ghadimi2018approximation,hong2020two,ji2021bilevel,huang2025efficiently,yang2022decentralized,wang2024efficient,gong2024accelerated,gong2024nearly,hao2024bilevel}}, 
which typically requires iterating with the (stochastic) Hessian-vector product (HVP) oracle to avoid directly computing the inverse Hessian.
However, accessing the HVP oracle may be expensive in many real-world applications \citep{sow2022convergence,song2019maml,finn2017model,nichol2018first}. 
To address this issue, the recent works introduce the penalized formulation for the bilevel problem to design the fully first-order methods, which efficiently achieves an approximate first-order stationary point without any second-order or HVP oracle calls  \citep{kwon2023fully,chen2025near,liu2022bome,shen2023penalty}.

It is worth noting that the bounded variance assumption used in existing analyses for stochastic bilevel optimization is restrictive and may not hold in practice.
The empirical studies on many real-world applications, such as image classification~\citep{simsekli2019tail}, large language model (LLM) training~\citep{zhang2020adaptive,ahn2023linear}, and reinforcement learning~\citep{garg2021proximal} have shown that the stochastic gradient typically follows the heavy-tailed distributions.
These observations motivate us to adopt the weaker $p$-th bounded central moment ($p$-BCM) assumption on the noise of the stochastic gradient, i.e., 
for all~$x\in\BR^{d_x}$ and $y\in\BR^{d_y}$, it holds that
\begin{align*}
    \BE[\Norm{\nabla F(x, y; \xi) - \nabla f(x, y)}^p] \leq \sigma_f^p
    \qquad \text{and} \qquad
    \BE[\Norm{\nabla G(x, y; \zeta) - \nabla g(x, y)}^p] \leq \sigma_g^p
\end{align*} 
for some $\sigma_f$, $\sigma_g > 0$ and $p \in (1, 2]$ is the order of central moment for the noise.
In the special case of $p = 2$, the $p$-BCM assumption reduces to the standard bounded variance assumption \citep{nemirovski2009robust,ghadimi2013stochastic}. 
In the case of $p \in (1, 2)$, the variance may be unbounded, leading to heavy-tailed noise distributions such as the L\'evy $\alpha$-stable distribution \citep{levy1925calcul,mandelbrot1960pareto}.
While the stochastic minimization problem under heavy-tailed has been extensively studied in recent years~\citep{gorbunov2020stochastic,gorbunov2022clipped,sadiev2023high,gorbunov2024high,zhang2020adaptive,nguyen2023high,cutkosky2020momentum,liu2023stochastic,hubler2024gradient,liunonconvex2025,sun2024gradient,he2025complexity}, 
the stochastic bilevel optimization remains unexplored beyond the bounded variance setting.

In this paper, we propose an efficient stochastic first-order method called nested-loop normalized stochastic bilevel approximation (N$^2$SBA) for nonconvex-strongly-convex bilevel optimization with heavy-tailed noise.
Our algorithm employs a nested-loop structure such that the outer loop can be regarded as inexact normalized stochastic gradient descent on the surrogate nonconvex function, 
while the inner loop solves the subproblems by clipped stochastic gradient descent.
We show that our N$^2$SBA can achieve an $\epsilon$-stationary point with the stochastic first-order oracle (SFO) complexity of $\tilde\fO\big(\kappa^{\frac{7p-3}{p-1}} \sigma^{\frac{p}{p-1}} \epsilon^{-\frac{4 p - 2}{p-1}}\big)$ in expectation, where $\kappa$ denotes the condition number and $\sigma = \max\{\sigma_f, \sigma_g\}$.
We further establish the high probability convergence guarantee for our method, matching the complexity of the in-expectation result up to the logarithmic factor. 
For the special case of the bounded variance setting (i.e., $p = 2$), 
our result matches the state-of-the-art SFO complexity of~$\tilde\fO(\kappa^{11} \sigma^{2} \epsilon^{-6})$ established by \citet{chen2025near}.
{ In addition, we provide a lower bound with condition number dependence for stochastic bilevel optimization under heavy-tailed noise, which indicates such problem is indeed more difficult than the single level one.}
We compare the result of the proposed N$^2$SBA with related work in Table~\ref{tbl:bilevel_res}.

\begin{table*}[t]
\caption{We compare our proposed N$^2$SBA with existing methods by showing the upper complexity bounds for achieving an $\epsilon$-stationary point in the stochastic nonconvex-strongly-convex bilevel optimization problem (\ref{bilevel_obj}), where $\kappa$ is the condition number  and~$\sigma \coloneqq \max\{\sigma_f, \sigma_g\}$. The second column indicates the oracles used in the algorithm, i.e., the stochastic Hessian-vector product  (SHVP) and the stochastic first-order oracle (SFO).} 
\label{tbl:bilevel_res}\vskip -0.05cm
\centering
\begin{threeparttable}
\begin{tabular}{cccccc}
\toprule
Methods  & Oracle  &  Complexity & $p$ 
\\
\midrule
\makecell{~~BSA \citep{ghadimi2018approximation}}~~&~~SFO/SHVP~~& $\fO\left(\dfrac{\kappa^9 \sigma^3}{\epsilon^{6}}\right)$ & 2
  \\\addlinespace
\makecell{TTSA \citep{hong2020two}\tnote{\S}}   &  SFO/SHVP & $\fO\left(\dfrac{{\rm poly}(\kappa,\sigma)}{\epsilon^{5}}\right)$ & 2  
  \\\addlinespace
  \makecell{stocBiO \citep{ji2021bilevel}}  &  SFO/SHVP & $\fO\left(\dfrac{\kappa^9 \sigma^2}{ \epsilon^{4}}\right)$ & 2  
  \\\addlinespace
   \makecell{F$^2$SA \citep{kwon2023fully}\tnote{\S}}  &  SFO & $\tilde{\fO}\left(\dfrac{{\rm poly}(\kappa, \sigma)}{ \epsilon^{7}}\right)$ & 2  
  \\\addlinespace
   \makecell{F$^2$BSA \citep{chen2025near}\tnote{\dag}}   &  SFO & $\tilde{\fO}\left(\dfrac{\kappa^{11} \sigma^2}{ \epsilon^{6}}\right)$ & 2  
  \\\addlinespace
   \makecell{N$^2$SBA (Theorem \ref{thm:bilevel_expect})}   & SFO &~~~$\tilde{\fO}\left(\dfrac{\kappa^{\frac{7p-3}{p-1}} \sigma^{\frac{p}{p-1}}} {\epsilon^{\frac{4 p - 2}{p-1}}}\right)$~~~&~~~(1, 2]~~~\\ \addlinespace
\bottomrule 
 \end{tabular}\vskip0.15cm
 \begin{tablenotes}
        \item[$\S$]  
        We use the notation ${\rm poly}(\kappa, \sigma)$ in the complexity bounds of TTSA and F$^2$BSA  since the dependency on the condition number and the noise level has not been explicitly presented by the analyses of \citet{hong2020two} and \citet{kwon2023fully}. \vskip0.1cm
        \item[$\dag$]  The original analysis of \citet[Theorem 4.2]{chen2025near} show that F$^2$SBA attains an $\epsilon$-stationary point with the SFO complexity of $\tilde{\fO}\left({\kappa^{12} \sigma^2}{ \epsilon^{-6}}\right)$. In fact, this result can be improved to $\tilde{\fO}\left({\kappa^{11} \sigma^2}{ \epsilon^{-6}}\right)$. Please refer to Appendix \ref{appendix:F2BDA} for details.
  \end{tablenotes}
\end{threeparttable}
\vskip -0.25cm
\end{table*}

We also study the stochastic nonconvex-strongly-concave minimax optimization problem
\begin{equation}\label{minimax_obj}
    \min_{x \in \BR^{d_x}} \max_{y \in \BR^{d_y}} f(x, y) \coloneqq \BE[F(x, y; \xi)],    
\end{equation}
which can be regarded as the special case of the bilevel problem (\ref{bilevel_obj}) by taking $G=-F$ in equation (\ref{bilevel_stochastic}).
This formulation covers various machine learning applications, including generative adversarial networks (GANs) \citep{goodfellow2020generative}, robust statistics \citep{xu2009robustness,shafieezadeh2015distributionally}, online learning \citep{cesa2006prediction}, and adversarial robustness~\citep{madry2017towards}.
Most existing studies on stochastic minimax optimization with heavy-tailed noise focus on the convex–concave setting~\citep{gorbunov2022clipped,sadiev2023high,gorbunov2023high}.
\citet{gorbunov2022clipped} and \citet{sadiev2023high} investigated problems under the more general star-cocoercive condition, which is different from our nonconvex–strongly-concave setting.
Following the idea of N$^2$SBA, we develop nested-loop normalized stochastic gradient descent ascent (N$^2$SGDA) for stochastic nonconvex-strongly-concave minimax optimization with heavy-tailed noise.
We show that N$^2$SGDA can achieve an $\epsilon$-stationary point of the function  $\Phi(x) \coloneqq \max_{y\in\BR^{d_y}}f(x, y)$ with the SFO complexity of $\tilde{\fO}\big(\kappa^{\frac{2p-1}{p-1}} \sigma^{\frac{p}{p-1}} \epsilon^{-\frac{3p-2}{p-1}}\big)$ in expectation (with the high probability). 
This result is near-optimal to $\sigma$ and $\epsilon$, since it matches the lower bound of $\Omega\big(\sigma^{\frac{p}{p-1}} \epsilon^{-\frac{3p-2}{p-1}}\big)$ up to logarithmic factor for finding an $\epsilon$-stationary point of stochastic nonconvex minimization problem under the $p$-BCM noise~\citep{zhang2020adaptive,liunonconvex2025}.
We compare the result of proposed N$^2$SBA with related work in Table~\ref{tbl:minimax_res}.

\begin{table*}[t]
\caption{
We compare our proposed N$^2$SGDA with existing methods by showing the upper complexity bounds for achieving an $\epsilon$-stationary point of $\Phi(x)\coloneqq\max_{y\in\BR^{d_y}}f(x,y)$ for the stochastic nonconvex-strongly-concave minimax optimization problem (\ref{minimax_obj}).
Here, we denote $\sigma\coloneqq \sigma_f$ for the ease of presentation.
} 
\label{tbl:minimax_res}\vskip-0.1cm
\centering
\begin{tabular}{cccccc}
\toprule
Methods  & ~~Oracle~~ & Complexity & 
$p$ 
\\
\midrule
\makecell{~~~~SGDmax \citep{jin2020local}~~~~} & SFO & $\tilde\fO\left(\dfrac{\kappa^3 \sigma^2}{\epsilon^{4}}\right)$ & 2
  \\\addlinespace
\makecell{SGDA \citep{lin2020gradient}} & SFO & $\fO\left(\dfrac{\kappa^3 \sigma^2}{\epsilon^{4}}\right)$ & 2
  \\\addlinespace
 N$^2$SGDA (Theorem~\ref{thm:minimax_high_prob}) & SFO  &  ~~~~$\tilde{\fO}\Bigg(\dfrac{\kappa^{\frac{2p-1}{p-1}}  \sigma^{\frac{p}{p-1}}}{\epsilon^{\frac{3p-2}{p-1}}} \Bigg)$~~~~ & ~~~(1,2]~~~   \\\addlinespace
\bottomrule    
\end{tabular}
\vskip -0.25cm
\end{table*}

\paragraph{Paper Organization} 
In Section \ref{sec:preliminaries}, we introduce preliminaries for our problem setting.
In Section \ref{sec:bilevel_opt}, we propose N$^2$SBA for solving stochastic nonconvex-strongly-convex bilevel optimization problem with heavy-tailed noise and establish both in-expectation and high-probability theortical guarantees.
In Section~\ref{sec:minimax_opt}, we propose the N$^2$SGDA method for solving stochastic nonconvex-strongly-concave minimax optimization problem that obtains the near-optimal SFO complexity with respect to the accuracy and noise level.
In Section \ref{sec:experiments}, we conduct numerical experiments to validate the effectiveness of our methods.
Finally, we conclude our work and outline several potential directions for future directions in Section \ref{sec:conclusion}.

\section{Preliminaries}\label{sec:preliminaries}


In this section, we formalize the problem setting for our stochastic bilevel optimization problem. 
We begin by introducing the standard assumptions on smoothness and lower boundedness, following \citep{ghadimi2018approximation,ji2021bilevel,chen2024finding}. 

\begin{asm} \label{asm:bilevel}
We suppose the bilevel optimization problem (\ref{bilevel_obj}) holds that
\begin{enumerate}[label=(\alph*)]
    \item the upper-level function $f(x,y)$ is $C_f$-Lipschitz in $y$, $L_f$-gradient Lipschitz, and twice continuously differentiable;
    \item the lower-level function $g(x,y)$ is $L_g$-gradient Lipschitz, $\rho_g$-Hessian Lipschitz, and $\mu$-strongly convex in $y$;
    \item the objective $\varphi(x)$ is lower bounded, i.e., $\inf_{x \in \BR^{d_x}} \varphi(x) > - \infty$.
\end{enumerate}
\end{asm}

Under Assumption \ref{asm:bilevel}, we define largest smoothness constant and the condition number of the problem (\ref{bilevel_obj}) as follows.

\begin{dfn}  \label{dfn:kappa-bilevel}
Under Assumption \ref{asm:bilevel}, we define the largest smoothness constant and the condition number of problem (\ref{bilevel_obj}) as $\ell: =\max \{ C_f,L_f,L_g,\rho_g  \}$ and $\kappa:= \ell / \mu$, respectively. 
\end{dfn}

Assumption \ref{asm:bilevel} also indicates the hyper-gradient $\nabla \varphi(x)$ has the following closed-form expression \citep{ghadimi2018approximation}.

\begin{restatable}[{\citet[Lemma 2.2]{ghadimi2018approximation}}]{prop}{propFsmooth}
    \label{prop:F-smooth}
Under Assumption \ref{asm:bilevel}, the function $\varphi$ in problem (\ref{bilevel_obj}) is differentiable and 
the hyper-gradient $\nabla \varphi(x)$ has the closed-form expression  
\begin{equation*}
    \nabla \varphi(x) = \nabla_x f(x, y^*(x)) - \nabla^2_{xy} g(x, y^*(x)) (\nabla^2_{yy} g(x, y^*(x)))^{-1} \nabla_y f(x, y^*(x)),
\end{equation*}
where $y^*(x)=\argmin_{y \in \BR^{d_y}} g(x,y)$ is uniquely defined.
Furthermore, the hyper-gradient~$\nabla \varphi(x)$ is $L_\varphi$-Lipschitz continuous with $L_\varphi = \fO(\kappa^3\ell)$.
\end{restatable}

We target to find the $\epsilon$-stationary point of $\varphi$, which is defined as follows.

\begin{dfn} \label{dfn:sta-hyper}
Given a differentiable function $\varphi(x): \BR^d \rightarrow \BR $,
we call the point $\hat x\in\BR^d$ is an $\epsilon$-stationary point of $\varphi(x)$ if it holds $ \Vert \nabla \varphi(\hat x) \Vert \le \epsilon$.
\end{dfn}

We impose the $p$-th bounded central moment ($p$-BCM) assumption  \citep{zhang2020adaptive} on the gradient noise for the upper-level and lower-level functions.

\begin{asm}\label{asm:pBCM}
    For given $x\in\BR^{d_x}$, $y\in\BR^{d_y}$,
    we suppose the algorithm can draw the stochastic indices $\xi\sim\fD_f$ and $\zeta\sim\fD_g$ and access the stochastic first-order oracle $\nabla F(x,y;\xi)$ and $\nabla G(x,y;\zeta)$ such that
    \begin{equation*}
        \BE[\nabla F(x,y;\xi)] = \nabla f(x,y), ~~~ \BE[\Norm{\nabla F(x,y;\xi) - \nabla f(x,y)}^p] \leq \sigma_f^p,
    \end{equation*}
    \begin{equation*}
        \BE[\nabla G(x,y;\zeta)] = \nabla g(x,y), ~~~ \BE[\Norm{\nabla G(x,y;\zeta) - \nabla g(x,y)}^p] \leq \sigma_g^p,
    \end{equation*}
    where $p \in (1,2]$ and $\sigma_f, \sigma_g > 0$.
\end{asm}

Additionally, we introduce the notation $\sigma \coloneqq \max\{\sigma_f, \sigma_g\}$ for the stochastic bilevel optimization problem~(\ref{bilevel_obj}).
We remark that Assumption~\ref{asm:pBCM} with $p = 2$ aligns with the classical bounded variance condition, while the variance of $\nabla F(x, y; \xi)$ and~$\nabla G(x, y; \zeta)$ are possibly unbounded in the case of~$p \in (1,2)$ \citep{levy1925calcul,mandelbrot1960pareto}.

\begin{algorithm}[t]
\caption{Nested-loop Normalized Stochastic Bilevel Approximation (N$^2$SBA)}
\label{alg:nsgd_bilevel}
\begin{algorithmic}[1] 
\STATE \textbf{Input:} initial points $x_0\in\BR^{d_x}$, $\hat {y}_{0,0} = \hat{z}_{0, 0} \in \BR^{d_y}$; penalized parameter $\lambda>0$; 
stepsizes $\eta_x>0, \{\eta_{y,t}>0\}_{t=0}^{T-1}, \{\eta_{z,t}>0\}_{t=0}^{T-1}$;  
sequences of clipping parameters $\{\tau_{t,k}\}_{k=0}^{K-1}$, $\{\tau_{t,k}'\}_{k=0}^{K-1}$; iteration numbers $T,K>0$; batch size $M>0$  \\[1mm]
\STATE \textbf{for} $t = 0, \dots, T-1$ \\[1mm]
\STATE \quad \label{line:inner-start}\textbf{for} $k = 0, \dots, K-1$ \\[1mm]
\STATE  \quad \quad   draw $\xi'_t\sim\fD_f$,~~$\zeta'_t\sim\fD_g$,~~and~~$\zeta_t\sim\fD_g$  \\[1mm]
\STATE  \quad \quad   \label{bilevel_y_start} $g_{y, t, k} = \nabla_y F(x_t,\hat{y}_{t, k}; \xi'_t) + \lambda\nabla_y G(x_{t}, \hat{y}_{t,k}; \zeta'_t)$ \\[1mm]
\STATE  \quad \quad   \label{bilevel_y_end} $\hat{y}_{t,k+1} =  \hat{y}_{t,k} - \eta_{y,t} {\rm clip}(g_{y, t, k},\tau_{t,k}')$ \\[1mm]
\STATE  \quad \quad  \label{bilevel_z_start} $g_{z, t,k} = \lambda \nabla_y G (x_{t}, \hat{z}_{t, k}; \zeta_t)$ \\[1mm]
\STATE  \quad \quad  \label{bilevel_z_end} $\hat{z}_{t, k+ 1} = \hat{z}_{t, k} - \eta_{z,t}{\rm clip}(g_{z, t, k}, \tau_{t,k})$\\[1mm]
\STATE \quad \label{line:inner-end}\textbf{end for} \\[1mm]
\STATE  \quad \label{line:y_t} $y_{t} = \hat{y}_{t,K}$,~~~ $\hat{y}_{t+1,0} = \hat{y}_{t,K}$ \\[1mm]
\STATE  \quad \label{line:z_t} $z_{t} = \hat{z}_{t,K}$,~~~ $\hat{z}_{t+1,0} = \hat{z}_{t,K}$ \\[1mm]
\STATE  \quad  \label{line:outer-start}  draw $\xi_{t,i}\sim \fD_f$~~and~~$\zeta_{t,i}\sim \fD_g$ for all $i\in[M]$ \\[1mm]
\STATE  \quad  $\displaystyle{g_{x, t} = \frac{1}{M}\sum_{i=1}^M} \left(\nabla_x F (x_{t}, y_{t}; \xi_{t,i}) + \lambda (\nabla_x G(x_t, y_t; \zeta_{t,i}) - \nabla_x G(x_t, z_t; \zeta_{t,i}))\right)$ \\[1mm]
\STATE  \quad  \label{line:outer-end} $x_{t+1} = x_{t} - \eta_{x}\cdot\dfrac{g_{x, t}}{\Norm{g_{x, t}}}$
\STATE \textbf{end for} \\[1mm]
\STATE draw $\hat{x}$ from $\{x_t\}_{t=1}^T$ uniformly \\[1mm]
\STATE \textbf{Return:} $\hat{x}$ \\[1mm]
\end{algorithmic}
\end{algorithm}

\section{Nested-loop Normalized Stochastic Bilevel Approximation}\label{sec:bilevel_opt}

We propose nested-loop normalized stochastic bilevel approximation (N$^2$SBA) in Algorithm~\ref{alg:nsgd_bilevel}, where we define the clipping operator as
\begin{align*}
    {\rm clip}(g, \tau) =
    \begin{cases}
        \min\left\{1, \dfrac{\tau}{\Norm{g}}\right\}g, & g\neq0, \\
        0, & g=0,
    \end{cases}
\end{align*}
for $g\in\BR^{d_y}$ and $\tau>0$. The remainder of this section first introduces the intuition of our algorithm design, then provides the convergence guarantees both in expectation and with the high probability.

\subsection{The Algorithm Design}\label{sec:bilevel_method}


We introduce the design of our N$^2$SBA by starting with the penalized reformulation 
\citep{kwon2023fully} for problem (\ref{bilevel_obj}), i.e.,
\begin{align}  \label{eq:opt-xyz}
       \min_{x \in \BR^{d_x}} \fL_{\lambda}^*(x)\coloneqq \min_{y \in \BR^{d_y}} \left\{f(x,y) + \lambda \left(g(x,y) - \min_{z \in \BR^{d_y}} g(x,z)\right)\right\}.
\end{align}
It is worth noting that the gradient of $\fL_{\lambda}^*$ has the closed-form expression that consists of the first-order information of $f$ and $g$, i.e.,
\begin{equation*}
    \nabla \fL_{\lambda}^*(x) = \nabla_x f(x, y_\lambda^*(x)) + \lambda (\nabla_x g(x, y_\lambda^*(x)) - \nabla_x g(x, y^*(x))),
\end{equation*}
where $y_{\lambda}^*(x)\coloneqq\argmin_{y\in\BR^{d_y}} f(x,y)+\lambda (g(x,y)-g^*(x))$ {and $g^*(x) = g(x, y^*(x))$}.
In addition, taking
$\lambda = \Theta(\kappa^3\ell\epsilon^{-1})$ 
guarantees that any $\epsilon$-stationary point of the function $\fL_{\lambda}^*$ corresponds to an $\fO(\epsilon)$-stationary point of the objective $\varphi$ \citep[Lemma 3.1]{kwon2023fully}.
Therefore, it is natural to solve the bilevel optimization problem (\ref{bilevel_obj}) by using first-order methods to address the surrogate problem (\ref{eq:opt-xyz}).
However, existing stochastic first-order bilevel optimization methods are based on the framework of stochastic gradient descent \citep{kwon2023fully,chen2025near,liu2022bome,shen2023penalty}, whose convergence guarantees require the assumption of bounded variance that does not hold in our $p$-BCM setting \citep[Remark 1]{zhang2020adaptive}.

In our proposed N$^2$SBA, we address the heavy-tailed noise by solving the surrogate problem~(\ref{eq:opt-xyz}) in the framework of inexact normalized stochastic gradient descent. 
Specifically, we first solve the sub-problems
\begin{align}\label{eq:sub-problems}
   \min_{y\in\BR^{d_x}} f(x_t,y) + \lambda g(x_t,y)
   \qquad\text{and}\qquad
    \min_{z\in\BR^{d_x}} g(x_t,z)
\end{align} 
to approximate $y_\lambda^*(x_t)$ and $y^*(x_t)$ by $y_t$ and $z_t$, respectively.
In the views of $y_t\approx y_\lambda^*(x_t)$ and~$z_t\approx y^*(x_t)$,
we then estimate $\nabla \fL_\lambda^*(x_t)$ by the mini-batch stochastic gradient of the function $\fL_\lambda(x,y,z)\coloneqq f(x,y) + \lambda \left(g(x,y) - g(x,z)\right)$ with respect to $x$, i.e.,
\begin{align*}
\displaystyle{g_{x, t} = \frac{1}{M}\sum_{i=1}^M} \left(\nabla_x F (x_{t}, y_{t}; \xi_{t,i}) + \lambda (\nabla_x G(x_t, y_t; \zeta_{t,i}) - \nabla_x G(x_t, z_t; \zeta_{t,i}))\right),    
\end{align*}
where $\xi_{t,i}\sim \fD_f$ and $\zeta_{t,i}\sim \fD_g$ for all $i\in[M]$.
Note that the $p$-BCM condition on the stochastic gradients of $f$ and $g$ (Assumption \ref{asm:pBCM}) implies the stochastic gradients of the objectives in sub-problems (\ref{eq:sub-problems}) and the function $\fL_\lambda$ also has the $p$-th bounded central moments, which typically requires the steps of clipping or normalization to guarantee the convergence of the first-order methods \citep{zhang2020adaptive,sadiev2023high,hubler2024gradient}.
Therefore, our N$^2$SBA (Algorithm \ref{alg:nsgd_bilevel}) solves the strongly-convex sub-problems~(\ref{eq:sub-problems}) by performing the clipped stochastic gradient descent (Line \ref{line:inner-start}--\ref{line:inner-end}) and minimizing the nonconvex function $\fL_\lambda^*$ by iterating with its inexact normalized stochastic gradient $g_{x,t}/\Norm{g_{x,t}}$ (Line~\ref{line:outer-start}--\ref{line:outer-end}).

\subsection{Convergence Analysis}\label{sec:bilevel_analysis}

In this subsection, we present the complexity analysis of our N$^2$SBA for solving the stochastic nonconvex-strongly-convex bilevel optimization problem with heavy-tailed noise. We use notations $\fO(\cdot)$, $\Theta(\cdot)$, $\tilde\fO(\cdot)$, and $\tilde\Theta(\cdot)$ to simply the presentation of some parameter settings, and the corresponding explicit expressions can be found in appendix.

As discussed in Section~\ref{sec:bilevel_method}, our objective is to identify an $\epsilon$-stationary point of $\fL_{\lambda}^*(x)$ with~$\lambda = \Theta(\kappa^3\ell\epsilon^{-1})$, which corresponds to an $\fO(\epsilon)$-stationary point of $\varphi(x)$.
Hence, we first establish the in-expectation result for our algorithm, which starts from the following key lemma with respect to the function $\nabla \fL_\lambda^*(x)$.

\begin{lem}\label{lemma:n2sba_one_iter}
    Under Assumptions \ref{asm:bilevel} and \ref{asm:pBCM},  running N$^2$SBA (Algorithm \ref{alg:nsgd_bilevel}) with $\lambda \geq 2 L_f / \mu$ holds that
    \begin{equation}
    \begin{split}
         \frac{1}{T}\sum_{t=0}^{T-1}  \BE[\Norm{\nabla \fL_{\lambda}^*(x_t)}] 
          \leq & \frac{\BE[\fL_{\lambda}^*(x_0) - \fL_{\lambda}^*(x_T)]}{\eta_x T} +  \frac{4 \sigma_f + 8 \lambda \sigma_g}{M^{\frac{p-1}{p}}} + \sum_{t=0}^{T-1} \frac{4 \lambda L_g \BE[\Norm{y_t - y_{\lambda}^*(x_t)}]}{T} \\
          & + \sum_{t=0}^{T-1} \frac{2 \lambda L_g \BE[\Norm{z_t - y^*(x_t)}]}{T} + \frac{D_3 \eta_x}{2},
    \end{split}
    \label{eq:key_descent}
    \end{equation}
    where $D_3=\fO(\ell \kappa^3)$.
\end{lem}
In order to find $\fO(\epsilon)$-stationary point of $\fL_{\lambda}^*(x)$, we have to bound the terms $\BE[\Norm{y_t - y_\lambda^*(x_t)}]$ and $\BE[\Norm{z_t - y^*(x_t)}]$ on the right-hand side of inequality \eqref{eq:key_descent}.
Recall that the points~$y_\lambda^*(x_t)$ and $y^*(x_t)$ are the solutions of the convex minimization problems in equation~(\ref{eq:sub-problems}).
Therefore, we introduce the following result for using the clipped stochastic gradient descent to solve the stochastic convex problem with heavy-tailed noise.

\begin{lem}
Assume the function ${h}(y) = \BE [{H}(y; \nu)]$ is $\ell_h$-smooth and $\mu_h$-strongly convex with the stochastic index $\nu\sim\fD_h$, and the stochastic gradient estimator $\nabla H(y; \nu)$ satisfies the conditions~$\BE[\nabla H(y; \nu)] = \nabla h(y)$
and $\BE[\Norm{\nabla {H}(y; \nu) - \nabla h(y)}^p] \leq \sigma_h^p$ for some $\sigma_h>0$. 
We run the clipped stochastic gradient descent 
    \begin{equation}\label{alg:sgd}
    \begin{cases}    
        \nu_k \sim \fD_h \\
        \hat{y}_{k+1} = \hat{y}_{k} - \eta_y {\rm clip}(\nabla H(\hat{y}_k; \nu_k), \tau_k)
    \end{cases}
    \end{equation}
    for $k = 0, 1, \dots, K$ with %
    \begin{align*}\small
    \begin{split}
        \tau_k = {\tilde\Theta}\left(\frac{\hat{R} \exp(-\eta_y\mu_h (1+k/2))}{ \eta_y}\right),~\eta_y = {\tilde\Theta}\left(\min \left\{\frac{1}{ \ell_h}, \frac{\ln B_K}{\mu_h (K+1)}\right\}\right), ~B_K={\tilde\fO}\left(1+\frac{\mu_h^2 K^{ \frac{2 (p - 1)}{p}} \hat{R}^2 }{\sigma_h^{2}} \right).
    \end{split}
    \end{align*}
Then for all $\hat{R}^2 \geq \Norm{\hat{y}_{0} - y^*}^2$,
we have
\begin{align*}
    \BE[\Norm{\hat{y}_{K} - y^*}^2] \leq  \tilde\fO \left(  \hat{R}^2 \exp\left( -\frac{K\mu_h }{\ell_h }\right)+ \frac{\sigma_h^2 }{\mu_h^2 K^{ \frac{2 (p - 1)}{p}}}{ + \frac{\hat{R}^2}{K^2}} \right).
\end{align*}
\label{lemma:y_recur_strongly} 
\end{lem}
By appropriate parameter settings, we can apply Lemma \ref{lemma:y_recur_strongly} to show that both
$\BE\left[\Norm{{y}_{t} - y_\lambda^*(x_{t})}\right]$     
and $\BE\left[\Norm{{z}_{t} - y^*(x_{t})}\right]$
are upper bounded by $\fO\left(\epsilon^2/(\ell^2 \kappa^3)\right)$ for all $t=0,\dots,T-1$ (Lemma~\ref{lemma:yz-distance} in Appendix \ref{appendix:thm3}).
Combining with the result of Lemma \ref{lemma:n2sba_one_iter}, we derive the following in-expectation convergence of our N$^2$SBA for stochastic bilevel optimization problem.
\begin{thm}
    Under Assumption \ref{asm:bilevel} and \ref{asm:pBCM}, we run Algorithm \ref{alg:nsgd_bilevel} (N$^2$SBA) with 
    \begin{align*}
        & \eta_x = \frac{\epsilon}{\ell \kappa^3}, ~~~ \eta_{y,t} = \tilde{\Theta} \left(\min\left\{\frac{1}{ (L_f + \lambda L_g)}, \frac{1}{\lambda \mu K}  \right\}\right), ~~~ \eta_{z,t} = \tilde{\Theta}\left(\min\left\{\frac{1}{ \lambda L_g}, \frac{1}{\lambda \mu K}  \right\}\right), \\
       & T = \fO\left(\frac{\Delta \ell \kappa^3}{\epsilon^2}\right),~~~M = \fO\left(\frac{\ell^{\frac{p}{p-1}} \kappa^{\frac{3p}{p-1}}\sigma^{\frac{p}{p-1}}}{\epsilon^{\frac{2p}{p-1}}}\right), ~~~ K = \tilde\fO\left(\frac{\ell^{\frac{p}{p-1}}  \kappa^{\frac{4p}{p-1}} \sigma^{\frac{p}{p-1}}}{\epsilon^{\frac{2p}{p-1}}} { + \frac{\ell^2 \kappa^3 R_y}{\epsilon^2} }\right), \\
       & \tau_{t,k} = \Tilde{\Theta} \left(\frac{\exp(- \eta_{z,t} \lambda \mu  k )R_z}{ \eta_{z,t}}\right),~\tau_{t,k}' = \Tilde{\Theta}\left(\frac{\exp(- \eta_{y,t} \lambda \mu k )R_y}{ \eta_{y,t}}\right), ~\lambda = \max\left\{\frac{ \kappa}{ R_0}, \frac{\ell \kappa^2}{ \Delta}, \frac{\ell \kappa^3}{ \epsilon}\right\},
    \end{align*}
    where we set
    \begin{align*}
       & R_y^2 \coloneqq \fO\left(R_0^2 + \frac{ \epsilon^4}{\ell^4 \kappa^6} + \frac{ \epsilon^2}{\ell^2 \kappa^4}\right),~~~ R_{z}^2 \coloneqq \fO\left( R_0^2 +  \frac{ \epsilon^4}{\ell^4 \kappa^6} + \frac{\epsilon^2}{\ell^2 \kappa^4} \right)\\
        & R_0^2 \geq \Norm{\hat{y}_{0,0} - y^*(x_0)}^2,~~~\Delta = \varphi(x_0) - \inf_{x \in \BR^{d_x}} \varphi(x).
    \end{align*}
    Then we have $\BE[\Norm{\nabla \varphi(\hat{x})}] = \fO( \epsilon)$. In particular, the overall SFO calls is upper bounded by 
    \begin{align*}
        \tilde\fO\left(\frac{\Delta\ell^{\frac{2p-1}{p-1}} \kappa^{\frac{7p-3}{p-1}}\sigma^{\frac{p}{p-1}}}{\epsilon^{\frac{4p-2}{p-1}}} {+ \frac{\Delta \ell^3 \kappa^6 R_0}{\epsilon^4} }\right).
    \end{align*}
    \label{thm:bilevel_expect}
\end{thm}
It is worth noting that our analysis is carried out under the general $p$-BCM noise condition (Assumption~\ref{asm:pBCM}). 
In the special case of $p=2$, it reduces to the standard bounded-variance setting, then Theorem~\ref{thm:bilevel_expect} yields an SFO complexity of $\tilde{\fO}(\kappa^{11}\sigma^2\epsilon^{-6})$ which aligns with the result of the best-known fully first-order methods~\citep{chen2025near}.   

The above in-expectation results guarantee the small gradient norm by taking the average on the outputs of running N$^2$SBA sufficient times,
while the results of the high-probability guarantee for a single run are more desirable in practice.
The following theorem establishes a high-probability guarantee for our N$^2$SBA algorithm.

\begin{thm}
    Under Assumption \ref{asm:bilevel} and \ref{asm:pBCM}, we run Algorithm~\ref{alg:nsgd_bilevel} (N$^2$SBA) with
\begin{align*}
        & \eta_x = \sqrt{\frac{\Delta}{\ell \kappa^3 T}},~~~\eta_{y,t} = \tilde{\Theta} \left(\min \left\{\frac{1}{L_f + \lambda L_g}, \frac{1}{\lambda \mu K}\right\} \right),\eta_{z, t} = \Tilde{\Theta}\left(\min \left\{\frac{1}{ \lambda L_g }, \frac{1}{\lambda \mu K}\right\}\right), \\ 
        &M = \fO\left(\frac{\ell^{\frac{p}{p-1}} \sigma^{\frac{p}{p-1}} \kappa^{\frac{3p}{p-1}}}{\epsilon^{\frac{2p}{p-1}}}\right),~~~ K = \tilde\fO\left(\frac{\ell^{\frac{p}{p-1}} \kappa^{\frac{4p}{p-1}} \sigma^{\frac{p}{p-1}}}{\epsilon^{\frac{2p}{p-1}}}\right),  ~~~\lambda = \max\left\{\frac{ \kappa}{ \sqrt{R}}, \frac{\ell \kappa^2}{ \Delta}, \frac{\ell \kappa^3}{ \epsilon}\right\},\\
        & \tau_{t,k} = \Tilde{\Theta}\left(\frac{\exp(- \eta_{z,t} \lambda \mu k ){R_z}}{ \eta_{z,t} }\right), ~~~\tau_{t,k}' = \tilde{\Theta}\left(\frac{\exp(- \eta_{y,t} \lambda \mu  k ){R_y}}{120 \eta_{y,t}}\right), ~~~T = \tilde\fO\left(\frac{\Delta \ell \kappa^3}{\epsilon^2}\right),
    \end{align*}
    where $R_y^2$, $R_z^2$, $R_0^2$, and $\Delta$ follow the settings of Theorem \ref{thm:bilevel_expect}.
    Then we have {$\frac{1}{T}\sum_{t=0}^{T-1} \Norm{ \nabla \varphi(x_t)}= \fO( \epsilon)$} with probability at least $1 - \delta$ for all $\delta\in(0,1)$.
    In particular, the overall SFO calls is upper bounded by\footnote{The notations $\tilde\Theta(\cdot)$ and $\tilde\fO(\cdot)$ in parameters setting hide the logarithmic factors including $\log(1/\delta)$.} 
    \begin{align*}
        \tilde\fO\left(\dfrac{\Delta\ell^{\frac{2p-1}{p-1}} \kappa^{\frac{7p-3}{p-1}} \sigma^{\frac{p}{p-1}}}{\epsilon^{\frac{4p-2}{p-1}}}\right).
    \end{align*}
    \label{thm:bilevel_high_prob}
\end{thm}
We remark that our SFO complexity for the high-probability result in Theorem~\ref{thm:bilevel_high_prob} matches that of the in-expectation counterpart in Theorem~\ref{thm:bilevel_expect}, up to the logarithmic factor. 
It also nearly matches the best-known in-expectation SFO complexity bounds under the bounded variance assumption (i.e., Assumption~\ref{asm:pBCM} with $p=2$) \citep{chen2025near}.

{
\subsection{The Lower Complexity Bound}\label{sec:lower_bound_bilevel}

In this subsection, we provide a lower bound for solving smooth nonconvex-strongly-convex bilevel optimization problems.
We use $\fF^{\text{nc-sc}}(C_f, L_f, L_g, \rho_g, \mu, \Delta)$ to denote the set of all bilevel optimization problems satisfying Assumption \ref{asm:bilevel} and 
\begin{align*}
    \varphi(x_0) - \inf_{x \in \BR^{d_x}}\varphi(x) \leq \Delta,
\end{align*}
where $x_0\in\BR^{d_x}$ is the initial point of the algorithm.

In our lower bound analyses, our construction focus on bilevel problems with quadratic lower-level functions, which are formally defined as follows.

\begin{dfn}
    Given $C_f, L_f, L_g > 0$, $\mu \in (0, L_g]$ and $\Delta > 0$, we denote bilevel problems with quadratic lower-level functions as $\fF^{\textnormal{nc-scq}}(C_f, L_f, L_g, \mu, \Delta)$, which is a subset of nonconvex-strongly-convex problems $\fF^{\textnormal{nc-sc}}(C_f, L_f, L_g, 0, \mu, \Delta)$.
    In particular, the lower-level function of a problem in $\fF^{\textnormal{nc-scq}}(C_f, L_f, L_g, \mu, \Delta)$ takes the form of
    \begin{align*}
        g(x, y) = \frac{1}{2} y^\top H y + x^\top J y + b^\top y,
    \end{align*}
    where $\mu I \preceq H \preceq L_g I$ and $- L_g I \preceq J \preceq L_g I$.
\end{dfn}
We then introduce the class of stochastic first-order oracle algorithms as follows.

\begin{dfn}\label{dfn:random_algo}
    A stochastic first-order algorithm $A$ consists of a distribution $\BP_r$ over a measurable set $\fR$ and a sequence of measurable mappings $\{A_t\}_{t \in \BN}$ such that $A_t$ takes in the first $t - 1$ oracle responses and the random seed $r \in \fR$ to produce the $t$-th query $w_t$, such that $w_0 = (0,0)$ and
    \begin{align}
        w_t = A_t(r, \BO^{\rm sfo}(w_0), \dots, \BO^{\rm sfo}(w_{t-1})), \quad \forall t \in \BN_+,
        \label{rand_algo}
    \end{align}
    where $\BO^{\rm sfo}(x, y) = (\hat{\nabla} F(x,y; \xi), \hat{\nabla} G(x,y;\zeta))$ satisfying the $p$-BCM condition in Assumption~\ref{asm:pBCM}.
    We denote $\fA^{\rm sfo}$ as the set of all randomized algorithms that follow the protocol (\ref{rand_algo}).
\end{dfn}

We follow the construction of \citet{chen2025condition} to establish the following lower bound for stochastic bilevel optimization under heavy-tailed noise.

\begin{thm}
    Given $L_f, L_g,\mu,\sigma,\Delta,C_f>0$ such that 
    $\mu \in (0, L_g]$ and $C_f \geq \sqrt{13 L_f \Delta}$,  
    the bilevel problem instance $(f, g) \in \fF^{\textnormal{nc-scq}}(C_f, L_f, L_g, \mu, \Delta)$ equipped with a stochastic first-order oracle (SFO) satisfying the $p$-BCM noise in Assumption~\ref{asm:pBCM}, such that finding an expected $\epsilon$-stationary point of the hyper-objective $\varphi(x) = f(x, z^*(x))$  with $\epsilon \leq \min(\sigma \kappa_g/92, \kappa_g \sqrt{L_f \Delta}/ 312)$, any stochastic first-order algorithm $A \in \fA^{\rm sfo}$ requires at least
    \begin{align}\label{eq:lower-bilevel}
        \Omega\left(  \frac{\Delta L_f \sigma^{\frac{p}{p-1}} \kappa_g^{\frac{3p-2}{p-1}}}{\epsilon^{\frac{3p-2}{p-1}}}\right)
    \end{align}
    SFO calls, where $\kappa_g = L_g / \mu$ is the condition number of the lower level problem.
    \label{thm:bilevel_lower_bound}
\end{thm}}

Note that the optimal SFO complexity for finding an $\epsilon$-stationary point in the single-level stochastic optimization problem with heavy-tailed noise is
$\fO(\Delta L_f \sigma^{\frac{p}{p-1}} \epsilon^{-\frac{3p-2}{p-1}})$ {\citep{liunonconvex2025}}.
Therefore, the additional dependence on $\kappa_g$ in our SFO lower bound (\ref{eq:lower-bilevel}) indicates the bilevel problem is indeed more challenging.
In the special case of $p=2$, the lower bound~(\ref{eq:lower-bilevel}) reduces to $\Omega(\Delta L_f \sigma^{2} \kappa_g^{4}\epsilon^{-4})$, matching the state-of-the-art SFO lower bound for stochastic bilevel optimization problems under the bounded-variance setting \citep{chen2025condition}.

\section{Application to Nonconvex-Strongly-Concave Minimax Optimization}\label{sec:minimax_opt}

In this section, we consider the stochastic nonconvex-strongly-concave minimax problem
\begin{equation}\label{eq:minimax}
    \min_{x \in \BR^{d_x}} \max_{y \in \BR^{d_y}} f(x, y) \coloneqq \BE[F(x, y; \xi)],    
\end{equation}
where $\xi\sim\fD_f$ is the stochastic index and the stochastic gradient estimate $\nabla F(x,y;\xi)$ satisfies the $p$-BCM assumption.
Concretely, we impose the following assumptions on problem (\ref{eq:minimax}).

\begin{asm}
    We suppose minimax optimization problem (\ref{eq:minimax}) holds that
    \begin{enumerate}[label=(\alph*)]
    \item the objective function $f(x,y)$ is $\ell$-smooth and $\mu$-strongly concave in $y$;
    \item the primal function $\Phi(\cdot) = \max_{y \in \BR^{d_y}} f(\cdot, y)$ is lower bounded, i.e., $\inf_{x \in \BR^{d_x}} \Phi(x) > -\infty$;
    \item for given $x\in\BR^{d_x}$ and $y\in\BR^{d_y}$, the algorithm can draw the stochastic indices $\xi\sim\fD_f$ and access the stochastic first-order oracle $\nabla F(x,y;\xi)$ such that
    $\BE[\nabla F(x,y;\xi)] = \nabla f(x,y)$ and 
    $\BE[\Norm{\nabla F(x,y;\xi) - \nabla f(x,y)}^p] \leq \sigma^p$,
    where $p \in (1,2]$ and $\sigma > 0$.
    \end{enumerate}
    \label{asm:minimax}
\end{asm}

\begin{algorithm}[t]
\caption{Nested-loop Normalized Stochastic Gradient Descent Ascent (N$^2$SGDA)}
\label{alg:nsgd_minimax}
\begin{algorithmic}[1] 
\STATE \textbf{Input:} initial points $x_0\in\BR^{d_x}, \hat {y}_{0,0}  \in \BR^{d_y}$, stepsizes $\eta_x>0, \{\eta_{y,t}>0\}_{t=0}^{T-1}$, sequences of clipping parameters~$\{\tau_{t,k}\}_{k=0}^{K-1}$, iteration numbers $T,K>0$, batch size $M>0$  \\[1mm]
\STATE \textbf{for} $t = 0, \dots, T-1$ \\[1mm]
\STATE \quad \textbf{for} $k = 0, \dots, K-1$ \\[1mm]
\STATE  \quad \quad  \label{inner_iter_start} draw $\xi_t'\sim\fD_f$\\[1mm]
\STATE  \quad \quad  $g_{y, t,k} =  \nabla_y F (x_{t}, \hat{y}_{t, k}; \xi_t')$ \\[1mm]
\STATE  \quad \quad  \label{inner_iter_end} $\hat{y}_{t, k+ 1} = \hat{y}_{t, k} + \eta_{y,t}{\rm clip}(g_{y, t, k}, \tau_{t,k})$\\[1mm]
\STATE \quad \textbf{end for} \\[1mm]
\STATE  \quad \label{minimax_update_y} $y_{t} = \hat{y}_{t,K}$,~~~ $\hat{y}_{t+1,0} = \hat{y}_{t,K}$ \\[1mm]
\STATE  \quad   draw $\xi_{t,i}\sim \fD_f$ for all $i\in[M]$  \\[1mm]
\STATE  \quad  $\displaystyle{g_{x, t} = \frac{1}{M}\sum_{i=1}^M}\nabla_x F (x_{t}, y_{t}; \xi_{t,i})$ \\[1mm]
\STATE  \quad \label{line:minimax_outer-end} $x_{t+1} = x_{t} - \eta_{x}\cdot\dfrac{g_{x, t}}{\Norm{g_{x, t}}}$
\STATE \textbf{end for} \\[1mm]
\STATE draw $\hat{x}$ from $\{x_t\}_{t=1}^T$ uniformly \\[1mm]
\STATE \textbf{Return:} $\hat{x}$ \\[1mm]
\end{algorithmic}
\end{algorithm}

It is worth noting that minimax problem~(\ref{eq:minimax}) is a special case of bilevel problem~(\ref{bilevel_obj}) by taking $G=-F$.
Therefore, we can directly solve minimax problem (\ref{eq:minimax}) by applying N$^2$SBA (Algorithm~\ref{alg:nsgd_bilevel}) with $\lambda=0$ and omitting the update on the variable $\hat z_{t,k}$.
We present this strategy in Algorithm \ref{alg:nsgd_minimax}, 
which is called Nested-loop Normalized Stochastic Gradient Descent Ascent (N$^2$SGDA). 

It is natural to applying the general theoretical results in Section \ref{sec:bilevel_opt} to analyze our N$^2$SGDA for minimax optimization, while this cannot achieve the tight complexity bounds.
Therefore, we need to specialize the structure of minimax problem(\ref{eq:minimax}) to refine the complexity bounds.
We first consider the properties of function $\Phi(\cdot)$.

\begin{lem}[{\citet[Lemma 4.3]{lin2020gradient}}]
Under Assumption \ref{asm:minimax},  the function $\Phi(\cdot)$ is~$(\ell + \kappa \ell)$-smooth and has the form of 
$\nabla \Phi(\cdot) = \nabla_{x} f(\cdot, y^*(\cdot))$,
where $\kappa \coloneqq \ell / \mu$ is the condition number and $y^*(\cdot) = \argmax_{y \in \BR^{d_y}} f(\cdot, y)$. 
In addition, the mapping $y^*(\cdot)$ is $\kappa$-Lipschitz.
\label{lemma:nonconvex_strongly_concave_prop}
\end{lem}
We then establish the descent lemma for function $\Phi(\cdot)$ as follows.
\begin{lem} Under Assumption \ref{asm:minimax}, running N$^2$SGDA (Algorithm \ref{alg:nsgd_minimax}) holds that
 \begin{align*}
         \frac{1}{T}\sum_{t=0}^{T-1}\BE[\Norm{\nabla \Phi(x_{t})}] \leq & \frac{\BE[\Phi(x_0) - \Phi(x_{T})]}{\eta_x T} + \frac{4 \sigma}{M^{\frac{p-1}{p}}} + \frac{2\ell \sum_{t=0}^{T-1} \BE[\Norm{ y_{t} - y^*(x_t)}]}{T} + \eta_{x} (\kappa + 1)  \ell.
    \end{align*}    
    \label{lemma:core_recur}
\end{lem}
Similar to the analysis for the bilevel problem, the iterations (line \ref{inner_iter_start}-\ref{inner_iter_end}) on $\hat y_{t,k}$ can be regarded as solving the strongly convex problem $\min_{y\in\BR^d} -f(x_t,y)$.
Hence, we combine Lemmas~\ref{lemma:y_recur_strongly} and \ref{lemma:core_recur} to achieve the in-expectation guarantee for N$^2$SGDA as follows.

\begin{thm}    \label{thm:minimax_expectation}
    Under Assumption \ref{asm:minimax}, we run Algorithm \ref{alg:nsgd_minimax} (N$^2$SGDA) with 
    \begin{align*}
       & \eta_x =  \frac{\epsilon}{\kappa \ell}, ~~~ \eta_{y,t} = \tilde{\Theta}\left( \min \left\{\frac{1}{\ell}, \frac{1}{\mu K}\right\}\right),~~~\tau_{t,k} = \Tilde{\Theta}\left(\frac{\exp(- \eta_{y,t} \mu k )R_y}{ \eta_{y,t}}\right),\\
        & T =  \fO\left(\frac{\Delta_\Phi \kappa \ell}{\epsilon^2}\right), ~~~  M = \left(\frac{\sigma}{\epsilon}\right)^{\frac{p}{p-1}}, ~~~ K = \tilde{\fO}\left(\kappa + \left( \frac{\ell^2 \sigma^2 }{\mu^2 \epsilon^2} \right)^{\frac{p}{2(p-1)}} {+ \frac{R_y \ell}{\epsilon}} \right),
    \end{align*}
    where we set $\Delta_\Phi = \Phi(x_0) - \min_{x \in \BR^{d_x}} \Phi(x)$, $R_y^2 = \fO(R_0^2 + { \epsilon^2}/{\ell^2})$,
    and $R_0^2 \geq \Norm{\hat{y}_{0,0} - y^*(x_0)}^2$.
    Then we have
   $\BE[\Norm{\nabla \Phi(\hat{x})}] \leq \epsilon$.    
    In particular, the overall SFO complexity is upper bounded by
    \begin{align*}
   & \tilde\fO\left( \dfrac{\Delta_\Phi \kappa^{\frac{2p-1}{p-1}} \ell \sigma^{\frac{p}{p-1}}}{\epsilon^\frac{3p-2}{p-1}}  + \dfrac{\Delta_\Phi \kappa^2 \ell}{\epsilon^2} {+ \frac{R_0 \Delta_\Phi \kappa \ell^2}{\epsilon^3}} \right).
    \end{align*}
\end{thm}
Note that the upper bound of N$^2$SGDA (Theorem \ref{thm:minimax_expectation}) is significantly better than that of N$^2$SBA (Theorem \ref{thm:bilevel_expect}).
Moreover, the dependency on $\sigma$ and $\epsilon$ in the result of Theorem \ref{thm:minimax_expectation} matches the lower bound of $\Omega\big(\sigma^{\frac{p}{p-1}}\epsilon^{-\frac{3p-2}{p-1}}\big)$  for finding the $\epsilon$-stationary point of the general nonconvex function under the $p$-BCM noise~\citep{liunonconvex2025}.
This also implies our upper bound is near-optimal to $\sigma$ and $\epsilon$.
Specifically, we consider the instance of minimax problem in which the objective function takes the form
\begin{equation*}
    f(x, y) = g(x) + h(y)
\end{equation*}
where $g(x)$ is the hard instance for nonconvex optimization under the $p$-BCM~\citep{liunonconvex2025} and $h(y)$ is a smooth and $\mu$-strongly concave function.
Consequently, identifying an 
$\epsilon$-stationary point of function~$\Phi(x)$ can be reduced to finding an $\fO(\epsilon)$-stationary point of the nonconvex function~$g(x)$, which requires at least the SFO complexity of $\Omega\big(\sigma^{\frac{p}{p-1}}\epsilon^{-\frac{3p-2}{p-1}}\big)$ \citep{liunonconvex2025}.
However, the optimality with respect to $\ell$ and $\mu$ is still an open problem even in the special case of $p=2$ \citep{li2021complexity}.


We now provide the high-probability guarantee for our N$^2$SGDA method.
\begin{thm}    \label{thm:minimax_high_prob}
    Under Assumption \ref{asm:minimax}, we run Algorithm \ref{alg:nsgd_minimax} (N$^2$SGDA) with 
    \begin{align*}
       & \eta_x =  \sqrt{\frac{\Delta_\Phi}{(\kappa + 1) \ell T}}, ~~~ \eta_{y,t} = \Tilde{\Theta}\left( \min \left\{\frac{1}{ \ell }, \frac{1}{\mu K}\right\}\right),~~~ \tau_{t,k} = \Tilde{\Theta}\left(\frac{\exp(- \eta_{y,t} \mu k )R_y}{ \eta_{y,t} }\right),\\
       &T =  \Tilde\fO\left(\frac{\Delta_\Phi \kappa \ell}{\epsilon^2}\right), ~~~~
         M = \left(\frac{\sigma}{\epsilon}\right)^{\frac{p}{p-1}}, ~~~~ K = \tilde{\fO}\left(\kappa + \left( \frac{\ell^2 \sigma^2 }{\mu^2 \epsilon^2} \right)^{\frac{p}{2(p-1)}} \right),
    \end{align*}
    where we set $\Delta_\Phi = \Phi(x_0) - \min_{x \in \BR^{d_x}} \Phi(x)$, $R_y^2 = R_0^2 + {10 \epsilon^2}/{\ell^2}$,
    and $R_0^2 \geq \Norm{\hat{y}_{0,0} - y^*(x_0)}^2$.
    Then we have $\sum_{t=0}^{T-1}\norm{\nabla\Phi({x}_t)}/T  \leq \epsilon$ with probability at least $1 - \delta$.    
    In particular, the overall SFO complexity is upper bounded by
    \begin{align*}
   & \tilde\fO\left( \frac{\Delta_\Phi \kappa^{\frac{2p-1}{p-1}} \ell \sigma^{\frac{p}{p-1}}}{\epsilon^\frac{3p-2}{p-1}}  + \frac{\Delta_\Phi \kappa^2 \ell}{\epsilon^2} \right).
    \end{align*}
\end{thm}
We remark that the SFO complexity in Theorem \ref{thm:minimax_high_prob} matches its in-expectation counterpart in Theorem \ref{thm:minimax_expectation}, up to a multiplicative logarithmic factor.
In addition, its dependency on $\delta$ and $\epsilon$ also nearly matches the lower bound.

\section{Experiments}\label{sec:experiments}

In this section, we conduct numerical experiments to validate the superiority of proposed algorithms for solving stochastic minimax and bilevel optimization problems.

\subsection{A Synthetic Problem}

We consider the synthetic nonconvex-strongly-concave minimax optimization problem
\begin{equation}\label{eq:prob:syntetic}
    \min_{x \in \BR^{d_x}}~~\max_{y \in \BR} f(x,y)=m_1 \left[ \norm{x}^2 + \sin(3 \sqrt{\Norm{x}^2 + 1})\right] + x^{\top} K y - m_2 \Norm{y}^2,
\end{equation}
where $K\in\BR^{d\times d}$ and $m_1, m_2 > 0$.
Following the setup of \citet{laguel2023high,laguel2024high},
we set $d_x = 30$ and $m_1=m_2=1$.
For the interaction matrix, we set $K = 10 \tilde{K} / \big\|\tilde{K}\big\|$, where~$\tilde{K} = (M + M^{\top})/2$ and $M\in\BR^{d \times d}$ has entries independently drawn from $\fN(0,1)$.
In~experiments, we employ the stochastic gradient estimate as $\nabla F(x, y; \xi) = \nabla f(x, y) + \xi$ for given $x\in\BR^{d_x}$ and $y\in\BR$, where $\xi = \hat{\xi} - \BE[\hat{\xi}]$ and $\hat{\xi}$ is drawn from the Pareto distribution with the shape parameter $\alpha\in(1,2]$.
This construction ensures that the estimate $F(x, y; \xi)$ satisfies our Assumption \ref{asm:minimax}.

We compare our N$^2$SGDA (Algorithm \ref{alg:nsgd_minimax}) against baseline methods including stochastic gradient descent ascent (SGDA) \citep{lin2020gradient} and stochastic gradient descent with a max-oracle (SGDmax) \citep{jin2020local,nouiehed2019solving}. 
The stepsizes for all algorithms are tuned from $\{1 \times 10^{-6}, 3 \times 10^{-6}, 1 \times 10^{-5} \ldots, 1 \}$.
For N$^2$SGDA and SGDmax, we fix the number of inner iterations be 10.
We set the total iteration count to $T = 5{,}000$ for N$^2$SGDA, Clipped-SGDmax, and SGDmax, and $T = 50{,}000$ for SGDA, so that all methods incur the same total number of stochastic first-order oracle calls. 
We perform 20 independent runs for each algorithm and present the averaged results.

We present the performance of all methods under Pareto noise with $\alpha\in\{1.2,1.5,1.8\}$ in Figure~\ref{fig:pareto_result}. 
We observe the proposed N$^2$SGDA converges faster than baseline methods SGDmax and SGDA, and its advantage is especially evident in the cases of $\alpha=1.2$ and $\alpha=1.5$.
Furthermore, the baselines SGDmax and SGDA exhibit pronounced oscillations in their trajectories, 
while the proposed N$^2$SGDA demonstrates markedly more stable convergence behavior under the heavy-tailed noise.

\begin{figure*}[t]
\centering
\begin{tabular}{ccc}
     \includegraphics[scale=0.4]{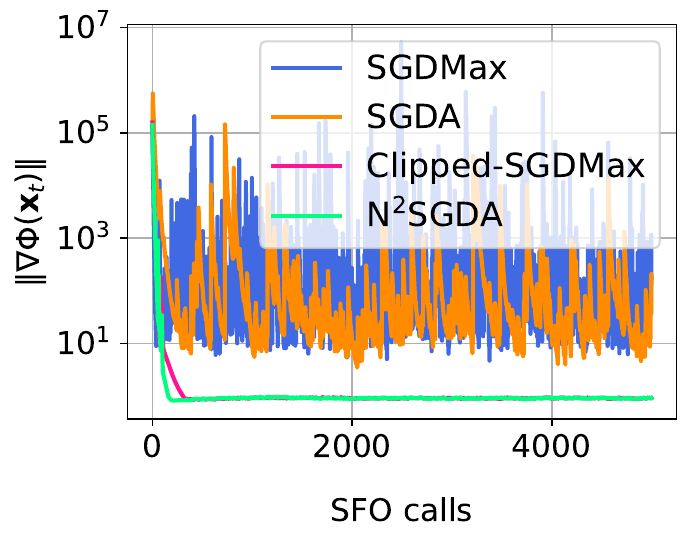}&
      \includegraphics[scale=0.4]{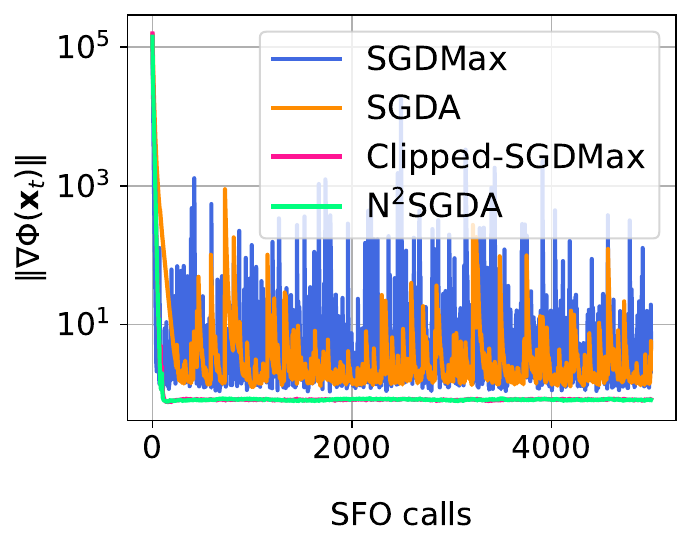}&
      \includegraphics[scale=0.4]{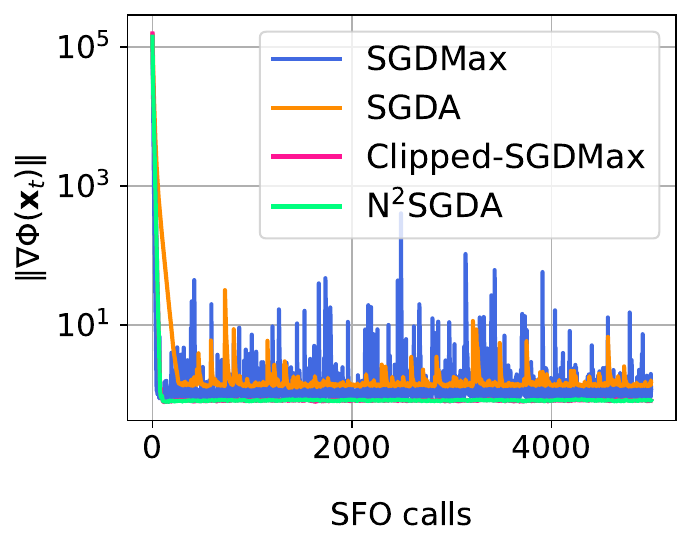} \\
      (a) $\alpha = 1.2$ & (b) $\alpha = 1.5$ & (c) $\alpha = 1.8$
\end{tabular}
\caption{We compare results of the number of SFO calls against $\Norm{\nabla \Phi(x)}$ on a synthetic nonconvex–strongly-concave minimax problem (\ref{eq:prob:syntetic}) with Pareto noise for the shape parameter $\alpha\in\{1.2, 1.5, 1.8\}$. The shaded region represents the standard deviation of the method.}
\label{fig:pareto_result} 
\end{figure*}

\subsection{Learnable Regularization}

We consider the problem of learnable regularization in logistic regression~\citep{grazzi2020iteration,ji2021bilevel,chen2025near}, which is formulated by the bilevel optimization problem,
\begin{align} \label{eq:l2reg}
\begin{split}
    & \min_{x \in \BR^p}  \varphi(x)\coloneqq \left\{ \frac{1}{\vert \fD^{\rm val} \vert} \sum_{(a_i,b_i) \in \fD^{\rm val}} \ell(\, \langle a_i,y^*(x) \rangle\,,b_i) \right\}, \\
    & \text{s.t~~} y^*(x) = \arg \min_{y \in \BR^{p \times c}}  \left\{ \frac{1}{\vert \fD^{\rm tr} \vert} \sum_{(a_i,b_i) \in \fD^{\rm tr}} 
\ell \left( \,\langle a_i,y \rangle\,,b_i \right) + {\rm tr} (y^\top \Sigma(x) \,y) \right\},
\end{split}
\end{align}
where  $\fD^{\rm tr}$ is the training dataset and $\fD^{\rm val}$ is the validation dataset, $a_i \in \BR^p$ is the data feature, $b_i \in \{0, 1, \dots, c-1\}$ is the corresponding class label, $c$ is the number of classes, $\Sigma(x):={\rm diag}(\exp(x_1),\dots,\exp(x_p))$ specifies the regularizers on each feature, and $\ell(\cdot, \cdot)$ denotes the cross entropy function, i.e., 
\begin{align*}
    \ell(Z, y) = - \sum_{i=1}^c \BI(i=y)\log \left( \frac{\exp(Z_i)}{\sum_{j = 1}^c \exp(Z_j)}\right),
\end{align*}
We conduct our experiments on dataset ``20 newsgroup'' that consists of  18,000 documents with 20 classes. Each document is represented as a $p$-dimensional vector ($p = 101{,}631$), where each of its coordinates corresponds to the term frequency–inverse document frequency of the word.

\begin{figure*}
    \centering
    \includegraphics[scale=0.34]{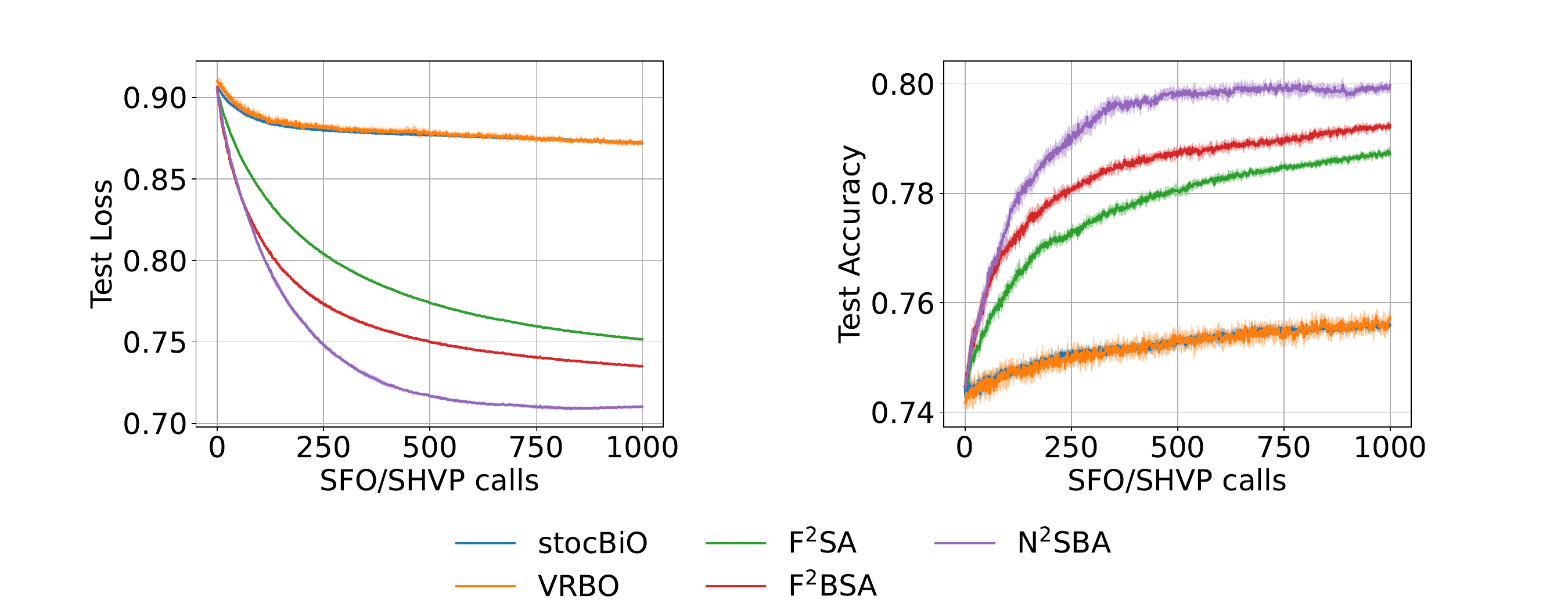} 
    \vskip-0.25cm
\caption{We present the results of stochastic oracle calls (SFO/SHVP) against test loss and test accuracy on bilevel problem (\ref{eq:l2reg}) for the problem of learnable regularization in logistic regression. The shaded region represents the standard deviation of the method.}
\label{fig:l2reg}
\end{figure*}

We compare our N$^2$SBA method with the stochastic Hessian-vector product-based methods stocBiO \citep{ji2021bilevel} and VRBO \citep{yang2021provably}, as well as the stochastic first-order methods F$^2$SA \citep{kwon2023fully} and F$^2$BSA \citep{chen2025near}.
For all algorithms, we set the number of inner iterations be 10 and the number of outer iterations be 1,000. 
In addition, we tune the Lagrangian parameter $\lambda$ from $\{10^1, 10^2, 10^3, 10^4\}$ for F$^2$SA, F$^2$BSA, and N$^2$SBA. 
The stepsizes of all algorithms are tuned from $\{10^{-5}, 10^{-4}., \ldots, 10^3\}$. 
We run the experiment with 10 different random seeds and present the empirical results in Figure~\ref{fig:l2reg}.
Clearly, our proposed N${}^2$SBA method outperforms all baseline methods. 

\subsection{Data Hyper-Cleaning}

We consider the application of data hyper-cleaning on multiple data sources, which is formulated by the bilevel optimization problem  
\begin{align}
\begin{split} \label{eq:cleaning}
    \min_{x \in \BR^m} & \varphi(x)\coloneqq \ell_{\rm val}(y^*(x)), \\
    \text{s.t~~}&  y^*(x)  = \argmin_{y \in \BR^p} \sum_{i=1}^m \sigma(x_i) \ell_{\rm tr}^i(y),
\end{split}
\end{align}
where $\ell_{\rm val}(\cdot)$ is the validation loss, $\ell_{\rm tr}^i(\cdot)$ is the training loss on the $i$th data source, $m$ is the number of data sources, and
$\sigma(x_i) = {\exp(x_i)}/{\sum_{j=1}^m \exp(x_j)}$ is the softmax function.
We perform our experiment on training the GPT-2 model with 124M parameters~\citep{radford2019language} on dataset ``Alpaca'' dataset~\citep{taori2023stanford}, which contains 52,000 instructions paired with demonstration outputs generated by OpenAI's ``text-davinci-003'' engine from~$m=10,000$ data sources. 
Following the setup of \citet{chen2025near}, we split the dataset into training and validation sets in an 8:2 ratio, then introduce corruption into the training set at a proportion of $q\in(0,1)$, i.e., 
implement the corruption by replacing the demonstration outputs with empty strings.



\begin{figure*}[t]
\centering
\begin{tabular}{ccc}
\!\!\!\!\!\includegraphics[scale=0.31]{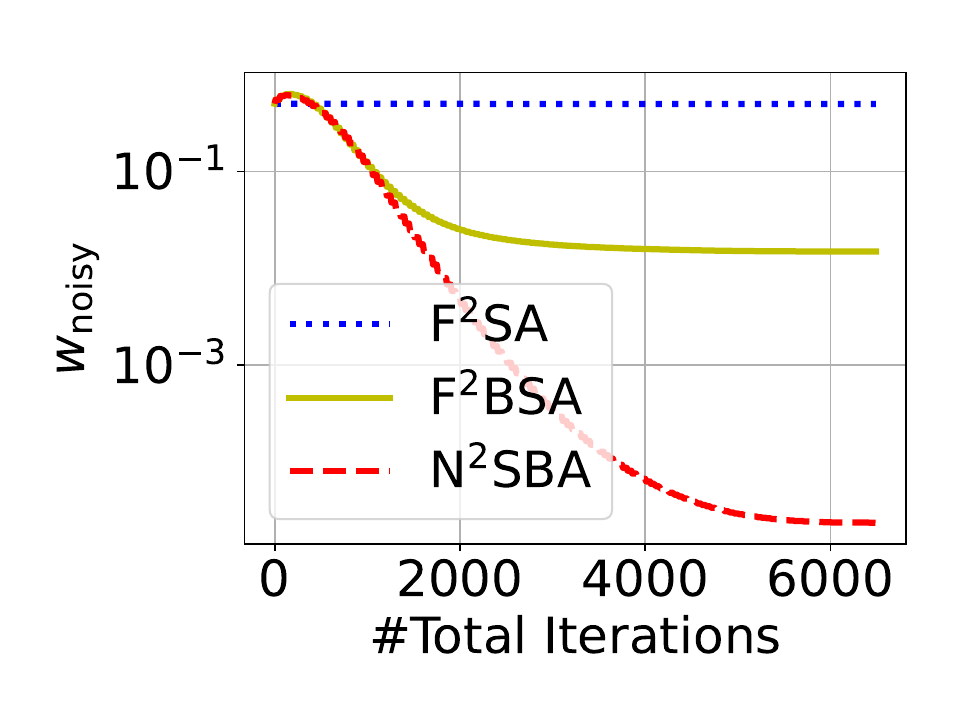} &\!\!\!\!\!\includegraphics[scale=0.31]{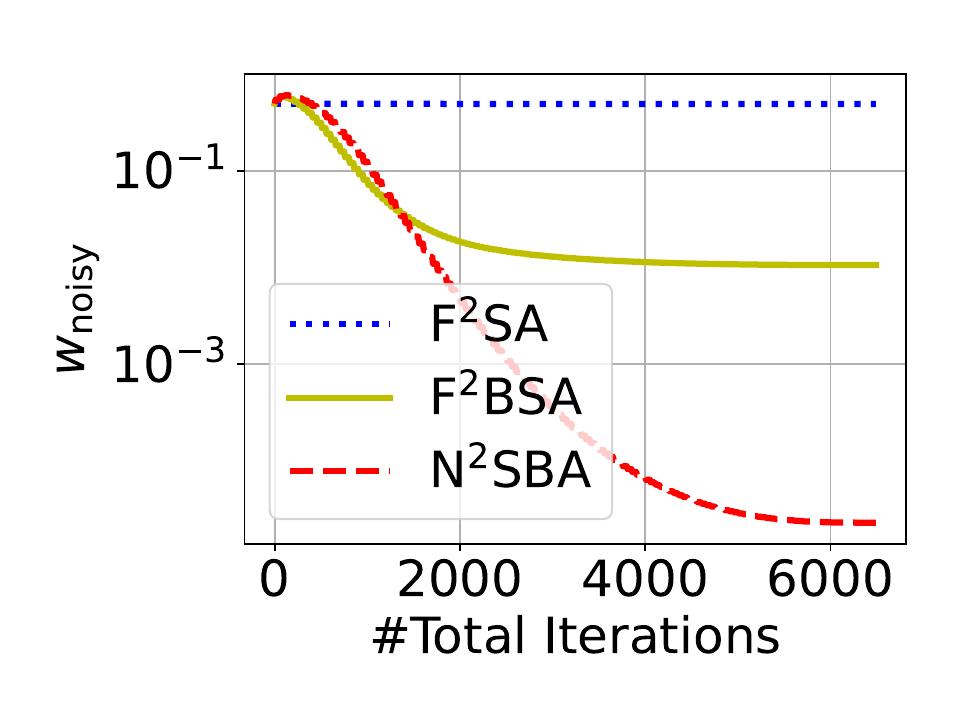}\!\!\!\!\!& \includegraphics[scale=0.31]{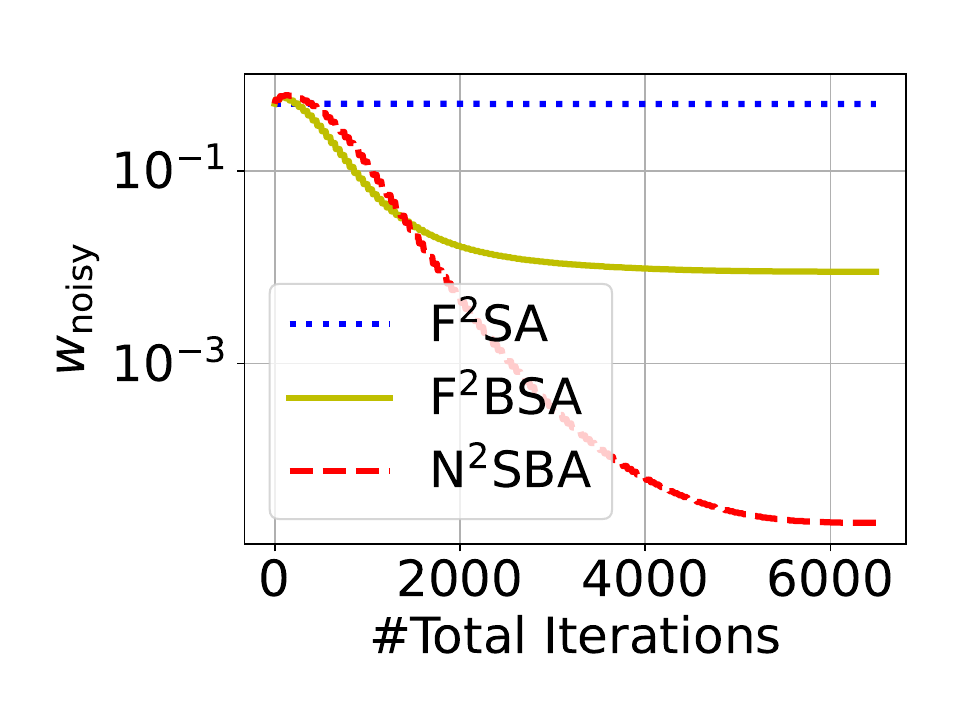} \\
(a)  $p=0.5$     &  (b) $p=0.9$ & (c) $p=0.99$
\label{fig:cleaning}
\end{tabular}
\caption{We present the results of SFO calls against the weights on garbage for the data hyper-cleaning problem (\ref{eq:cleaning}) under different corruption ratios $q\in\{0.5, 0.9, 0.99\}$.  The shaded region represents the standard deviation of the method.}
\end{figure*}


We run our experiment on a system of 8 GPUs.
We compare our proposed method N${}^2$SBA with F${}^2$SA~\citep{kwon2023fully} and F${}^2$BSA \citep{chen2025near}. 
Note that we do not include the SHVP-based methods since they are difficult to implement in multi-GPU training~\citep{tarzanagh2022fednest,chen2023decentralized,chen2025near}. For all algorithms, we use a batch size of 64 and set a fixed penalty parameter $\lambda = 10^3$.
We run our experiment with 10 different random seeds and present the empirical results for different corruption ratios $q=0.5, 0.9, 0.99$ in Figure~\ref{fig:cleaning}. 
We observe that our proposed N${}^2$SBA methods converge significantly faster than F${}^2$SA and F${}^2$BSA,
which aligns with the empirical findings in the single-level nonconvex problem \citet{zhang2020adaptive} that the performance of SGD-based methods deteriorates in the presence of heavy-tailed noise, 
whereas the strategies of gradient clipping and normalization can mitigate this effect and thereby improve the convergence.

\section{Conclusion and Future Work}\label{sec:conclusion}

In this work, we propose a stochastic first-order method called N$^2$SBA for solving stochastic nonconvex–\-strongly-convex bilevel optimization problems in the presence of heavy-tailed noise. 
We establish both in-expectation and high-probability theoretical guarantees for N$^2$SBA under the $p$-BCM gradient noise assumption. 
We prove that the stochastic first-order (SFO) complexity of our method matches that of the best-known stochastic first-order bilevel optimization algorithms under the specific bounded-variance assumption. 
Furthermore, we introduce N$^2$SGDA for addressing stochastic nonconvex–\-strongly-concave minimax optimization problems with heavy-tailed noise, achieving the SFO complexity with the nearly optimal dependence on noise level and accuracy. 

In future work, it is interesting to incorporate acceleration techniques \citep{sadiev2023high,he2025complexity,fradin2025tight} 
to further improve the dependence on condition numbers in upper complexity bounds.
Another promising direction is to simplify our algorithms for bilevel and minimax optimization problems, e.g., establish the single-loop methods.
In addition, it would be valuable to extend our methodology to constrained bilevel and minimax optimization problems.

\newpage
\appendix

\section{Supporting Lemmas} \label{sec:supporting_lemmas}

In this section, we revisit some technical lemmas that are essential for establishing the convergence guarantee of the proposed algorithms.

\begin{lem}[{\citet[Lemma 7]{hubler2024gradient}}]
    For all $a, b \in \BR^d$ such that $b \neq 0$, it holds that
    $a^{\top} b/\Norm{b} \geq \Norm{a} - 2 \Norm{a - b}$.
    \label{lemma:lower_bound}
\end{lem}

\begin{lem}[{\citet[Lemma 10]{hubler2024gradient}}]
    Let $p \in (1, 2]$ and $X_1, \dots, X_n \in \BR^d$ be a martingale difference sequence, i.e., $\BE[X_j \mid X_{j-1}, \dots, X_1] = 0$ a.s. and $\BE[\Norm{X_j}^p] < +\infty$
    for all $j= 1, \dots, n$. 
    Define $S_n \coloneqq \sum_{j=1}^n X_j$, then we have $\BE[\Norm{S_n}^p] \leq 2 \sum_{j=1}^n \BE[\Norm{X_j}^p]$.
    \label{lemma:mds_bound}
\end{lem}

\begin{lem}[{\citet[Example 4.1.7]{durrett2019probability}}]
    Let $(\Omega, \fF, \BP)$ be a probability space and $X$, $Y$ be independent random variables mapping to measurable spaces $(E_1, \Sigma_1)$ and $(E_2, \Sigma_2)$ respectively. 
    Furthermore, let $h: E_1 \times E_2 \to \BR^d$ be a (Lebesgue-)measurable function with $\BE[\Norm{h(X, Y)}] < \infty$.
    Then we have $\BE[h(X, Y) \mid X] \overset{a.s.}{=} g(X)$, where $g(\cdot) \coloneqq \BE[h(\cdot, Y)]$.
    \label{lemma:random_variable_mapping}
\end{lem}
The following lemma provides an upper bound on the norm of the difference between a minibatch gradient estimator and its expectation.
\begin{lem}
Assume that $\nabla f(x, y) = \BE[\nabla F(x, y; \xi)] $ and $\BE[\Norm{\nabla F(x, y; \xi) - \nabla f(x, y)}^p] \leq \sigma^p$.
Let $g = \frac{1}{M} \sum_{i=1}^M \nabla F(x, y; \xi_i)$, then it holds that $\BE[\Norm{g - \nabla f(x, y)}] \leq {2 \sigma}/{M^{\frac{p-1}{p}}}$.
 \label{lemma:bound_var}
\end{lem}
\begin{proof}
    Recall that the minibatch gradient estimator is defined as 
    \begin{align*}
        g = \frac{1}{M} \sum_{i=1}^M \nabla F (x, y; \xi_i).
    \end{align*}
    We start by controlling the expected deviation of $g$ from $\nabla f(x, y)$ using Lemma~\ref{lemma:mds_bound}.
    We define $X_j(x, y) \coloneqq \nabla F(x, y; \xi_j) - \nabla f(x, y)$ for all $j \in [M]$.
    Now note that $X_1(x, y), \dots, X_M(x, y)$ are independent random variables with mean zero and hence a Martingale Difference Sequence (MDS). 
    Furthermore, note that $\BE[\Norm{X_j(x, y)}^p] \leq \sigma^p$ by the assumption.
    Therefore, we can apply the Lemma~\ref{lemma:mds_bound} to get
    \begin{align*}
       l(x, y) \coloneqq \BE\left[ \Big\lVert \sum_{j=1}^M X_j(x, y) \Big\rVert^p\right] \leq 2 \sum_{j=1}^M \BE[\Norm{X_j(x, y)}^p] \leq 2 M \sigma^p.
    \end{align*}
    Next we calculate
    \begin{align*}
        \BE[\Norm{g - \nabla f(x, y)} \mid x, y] = & \BE\left[ \Big \lVert \frac{1}{M} \sum_{j=1}^M \left( \nabla F(x, y; \xi_j) - \nabla f(x, y)\right) \Big \rVert ~\Big|~ x, y \right] \\
        \leq & \frac{1}{M} \BE\left[\Big \lVert \sum_{j=1}^M \left( \nabla F(x, y; \xi_j) - \nabla f(x, y)\right) \Big \rVert^p ~\Big|~ x, y\right]^{1/p},
    \end{align*}
where we applied Jensen's inequality in the last inequality. 
Next we define
\begin{align*}
    Y=(\xi_1, \dots, \xi_M) \quad {\rm and} \quad h(x, y, Y) = \Big \lVert\sum_{j=1}^M \left( \nabla F(x, y, \xi_j) - \nabla f(x, y)\right) \Big \rVert^p,
\end{align*}
 and note that $(x, y)$ and $Y$ are independent.
 Hence, we may apply Lemma~\ref{lemma:random_variable_mapping} which yields
 \begin{align*}
     \BE[\Norm{g - \nabla f(x, y)}] \leq \frac{1}{M} \BE[h(x, y, Y) \mid x, y]^{1/p} = \frac{1}{M} l(x, y)^{1/p} \leq  \frac{2 \sigma}{M^{\frac{p-1}{p}}}.
 \end{align*}
\end{proof} 


The clip operator enjoys the following properties.
\begin{lem}[{\citet[Lemma 5.1]{sadiev2023high}}]
Let $X$ be a random vector in $\BR^d$ and define $\Tilde{X} = {\rm clip}(X, \tau) \coloneqq \min(1, \tau / \Norm{X}) X$, 
then we have $\|\Tilde{X} - \BE[\Tilde{X}]\| \leq 2 \tau$.
Moreover, if for some $\sigma \geq 0$ and $p \in (1, 2]$ we have $\BE[X] = x \in \BR^d$, $\BE[\Norm{X-x}^p] \leq \sigma^p$, and $\Norm{x} \leq \tau / 2$, then we have
\begin{equation*}
    \big\|\BE[\Tilde{X}] - x\big\| \leq \frac{2^p \sigma^{p}}{\tau^{p - 1}},~~
    \BE\left[\big\|\Tilde{X} - x\big\|^2 \right] \leq 18 \tau^{2-p} \sigma^{p}~~\text{and}~~
    \BE\left[\big\|\Tilde{X} - \BE[\Tilde{X}]\big\|^2\right] \leq 18 \tau^{2-p} \sigma^{p}.
\end{equation*}
\label{lemma:clip_grad}
\end{lem}
Finally, we present some technical lemmas for high-probability convergence analysis.
\begin{lem}[{\citet[Lemma 1]{li2020high}}]
    Let $\{\fF_t\}_{t\in \mathbb{N}}$ be a filtration and $\{\Theta_t\}_{t\in \mathbb{N}}$ be a Martingale Difference Sequence with respect to $\{\fF_t\}_{t\in \mathbb{N}}$.
    Furthermore, for each $t \in \mathbb{N}_+$, let $\sigma_t$ be $\fF_{t-1}$-measurable and assume that $\BE[\exp({\Theta_t^2}/{\sigma_t^2})\mid \fF_{t-1}] \leq {\rm e}$.
    Then for all $T \in \mathbb{N}$, $ \chi > 0$ and~$\hat\delta \in (0, 1)$, it holds that
    \begin{align*} 
        \BP \left( \sum_{t=1}^T \Theta_t \leq \frac{3\chi}{4} \sum_{t=1}^T \sigma_t^2 + \frac{1}{\chi} \ln \left(\frac{1}{\hat\delta}\right) \right) \geq 1 - \hat\delta.
    \end{align*}
    \label{lemma:hp_concentrate}
\end{lem}
Our proof also uses the following Bernstein inequality for martingale differences \citep{bennett1962probability,dzhaparidze2001bernstein,freedman1975tail}.
\begin{lem}[{\citet[Lemma~1.6]{freedman1975tail}}]
    Let the sequence of random variables $\{X_i\}_{i \geq 1}$ form a martingale difference sequence, i.e., $\BE[X_i \mid X_{i-1}, \dots, X_1] = 0$ for all $i \geq 1$.
    Assume that conditional variances $\sigma_i^2 \coloneqq \BE[X_i^2 \mid X_{i-1}, \ldots, X_1]$ exist and are bounded and assume also that there exists a deterministic constant $c>0$ such that $|X_i| \leq c$ almost surely for all $i \geq 1$.
    Then for all $b > 0$, $G > 0$ and $n > 1$, it holds that
    \begin{equation}
        \BP\left\{\bigg|\sum_{i=1}^n X_i\bigg| > b \quad{\rm and}\quad \sum_{i=1}^n \sigma_i^2 \leq G \right\} \leq 2 \exp\left(-\frac{b^2}{2 G + 2 cb/3}\right).
    \end{equation}
    \label{lemma:bernstein_inequality}
\end{lem}


\section{Proofs for Results of Stochastic Bilevel Optimization}\label{appendix:bilevel}

In this section, we establish the convergence guarantee of the N$^2$SBA algorithm for solving the nonconvex–strongly-convex bilevel optimization problem. 
We begin by presenting key properties of the bilevel objective function $\varphi(x)$ and the auxiliary function $\fL_{\lambda}^*(x)$.

\begin{lem}[{\citet[page 4]{kwon2023fully}}]
Under Assumption \ref{asm:bilevel}, for all $\lambda \geq 2 L_f / \mu$, the function $\fL_{\lambda} (x, y) \coloneqq f(x, y) + \lambda(g(x, y) - g^*(x))$ is $(\lambda \mu / 2)$-strongly convex in $y$ for given~$x\in\BR^{d_x}$. 
\label{lemma:lambda_strongly_convex}
\end{lem}





\begin{lem}[{\citet[Lemma~B.2--B.5]{chen2025near}}] \label{lemma:bilevel_relationship}
Under Assumption \ref{asm:bilevel}, for $\lambda \geq 2 L_f / \mu$, it holds that
\begin{align*}
   & \Norm{y^*_\lambda(x) - y^*(x)} \leq {C_f}/{(\lambda \mu)}, ~~~~ |\fL_\lambda^*(x) - \varphi(x)| \leq {D_0}/{\lambda}, \\
   & \Norm{\nabla \fL_\lambda^*(x) - \nabla \varphi(x)} \leq {D_1}/{\lambda}, ~~~ \Norm{\nabla y^*(x) - \nabla y_\lambda^*(x)} \leq {D_2}/{\lambda},
\end{align*}
where we define 
    \begin{align*}
       & D_0 \coloneqq \left(C_f \! + \! \frac{C_f L_g}{2 \mu}\right) \frac{C_f}{\mu} = \fO(\ell \kappa^2), ~~~D_1 \coloneqq \left(L_f \!+ \!\frac{\rho_g L_g}{\mu}\! + \!\frac{C_f L_g \rho_g}{2 \mu^2}\! + \!\frac{C_f \rho_g}{2 \mu}\right) \frac{C_f}{\mu} = \fO(\ell \kappa^3), \\
       &  D_2 \coloneqq \left(\frac{1}{\mu} + \frac{2 L_g}{\mu^2}\right) \left(L_f + \frac{C_f \rho_g}{\mu}\right) = \fO(\kappa^3).
    \end{align*}
\end{lem}

\begin{lem}[{\citet[{Lemma~B.6}]{chen2025near}}]
    Under Assumption \ref{asm:bilevel}, for $\lambda \geq 2 L_f / \mu$, it holds that $\Norm{\nabla y^*(x)} \leq L_g / \mu$ and $\Norm{\nabla y_\lambda^*(x)} \leq 4 L_g / \mu$.
    \label{lemma:y_smooth}
\end{lem}

\begin{lem}[{\citet[{Lemma~B.7}]{chen2025near}}]
    Under Assumption \ref{asm:bilevel}, for $\lambda \geq 2 L_f / \mu$, it holds that $\nabla \fL_{\lambda}^*(x)$ is $D_3$-Lipschitz, where
    \begin{align*}
        D_3 \coloneqq L_f + \frac{4 L_f L_g}{\mu} + \frac{C_f \rho_g}{\mu} + \frac{C_f L_g \rho_g}{\mu^2} + L_g D_2 = \fO(\ell \kappa^3).
    \end{align*}
    \label{lemma:smooth_bound}
\end{lem}

\subsection{The In-Expectation Convergence Guarantee}

In this subsection, we present the proofs for the in-expectation convergence guarantee of our N$^2$SBA shown in Section \ref{sec:bilevel_analysis}.

\subsubsection{The Proof of Lemma \ref{lemma:n2sba_one_iter}}
\begin{proof}
     Lemma \ref{lemma:smooth_bound} implies that $\fL_{\lambda}$ is $D_3$-smooth, then we have
    \begin{equation}
    \begin{split}
        \fL_{\lambda}^*(x_{t+1}) \leq & \fL_{\lambda}^*(x_t) + \langle \nabla \fL_{\lambda}^*(x_t), x_{t+1} - x_t \rangle + \frac{D_3}{2} \Norm{x_{t+1} - x_t}^2 \\
        =  & \fL_{\lambda}^*(x_t) - \eta_x \left\langle \nabla \fL_{\lambda}^*(x_t), \frac{g_{x, t}}{\Norm{g_{x,t}}} \right\rangle + \frac{D_3 \eta_x^2}{2} \\
        \leq & \fL_{\lambda}^*(x_t) - \eta_x \Norm{\nabla \fL_{\lambda}^*(x_t)} + 2 \eta_x \Norm{\nabla \fL_{\lambda}^*(x_t) - g_{x, t}}  + \frac{D_3 \eta_x^2}{2},
        \end{split}
        \label{bilevel_iter}
    \end{equation}
    where the last inequality is due to Lemma~\ref{lemma:lower_bound}.
    Taking the average on equation (\ref{bilevel_iter}) over $t=0,1,\dots,T-1$, we obtain
    \begin{equation}
        \frac{1}{T}\sum_{t=0}^{T-1}  \Norm{\nabla \fL_{\lambda}^*(x_t)} \leq \frac{\fL_{\lambda}^*(x_0) - \fL_{\lambda}^*(x_T)}{\eta_x T} + \sum_{t=0}^{T-1}  \frac{2\Norm{\nabla \fL_{\lambda}^*(x_t) - g_{x, t}} }{T} + \frac{D_3 \eta_x}{2}.
        \label{bilevel_main_bound}
    \end{equation}    
    Based on the facts \cite[Lemma 3.1]{kwon2023fully}
    \begin{align*}        
    \nabla \fL_{\lambda}^*(x) = \nabla_x f(x, y_\lambda^*(x)) + \lambda (\nabla_x g(x, y_\lambda^*(x)) - \nabla_x g(x, y^*(x)))
    \end{align*}
    and    
    \begin{align*}
         g_{x, t} = \frac{1}{M}\sum_{i=1}^M \left(\nabla_x F (x_{t}, y_{t}; \xi_{t,i}) + \lambda (\nabla_x G(x_t, y_t; \zeta_{t,i}) - \nabla_x G(x_t, z_t; \zeta_{t,i}))\right),
    \end{align*}
    we apply the triangle inequality to obtain
    \begin{equation*}
    \begin{split}
         & \Norm{g_{x, t} - \nabla \fL_{\lambda}^*(x_t)} \\
          \leq &  \Big\|\frac{1}{M}\sum_{i=1}^M \nabla_x F(x_t, y_t; \xi_{t,i}) - \nabla_x f(x_t, y_t)\Big\| + \lambda \Big \|\frac{1}{M}\sum_{i=1}^M \nabla_x G(x_t, y_t; \zeta_{t,i}) - \nabla_x g(x_t, y_t)\Big \|  \\
         & + \lambda \Big \|\frac{1}{M}\sum_{i=1}^M \nabla_x G(x_t, z_t; \zeta_{t,i}) - \nabla_x g(x_t, z_t)\Big \| + \Norm{\nabla_x f(x_t, y_t) - \nabla_x f(x_t, y_{\lambda}^*(x_t))} \\
         &  +\lambda \Norm{\nabla_x g(x_t, y_t) -\nabla_x g(x_t, y_{\lambda}^*(x_t))} + \lambda \Norm{\nabla_x g(x_t, z_t) - \nabla_x g(x_t, y^*(x_t))}\\
         \leq & \Big\|\frac{1}{M}\sum_{i=1}^M \nabla F(x_t, y_t; \xi_{t,i}) - \nabla f(x_t, y_t)\Big\| + \lambda \Big \|\frac{1}{M}\sum_{i=1}^M \nabla G(x_t, y_t; \zeta_{t,i}) - \nabla g(x_t, y_t)\Big \|  \\
         &+ \lambda \Big \|\frac{1}{M}\sum_{i=1}^M \nabla G(x_t, z_t; \zeta_{t,i}) - \nabla g(x_t, z_t)\Big \| + \Norm{\nabla f(x_t, y_t) - \nabla f(x_t, y_{\lambda}^*(x_t))} \\
         & +\lambda \Norm{\nabla g(x_t, y_t) -\nabla g(x_t, y_{\lambda}^*(x_t))} + \lambda \Norm{\nabla g(x_t, z_t) - \nabla g(x_t, y^*(x_t))} .
    \end{split}
    \end{equation*}
    Taking the expectation on both sides of the above inequality, we have
    \begin{equation}
    \begin{split}
        & \BE[\Norm{g_{x, t} - {\nabla \fL_{\lambda}^*(x_t)}}] \\
        \leq & \frac{2 \sigma_f}{M^{\frac{p-1}{p}}} + \frac{4 \lambda \sigma_g}{M^{\frac{p-1}{p}}} + (L_f + \lambda L_g) \BE[\Norm{y_t - y_{\lambda}^*(x_t)}] + \lambda L_g \BE[\Norm{z_t - y^*(x_t)}] \\
         \leq &  
    \frac{2 \sigma_f}{M^{\frac{p-1}{p}}} + \frac{4 \lambda \sigma_g}{M^{\frac{p-1}{p}}} + \left(\frac{\lambda \mu}{2} + \lambda L_g\right) \BE[\Norm{y_t - y_{\lambda}^*(x_t)}] + \lambda L_g \BE[\Norm{z_t - y^*(x_t)}] \\
        \leq & \frac{2 \sigma_f + 4 \lambda \sigma_g}{M^{\frac{p-1}{p}}} + 2 \lambda L_g \BE[\Norm{y_t - y_{\lambda}^*(x_t)}] + \lambda L_g \BE[\Norm{z_t - y^*(x_t)}],
    \end{split}
    \label{bilevel_var_bound}
    \end{equation}
    where the first inequality follows from Lemma \ref{lemma:bound_var}, {the second inequality is due to the setting of $\lambda \geq 2 L_f / \mu$, and the last inequality follows the fact $L_g \geq \mu$.}
    Combing the results of (\ref{bilevel_main_bound}) and  (\ref{bilevel_var_bound}), we obtain
    \begin{equation}
    \begin{split}
         \frac{1}{T}\sum_{t=0}^{T-1}  \BE[\Norm{\nabla \fL_{\lambda}^*(x_t)}] \leq & \frac{\BE[\fL_{\lambda}^*(x_0) - \fL_{\lambda}^*(x_T)]}{\eta_x T} + \frac{2}{T}\sum_{t=0}^{T-1}  \BE[\Norm{\nabla \fL_{\lambda}^*(x_t) - g_{x, t}}]  + \frac{D_3 \eta_x}{2} \\
          \leq & \frac{\BE[\fL_{\lambda}^*(x_0) - \fL_{\lambda}^*(x_T)]}{\eta_x T} +  \frac{4 \sigma_f + 8 \lambda \sigma_g}{M^{\frac{p-1}{p}}} + \frac{4 \lambda L_g}{T}\sum_{t=0}^{T-1}  \BE[\Norm{y_t - y_{\lambda}^*(x_t)}] \\
          & + \frac{2 \lambda L_g}{T} \sum_{t=0}^{T-1} \BE[\Norm{z_t - y^*(x_t)}] + \frac{D_3 \eta_x}{2},
    \end{split}
    \label{bilevel_main_bound_2}
    \end{equation}
    which finishes the proof.
\end{proof}

\subsubsection{The Proof of Theorem \ref{thm:bilevel_expect}}\label{appendix:thm3}

We first provide an upper bound on the $p$-th bounded central moment of the stochastic gradient estimator of $\fL_{\lambda}(x, y)$ as follows.

\begin{lem}
    Under Assumption \ref{asm:pBCM}, it holds that
    \begin{align*}
        {\BE}[\Norm{\nabla_y F(x, y; \xi) + \lambda \nabla_y G(x, y; \zeta) - \nabla_y f(x, y) - \lambda \nabla_y g(x, y)}^p] \leq & (2\sigma_f + 2\lambda \sigma_g)^p.
    \end{align*}
    \label{lemma:var_bound}
\end{lem}
\begin{proof}
    Based on  Assumption \ref{asm:pBCM}, we have
    
     \begin{equation}
        \begin{split}
            & {\BE}[\Norm{\nabla_y F(x, y; \xi) + \lambda \nabla_y G(x, y; \zeta) - \nabla_y f(x, y) - \lambda \nabla_y g(x, y)}^p] \\
       \leq & {\BE}[(\Norm{\nabla_y F(x, y; \xi) - \nabla_y f(x, y)} + \lambda \Norm{\nabla_y G(x, y; \zeta) - \nabla_y g(x, y)})^p] \\
       \leq & 2^{p-1}{\BE}[\Norm{ \nabla F(x, y; \xi) - \nabla f(x, y) }^p] + 2^{p-1} \lambda^p {\BE}[\Norm{\nabla G(x, y; \zeta) - \nabla g(x, y)}^p] \\
       \leq & 2^{p-1} \sigma_f^p + 2^{p-1} \lambda^p \sigma_g^p \\
       \leq & 2^{p-1} (\sigma_f + \lambda \sigma_g)^p \\
       \leq & (2\sigma_f + 2\lambda \sigma_g)^p,
        \end{split}
        \label{var_upper_bound}
    \end{equation}
    
    where the first inequality is due to the triangle inequality, and the second inequality follows from the convexity of $\Norm{\cdot}^p$.
    The third inequality is due to Assumption \ref{asm:pBCM}, and the fourth inequality follows from $a^p + b^p \leq (a+b)^p$ for any positive number $a$ and $b$.
\end{proof}

We then provide lemmas to upper bound the {terms} $\BE[\Norm{y_t - y_{\lambda}^*(x_t)}]$ and $\BE[\Norm{z_t - y^*(x_t)}]$.
\begin{lem}\label{lemma:yz-distance}
    Following the setting of Theorem \ref{thm:bilevel_expect}, for all $t=0,\dots,T-1$, we have
    \begin{align*}
    \BE\left[\Norm{{y}_{t} - y_\lambda^*(x_{t})}\right] 
    = \fO\left(\frac{\epsilon^2}{\ell^2 \kappa^3}\right)    
    \quad \text{and} \quad
    \BE\left[\Norm{{z}_{t} - y^*(x_{t})}\right] 
    = \fO\left(\frac{\epsilon^2}{\ell^2 \kappa^3}\right)
    \end{align*}
\end{lem}
\begin{proof}
We specify the parameter settings for N$^2$SBA as
 \begin{align*}
 \begin{split}
        & \lambda = \max\left\{\frac{ \kappa}{ R_0}, \frac{\ell \kappa^2}{ \Delta}, \frac{\ell \kappa^3}{ \epsilon}\right\},~~~ \eta_{y,t} = \min\left\{\frac{1}{400 (L_f + \lambda L_g)}, \frac{2 \ln B_{y,t}}{\lambda \mu K}  \right\} \\
        & \eta_{z,t} = \min\left\{\frac{1}{400 \lambda L_g}, \frac{\ln(B_{z,t})}{\lambda \mu K}  \right\},~~~  K = \tilde\fO\left(\frac{\ell^{\frac{p}{p-1}}  \kappa^{\frac{4p}{p-1}} \sigma^{\frac{p}{p-1}}}{\epsilon^{\frac{2p}{p-1}}}+ {\frac{\ell^2 \kappa^3 R_y}{\epsilon^2}}\right),~~~  \eta_x = \frac{\epsilon}{\ell \kappa^3}, \\
       & \tau_{t,k} = \frac{\exp(- \eta_{z,t} \lambda \mu (1/2 + k/ 4) )R_z}{120 \eta_{z,t}},~~\tau_{t,k}' = \frac{\exp(- \eta_{y,t} \lambda \mu (1 + k/ 2) )R_y}{120 \eta_{y,t}},
     \end{split}
    \end{align*}
    where we set 
    \begin{align*}
    \small\begin{split}
       &B_{y, t} = \max\left\{2, \frac{\lambda^2\mu^2 K^{ \frac{2 (p - 1)}{p}} R_y^2}{5400^{\frac{2}{p}} (\sigma_f +\lambda\sigma_g)^2 (\ln (B_{y, t}))^2 }\right\},~~~ B_{z, t} = \max\left\{2, \frac{\mu^2 K^{ \frac{2 (p - 1)}{p}} R_z^2}{5400^{\frac{2}{p}} \sigma_g^2 (\ln (B_{z, t}))^2 }\right\},\\
       & R_y^2 = R_{z}^2 \coloneqq \fO\left(R_0^2 \! + \! \frac{ \epsilon^4}{\ell^4 \kappa^6} \! + \!\frac{ \epsilon^2}{\ell^2 \kappa^4}\right),~R_0 \geq \Norm{\hat{y}_{0,0} - y^*(x_0)},~\Delta = \varphi(x_0) -\!\! \inf_{x \in \BR^{d_x}}\!\!\varphi(x).
    \end{split}
    \end{align*}    

    We first consider the term $\BE[\Norm{y_t - y_{\lambda}^*(x_t)}]$.    
    Recall that Algorithm \ref{alg:nsgd_bilevel} set $y_t=\hat{y}_{t, K}$ in line~\ref{line:y_t}, where $\hat{y}_{t, K}$ can be regarded as the output of applying clipped stochastic gradient descent (lines \ref{bilevel_y_start}--\ref{bilevel_y_end}) to minimize the function $\fL_{\lambda} (x_t, y)=f(x_t, y) + \lambda(g(x_t, y) - g^*(x_t))$ with respect to~$y$. The updates use the stochastic gradient estimate 
    $\nabla_y F(x_t, y; \xi'_t) + \lambda \nabla_y G(x_t, y; \zeta'_t)$ for $\fL_{\lambda} (x_t, y)$. 
    Since the setting of $\lambda$ and Lemma~\ref{lemma:lambda_strongly_convex} implies that the function $\fL_{\lambda} (x_t, y)$ is $(\lambda \mu / 2)$-strongly convex in~$y$, we apply Lemma \ref{lemma:y_recur_strongly} with $h(y) = \fL_\lambda(x_t,y)$, $\sigma_h = 2(\sigma_f + \lambda \sigma_g)$, $\ell_h= 2 \lambda L_g$, and $\mu_h= \lambda \mu/2$ to obtain the following result 
    \begin{align}\label{eq:dis_t}\small
    \begin{split}
        & \BE\left[\Norm{\hat{y}_{t, K} - y_\lambda^*(x_{t})}^2\right] \\
        = & {
        \fO \left({\hat R}_{y,t}^2 \exp\left( -\frac{K\mu }{L_g \ln(K L_g/\mu)}\right) + \frac{(\sigma_f + \lambda \sigma_g)^2 (\ln (B_K))^2(\ln({K}L_g/\mu))^{\frac{2(p-1)}{p}}}{\lambda^2\mu^2 K^{ \frac{2 (p - 1)}{p}}}+ \frac{{\hat R}_{y,t}^2 }{K^2} \right) },
    \end{split}
    \end{align}
    for all ${\hat R}_{y,t}^2 \geq \BE[\Norm{\hat{y}_{t, 0} - y_\lambda^*(x_{t})}^2]$. Here, the choice of $\sigma_h = 2(\sigma_f + \lambda \sigma_g)$ that corresponds to the upper bound for the $p$-th central moment shown in Lemma \ref{lemma:var_bound}.

    We then use induction to show each $\BE[\Norm{\hat{y}_{t, 0} - y_\lambda^*(x_{t})}^2]$ has a uniform upper bound, i.e.,
    \begin{equation}
    \BE\left[\Norm{\hat{y}_{t,0} - y^*_{\lambda}(x_t)}^2\right] = R_y^2 \coloneqq \fO\left(R_0^2 + \frac{ \epsilon^4}{\ell^4 \kappa^6} + \frac{ \epsilon^2}{\ell^2 \kappa^4}\right)
    \label{y_induction}
    \end{equation}
    holds for all $t = 0, \ldots, T-1$. 
For the induction base $t=0$, it follows that 
\begin{align}\label{eq:induction-dist-000}
\!\!\!\BE[\Norm{\hat{y}_{0,0} - y^*_{\lambda}(x_0)}^2] \leq 2 \BE[\Norm{\hat{y}_{0,0} - y^*(x_0)}^2] + 2\Norm{y^*(x_0) - y_\lambda^*(x_0)}^2 \leq 2 R_0^2 + 2 R_0^2 = 4 R_0^2,   
\end{align}
where the first step follows the Young's inequality and the second step is due to {the result~$\Norm{y^*_\lambda(x) - y^*(x)} \leq {C_f}/{(\lambda \mu)}$
from Lemma \ref{lemma:bilevel_relationship} and the settings of $R_0\geq  \Norm{\hat{y}_{0,0} - y^*(x_0)}$ and $\lambda=\max\left\{{ \kappa}{ R_0^{-1}}, {\ell \kappa^2}{ \Delta^{-1}}, {\ell \kappa^3}{ \epsilon^{-1}}\right\}$.}

For the induction step, we suppose \eqref{y_induction} holds for all~$t = 0, \dots, T'-1$. In the case of $t = T'$, we have
\begin{align}\label{eq:induction-dist-00}
\begin{split}
    \Norm{\hat{y}_{T',0} - y_{\lambda}^*(x_{T'})}^2 \leq & 2 \Norm{\hat{y}_{T'-1,K} - y_{\lambda}^*(x_{T'-1})}^2 + 2 \Norm{y_{\lambda}^*(x_{T'}) - y_{\lambda}^*(x_{T'-1})}^2 \\
    \leq & 2 \Norm{\hat{y}_{T'-1,K} - y_{\lambda}^*(x_{T'-1})}^2 +  \frac{32 L_g^2}{\mu^2}\Norm{x_{T'} - x_{T'-1}}^2 \\
    \leq & 2 \Norm{\hat{y}_{T'-1,K} - y_{\lambda}^*(x_{T'-1})}^2 +  \frac{32 L_g^2 \eta_x^2}{\mu^2},
\end{split}    
\end{align}
where the first inequality follows from Young's inequality, the second inequality is due to Lemma~\ref{lemma:y_smooth}, and the last inequality is based on line \ref{line:outer-end} of Algorithm \ref{alg:nsgd_bilevel}.
Following the derivation of \eqref{eq:dis_t}, we have
\begin{align}\label{eq:induction-T-1}\small
\begin{split}
    &\BE\left[\Norm{\hat{y}_{T'-1, K} - y_\lambda^*(x_{T'-1})}^2\right] \\
    = & {\fO\left(\! {\hat R}_{y,T'-1}^2 \exp\left(- \frac{\mu K}{ L_g \ln(K L_g/\mu )} \right) \!+ \! \frac{ ( \sigma_f \!+ \!\lambda \sigma_g)^2 (\ln (B_{y,T'-1}))^2 (\ln(K L_g/\mu ))^{\frac{2(p-1)}{p}}}{\lambda^2 \mu^2 K^{ \frac{2 (p - 1)}{p}} } \!+\!\frac{{\hat R}_{y,T'-1}^2 }{K^2} \!\right)}
\end{split}
\end{align}
for all ${\hat R}_{y,T'-1}^2 \geq \BE[\Norm{\hat{y}_{T'-1, 0} - y_\lambda^*(x_{T'-1})}^2]$. 
In addition, the induction hypothesis (\ref{y_induction}) implies that we can 
apply \eqref{eq:induction-T-1} by
taking ${\hat R}_{y,T'-1}^2=R_y^2$ to achieve
\begin{align*}
    &\BE\left[\Norm{\hat{y}_{T'-1, K} - y_\lambda^*(x_{T'-1})}^2\right]\\
    =& {\fO \left( R_{y}^2 \exp\left(- \frac{\mu K}{ L_g \ln(K L_g/\mu )} \right) +  \frac{ ( \sigma_f + \lambda \sigma_g)^2 (\ln (B_{y,T'-1}))^2 (\ln(K L_g/\mu ))^{\frac{2(p-1)}{p}}}{\lambda^2 \mu^2 K^{ \frac{2 (p - 1)}{p}}}+\frac{R_{y}^2 }{K^2}\right).}
\end{align*}
Combining above inequality with \eqref{eq:induction-dist-00}, we have
\begin{align*}\small
\begin{split}    
    &\BE\left[\Norm{\hat{y}_{T',0} - y_{\lambda}^*(x_{T'})}^2\right] \\
\leq & {\fO \left( R_{y}^2 \exp\left(- \frac{\mu K}{ L_g \ln(K L_g/\mu )} \right) +  \frac{ ( \sigma_f + \lambda \sigma_g)^2 (\ln (B_{y,T'-1}))^2 (\ln(K L_g/\mu ))^{\frac{2(p-1)}{p}}}{\lambda^2 \mu^2 K^{ \frac{2 (p - 1)}{p}}}+\frac{R_{y}^2 }{K^2}\right) +  \frac{32 L_g^2 \eta_x^2}{\mu^2}} \\
\leq & { \fO\left(\frac{ \epsilon^4}{\ell^4 \kappa^6}\right)  + \frac{32 L_g^2 \eta_x^2}{\mu^2} 
= \fO\left(\frac{ \epsilon^4}{\ell^4 \kappa^6} + \frac{ \epsilon^2}{\ell^2 \kappa^4}\right) } 
    \leq  R_y^2,
\end{split}
\end{align*}
where the second inequality is based on the setting of $K$,
the third inequality is based on the setting of $\eta_x$,
and the last inequality is based on the definition of $R_y^2$.
This completes the induction.

Therefore, we can apply Lemma \ref{lemma:y_recur_strongly} on $\Norm{\hat{y}_{t, K} - y_\lambda^*(x_{t})}^2$ with $\hat{R}_{y,t} = R_y$ to obtain
\begin{align*}
\begin{split}    
   & \BE\left[\Norm{{y}_{t} - y_\lambda^*(x_{t})}^2\right] = \BE\left[\Norm{\hat{y}_{t, K} - y_\lambda^*(x_{t})}^2\right] \\ 
   \leq & {\fO \left(  R_y^2 \exp\left( -\frac{K\mu }{L_g \ln(K L_g/\mu )}\right) + \frac{(\sigma_f + \lambda \sigma_g)^2 (\ln (B_K))^2 \big(\ln\big({    K  }L_g / \mu\big)\big)^{\frac{2(p-1)}{p}}}{\lambda^2 \mu^2 K^{ \frac{2 (p - 1)}{p}}} + \frac{R_y^2 }{K^2} \right)}\\
   =& \fO\left(\frac{\epsilon^4}{\ell^4 \kappa^6}\right)
\end{split}
\end{align*}
for all $t=0,\dots,T-1$. Consequently, we use Jensen's inequality to achieve
\begin{align*}   
   & \BE\left[\Norm{{y}_{t} - y_\lambda^*(x_{t})}\right] 
   \leq \sqrt{\BE\left[\Norm{{y}_{t} - y_\lambda^*(x_{t})}^2\right]}   
    = \fO \left(\frac{\epsilon^2}{\ell^2 \kappa^3}\right).
\end{align*}    

    Next, we consider the term $\BE[\Norm{z_t - y^*(x_t)}]$.    
    Recall that Algorithm \ref{alg:nsgd_bilevel} set $z_t=\hat{z}_{t, K}$ in line~\ref{line:z_t}, where $\hat{z}_{t, K}$ can be regarded as the output of applying clipped stochastic gradient descent (lines \ref{bilevel_z_start}--\ref{bilevel_z_end}) to minimize the function $\lambda g(x_t, y)$ with respect to~$y$. The updates use the stochastic gradient estimate 
    $\lambda \nabla_y G(x_t, y; \zeta_t)$ for $\lambda g(x_t, y)$. 
    Since the function $\lambda g(x_t, y)$ is $\lambda \mu$-strongly convex in~$y$, we apply Lemma \ref{lemma:y_recur_strongly} with $h(y) = \lambda g(x_t,y)$, $\sigma_h = \lambda \sigma_g$, $\ell_h=  \lambda L_g$, and $\mu_h= \lambda \mu$ to obtain
\begin{align}\label{eq:dis_zt}
\begin{split}
    & \BE\left[\Norm{\hat{z}_{t, K} - y^*(x_{t})}^2\right] \\ 
    =& {\fO\left( \hat{R}_{z,t}^2 \exp\left(- \frac{\mu K}{ L_g \ln(K L_g /\mu)} \right) +  \frac{ \sigma_g^2 (\ln (B_{z,t}))^2 (\ln(K L_g /\mu))^{\frac{2(p-1)}{p}}}{\mu^2 K^{ \frac{2 (p - 1)}{p}} } + \frac{\hat{R}_{z,t}^2}{K^2}\right).}
\end{split}
\end{align}
for all ${\hat R}_{z,t}^2 \geq \BE[\Norm{\hat{z}_{t, 0} - y^*(x_{t})}^2]$. \newpage

Similar to the derivation of \eqref{y_induction}, 
we use induction to show each $\BE[\Norm{\hat{z}_{t, 0} - y^*(x_{t})}^2]$ has a uniform upper bound, i.e.,
    \begin{equation}
    \BE\left[\Norm{\hat{z}_{t,0} - y^*(x_t)}^2\right] \leq R_z^2 \coloneqq \fO\left(R_0^2 +  \frac{ \epsilon^4}{\ell^4 \kappa^6} + \frac{ \epsilon^2}{\ell^2 \kappa^4}\right)
    \label{z_induction}
    \end{equation}
    holds for all $t = 0, \ldots, T-1$. 
    For the induction base $t=0$, it follows that 
\begin{align}\label{z_induction_base_case}
\BE[\Norm{\hat{z}_{0, 0} - y^*(x_0)}^2] = \BE[\Norm{\hat{y}_{0, 0} - y^*(x_0)}^2]\leq R_0^2 \leq R_{z}^2,   
\end{align}
where the first step follows the definition of $R_0$ and the second step is due to the setting of~$R_0$.
For the induction step, we suppose \eqref{z_induction} holds for all~$t = 0, \dots, T'-1$. In the case of $t = T'$, we have
\begin{align}\label{eq:induction-dist-z0}
\begin{split}
    \Norm{\hat{z}_{T',0} - y^*(x_{T'})}^2 \leq & 2 \Norm{\hat{z}_{T'-1, K} - y^*(x_{T'-1})}^2 + 2 \Norm{y^*(x_{T'}) - y^*(x_{T'-1})}^2 \\
    \leq &  2 \Norm{\hat{z}_{T'-1, K} - y^*(x_{T'-1})}^2 + \frac{2 L_g^2}{\mu^2}\Norm{x_{T'} - x_{T'-1}}^2 \\
    = &  2 \Norm{\hat{z}_{T'-1, K} - y^*(x_{T'-1})}^2 + \frac{2 L_g^2 \eta_x^2}{\mu^2} ,
\end{split}    
\end{align}
where the first inequality follows from Young's inequality, the second inequality is due to Lemma~\ref{lemma:y_smooth}, and the last inequality is based on line \ref{line:outer-end} of Algorithm \ref{alg:nsgd_bilevel}.
Following the derivation of \eqref{eq:dis_zt}, we have
\begin{align}\label{eq:induction-z_T-1}
\begin{split}
      &\BE\left[\Norm{\hat{z}_{T'-1, K} - y^*(x_{T'-1})}^2\right]\\ 
      =& {\fO \left( \hat{R}_{z,T'-1}^2 \exp\left(- \frac{\mu K}{ L_g \ln(KL_g /\mu)} \right) +  \frac{ \sigma_g^2 (\ln (B_{z,T'-1}))^2 (\ln(K L_g /\mu))^{\frac{2(p-1)}{p}} }{\mu^2 K^{ \frac{2 (p - 1)}{p}} } + \frac{\hat{R}_{z,T'-1}^2}{K^2} \right)}
\end{split}
\end{align}
for all ${\hat R}_{z,T'-1}^2 \geq \BE[\Norm{\hat{z}_{T'-1, 0} - y^*(x_{T'-1})}^2]$. 
In addition, the induction hypothesis (\ref{z_induction}) implies we can 
apply \eqref{eq:induction-z_T-1} by
taking ${\hat R}_{z,T'-1}^2=R_z^2$ to achieve
\begin{align*}
\begin{split}    
    &\BE\left[\Norm{\hat{z}_{T'-1, K} - y^*(x_{T'-1})}^2\right] \\ 
    = & {\fO \left(R_z^2 \exp\left(- \frac{\mu K}{ L_g \ln(K L_g /\mu)} \right) +  \frac{ \sigma_g^2 (\ln (B_{z,T'-1}))^2 (\ln(K L_g /\mu))^{\frac{2(p-1)}{p}}}{\mu^2 K^{ \frac{2 (p - 1)}{p}} } + \frac{R_z^2}{K^2}\right).}
\end{split}
\end{align*}
Combining above inequality with \eqref{eq:induction-dist-z0}, we have
\begin{align*}
\begin{split}    
& \BE\left[\Norm{\hat{z}_{T',0} - y^*(x_{T'})}^2\right]\\
    \leq &  {\fO \left(R_z^2 \exp\left(- \frac{\mu K}{ L_g \ln(K L_g /\mu)} \right) +  \frac{ \sigma_g^2 (\ln (B_{z,T'-1}))^2 (\ln(K L_g /\mu))^{\frac{2(p-1)}{p}}}{\mu^2 K^{ \frac{2 (p - 1)}{p}} } + \frac{R_z^2}{K^2}\right) + \frac{2 L_g^2 \eta_x^2}{\mu^2}}\\
\leq & {\fO\left( \frac{ \epsilon^4}{\ell^4 \kappa^6}\right) + \frac{2 L_g^2 \eta_x^2}{\mu^2}
\leq  \fO\left(\frac{ \epsilon^4}{\ell^4 \kappa^6} + \frac{ \epsilon^2}{\ell^2 \kappa^4} \right)}
    \leq  R_z^2, 
\end{split}
\end{align*}
where the second inequality is based on the setting of $K$,
the third inequality is based on the setting of $\eta_x$,
and the last inequality is based on the definition of $R_z^2$.
This completes the induction.

Therefore, we can apply Lemma \ref{lemma:y_recur_strongly} on $\Norm{\hat{z}_{t, K} - y^*(x_{t})}^2$ with $\hat{R}_{z,t} = R_z$ to obtain
\begin{align*}
\begin{split}    
   & \BE\left[\Norm{{z}_{t} - y^*(x_{t})}^2\right] = \BE\left[\Norm{\hat{z}_{t, K} - y^*(x_{t})}^2\right] \\ 
   =& {\fO \left(R_z^2 \exp\left(- \frac{\mu K}{ L_g \ln(K L_g /\mu)} \right) +  \frac{ \sigma_g^2 (\ln (B_{z,t}))^2 (\ln(K L_g /\mu))^{\frac{2(p-1)}{p}}}{\mu^2 K^{ \frac{2 (p - 1)}{p}} } + \frac{R_z^2}{K^2}\right) 
   = \fO\left( \frac{\epsilon^4}{\ell^4 \kappa^6}\right)}
\end{split}
\end{align*}
for all $t=0,\dots,T-1$.
Consequently, we use Jensen's inequality to achieve
\begin{align*}   
   & \BE\left[\Norm{{z}_{t} - y^*(x_{t})}\right] 
   \leq \sqrt{\BE\left[\Norm{{z}_{t} - y^*(x_{t})}^2\right]}   
   = \fO\left(\frac{\epsilon^2}{\ell^2 \kappa^3}\right).
\end{align*}    
\end{proof}

We then provide the proof of Theorem \ref{thm:bilevel_expect}.

\begin{proof}
We follow the parameter settings in the proof of Lemma \ref{lemma:yz-distance} and additionally set
\begin{align*}
T = \fO\left(\frac{\Delta \ell \kappa^3}{\epsilon^2}\right)
\qquad\text{and}\qquad
M = \fO\left(\frac{\ell^{\frac{p}{p-1}} \kappa^{\frac{3p}{p-1}}\sigma^{\frac{p}{p-1}}}{\epsilon^{\frac{2p}{p-1}}}\right).
\end{align*}
    According to Lemma \ref{lemma:bilevel_relationship} and the settings of $\lambda$ and $D_1$, we have
    \begin{align}\label{eq:diff-phi-L}
       \Norm{\nabla \fL_{\lambda}^*(x) - \nabla \varphi(x)} =  \fO(\epsilon)
    \end{align}
    for all $x\in\BR^{d_x}$. Therefore, we only need to consider the task of finding an $\fO(\epsilon)$-stationary point of $\fL_{\lambda}^*$. Recall that Lemma \ref{lemma:n2sba_one_iter} implies
    \begin{equation}
    \begin{split}
         \!\frac{1}{T}\sum_{t=0}^{T-1}  \BE[\Norm{\nabla \fL_{\lambda}^*(x_t)}] 
          \leq & \frac{\BE[\fL_{\lambda}^*(x_0) - \fL_{\lambda}^*(x_T)]}{\eta_x T} +  \frac{4 \sigma_f + 8 \lambda \sigma_g}{M^{\frac{p-1}{p}}} + \frac{4 \lambda L_g}{T}\sum_{t=0}^{T-1}  \BE[\Norm{y_t - y_{\lambda}^*(x_t)}]\! \\
          & + \sum_{t=0}^{T-1} \frac{2 \lambda L_g \BE[\Norm{z_t - y^*(x_t)}]}{T} + \frac{D_3 \eta_x}{2}.
    \end{split}  
    \label{eq:key_descent-2}
    \end{equation}
    In the remainder of this proof, we will show that each term on the right-hand side of the equation~(\ref{eq:key_descent-2}) can be upper bounded by $\fO(\epsilon)$ in expectation, which means that the output $\hat x$ is an  $\fO(\epsilon)$-stationary point of $\fL_{\lambda}^*$, also an  $\fO(\epsilon)$-stationary point of $\varphi$.
    
    According to Lemma \ref{lemma:bilevel_relationship} and the settings of $\lambda$ and $D_0$, we have
    \begin{align}\label{eq:L_lambda_Delta}
        \fL_{\lambda}^*(x_0) - \inf_{x \in \BR^{d_x}} \fL_{\lambda}^* (x) = \fO(\Delta).
    \end{align}
    Based on the settings of $\eta_x$, $T$, and $M$, the fact $D_3=\fO(\ell \kappa^3)$ (see Lemma \ref{lemma:smooth_bound}), and \eqref{eq:L_lambda_Delta}, we have
    \begin{equation}\label{eq:L-sigma-Deta}
        \frac{\BE[\fL_{\lambda}^*(x_0) - \fL_{\lambda}^*(x_T)]}{\eta_x T} = \fO(\epsilon), \quad
        \frac{4 \sigma_f + 8 \lambda \sigma_g}{M^{\frac{p-1}{p}}}= \fO(\epsilon), \quad \text{and} \quad
        \frac{D_3 \eta_x}{2}= \fO(\epsilon).
    \end{equation}   
According to Lemma \ref{lemma:yz-distance}, we have
\begin{align}\label{eq:yz-distance}
    \BE\left[\Norm{{y}_{t} - y_\lambda^*(x_{t})}\right] 
    = \fO\left(\frac{\epsilon^2}{\ell^2 \kappa^3}\right)    
    \quad \text{and} \quad
    \BE\left[\Norm{{z}_{t} - y^*(x_{t})}\right] 
    = \fO\left(\frac{\epsilon^2}{\ell^2 \kappa^3}\right)
\end{align}
Combining results of (\ref{eq:key_descent-2}), (\ref{eq:L-sigma-Deta}), and (\ref{eq:yz-distance}), we have
\begin{align*}
    \BE[\Norm{\nabla \fL_{\lambda}^*(\hat x)}] = \frac{1}{T}\sum_{t=0}^{T-1}  \BE[\Norm{\nabla \fL_{\lambda}^*(x_t)}] = \fO(\epsilon),
\end{align*}
where the first step is due to that $\hat{x}$ is uniformly sampled from $\{x_t\}_{t=1}^T$.
Therefore, \eqref{eq:diff-phi-L} implies the point $\hat x$ is an $\fO(\epsilon)$-stationary point of $\varphi$ in expectation.
In addition, the total SFO complexity is 
\begin{align*}
    TM + TK = \tilde\fO\left(\frac{\Delta\ell^{\frac{2p-1}{p-1}} \kappa^{\frac{7p-3}{p-1}}\sigma^{\frac{p}{p-1}}}{\epsilon^{\frac{4p-2}{p-1}}} {+ \frac{\Delta \ell^3 \kappa^6 R_0}{\epsilon^4} } \right).
\end{align*}
\end{proof}

\subsection{The High-Probability Convergence Guarantee}

In this subsection, we present the proofs for the high-probability convergence guarantee of our N$^2$SBA shown in Section \ref{sec:bilevel_analysis}.
 We denote the natural filtration of our method by~$\fF_t \coloneqq \sigma(\{g_{x,s}\}_{s \leq t}, \{g_{y,s,k}, g_{z,s,k}\}_{s \leq {t+1},\,k \leq K})$.

\subsubsection{Technical Lemmas}

We first provide a lemma to bound the gradient norm of $\fL_{\lambda}^*$ as follows.

\begin{lem}
We denote $\delta \in (0,1)$ and $\tilde{\fL}_{\lambda}(x,y,z) = f(x,y) + \lambda \left(g(x,y) -  g(x,z)\right)$. 
Under Assumptions \ref{asm:bilevel} and \ref{asm:pBCM}, running N$^2$SBA (Algorithm \ref{alg:nsgd_bilevel}) with $\lambda \geq 2 L_f / \mu$ holds
\begin{equation}
    \begin{split}
       \frac{1}{T} \sum_{t=0}^{T-1}  \Norm{\nabla \fL_{\lambda}^*(x_t)} 
        \leq & \frac{ 2 (\fL_{\lambda}^*(x_0) - \fL_{\lambda}^*(x_T))}{\eta_x T} +  4\sum_{t=0}^{T-1} \frac{ {\BE\left[\Norm{g_{x, t} - \nabla_x \tilde{\fL}_{\lambda}(x_t, y_t, z_t)} \mid \fF_{t-1}\right]}}{T}     \\
        &   + 4\sum_{t=0}^{T-1} \frac{\Norm{\nabla_x \tilde{\fL}_{\lambda}(x_t, y_t, z_t) - \nabla \fL_{\lambda}^*(x_t)}}{T} + D_3 \eta_x \\
        & + 12 \left( \frac{\Norm{\nabla \fL_{\lambda}^*(x_0)}}{T} + \eta_x D_3\right)\ln\left(\frac{2}{\delta}\right)
    \end{split}
    \label{high_prob_core_bound}
\end{equation}
with probability at least $1-\delta/2$.
\end{lem}

\begin{proof}
    Let $\nu_t \coloneqq {\langle \nabla \fL_{\lambda}^* (x_t), g_{x,t} \rangle}/(\Norm{\nabla \fL_{\lambda}^*(x_t)} \Norm{g_{x,t}})$.
    According to \eqref{bilevel_iter}, we have
    \begin{align*}
        \fL_{\lambda}^*(x_{t+1}) \leq  & \fL_{\lambda}^*(x_t) - \eta_x \left\langle \nabla \fL_{\lambda}^*(x_t), \frac{g_{x, t}}{\Norm{g_{x,t}}} \right\rangle + \frac{D_3 \eta_x^2}{2} 
        =  \fL_{\lambda}^*(x_t) - \eta_x \Norm{\nabla \fL_{\lambda}^*(x_t)} \nu_t + \frac{D_3 \eta_x^2}{2}.
    \end{align*}
    Summing over above inequality with $t= 0$ to $T-1$, we obtain   
    \begin{align}
        \sum_{t=0}^{T-1} \eta_x \Norm{\nabla \fL_{\lambda}^*(x_t)} \nu_t \leq \fL_{\lambda}^*(x_0) - \fL_{\lambda}^*(x_T) + \frac{D_3 \eta_x^2 T}{2}.
        \label{nu_upper}
    \end{align}
    Let $\psi_t \coloneqq \BE[\nu_t \mid \fF_{t-1}]$ and denote $\{\Theta_t \coloneqq -\eta_x (\nu_t - \psi_t)\Norm{\nabla \fL_{\lambda}^*(x_t)}\}_{t \in \BN}$ as the martingale difference sequence with respect to $\{\fF_t\}_{t \in \BN}$.
    Note that $|\nu_t|\leq 1$, then we have
    \begin{align*}
        \exp\left(\frac{\Theta_t^2}{4 \eta_x^2 \Norm{\nabla \fL_{\lambda}^*(x_t)}^2}\right) = \exp \left( \frac{(\nu_t - \psi_t)^2}{4} \right) \leq {\rm e}.
    \end{align*}
    Consequently, we apply Lemma \ref{lemma:hp_concentrate} with $\Theta_t= -\eta_x (\nu_t - \psi_t)\Norm{\nabla \fL_{\lambda}^*(x_t)}$, $\sigma_t = 2 \eta_x \Norm{\nabla \fL_{\lambda}^*(x_t)}$, and $\hat\delta=\delta/2$ to obtain that
    \begin{align*} 
        \sum_{t=1}^T \eta_x (\psi_t - \nu_t)\Norm{\nabla \fL_{\lambda}^*(x_t)} \leq & 3 \chi \eta_x^2 \sum_{t=1}^T  \Norm{\nabla \fL_{\lambda}^*(x_t)}^2 + \frac{1}{\chi} \ln \left(\frac{2}{\delta}\right)
    \end{align*}
    holds with probability at least $1 - \delta / 2$  for all $\chi > 0$.
    Rearranging above inequality, we have
    \begin{align}\label{eq:high-L-lambda}
    \begin{split}        
         & \sum_{t=0}^{T-1} \eta_x (\psi_t - 3 \chi \eta_x \Norm{\nabla \fL_{\lambda}^*(x_t)}) \Norm{\nabla \fL_{\lambda}^*(x_t)} \\
         \leq & \sum_{t=1}^T \eta_x \nu_t\Norm{\nabla \fL_{\lambda}^*(x_t)} + \frac{1}{\chi} \ln \left(\frac{2}{\delta}\right) 
         \leq \fL_{\lambda}^*(x_0) - \fL_{\lambda}^*(x_T) + \frac{D_3 \eta_x^2 T}{2} + \frac{1}{\chi}\ln\left(\frac{2}{\delta}\right)
    \end{split}
     \end{align}
    holds with probability at least $1 - \delta / 2$, where the last step is based on \eqref{nu_upper}.
    In addition, the smoothness of $\fL_{\lambda}^*$ shown in Lemma \ref{lemma:smooth_bound} implies
    \begin{align}\label{L_lambda_smooth}
        \Norm{\nabla \fL_{\lambda}^*(x_t)} \leq \Norm{\nabla \fL_{\lambda}^*(x_0)} + D_3 t \eta_x.
    \end{align}
    Then with probability at least $1 - \delta / 2$, we have
    \begin{align}\label{bilevel_high_prob_interm_1}
    \begin{split}        
        & \sum_{t=0}^{T-1} \eta_x \left(\psi_t - \frac{1}{2}\right) \Norm{\nabla \fL_{\lambda}^*(x_t)} 
        \\
        \leq & \sum_{t=0}^{T-1} \eta_x \left(\psi_t - \frac{\Norm{\nabla \fL_{\lambda}^*(x_t)}}{2(\Norm{\nabla \fL_{\lambda}^*(x_0)} + \eta_x TD_3)}\right) \Norm{\nabla \fL_{\lambda}^*(x_t)} \\ 
        \leq & \fL_{\lambda}^*(x_0) - \fL_{\lambda}^*(x_T) + \frac{D_3 \eta_x^2 T}{2} 
        + 6 (\eta_x \Norm{\nabla \fL_{\lambda}^*(x_0)} + \eta_x^2 T D_3)\ln\left(\frac{2}{\delta}\right),
    \end{split}
    \end{align}
    where the first inequality follows from (\ref{L_lambda_smooth}) and the second inequality is due to \eqref{eq:high-L-lambda} with
     $\chi = 1/(6 (\eta_x \Norm{\nabla \fL_{\lambda}^*(x_0)} + \eta_x^2 TD_3))$.
    
    By applying Lemma \ref{lemma:lower_bound} with $a=\nabla \fL_{\lambda}^*(x_t)$ and $b=g_{x,t}$, we get
    
    \begin{align}
        \psi_t \Norm{\nabla \fL_{\lambda}^*(x_t)} \geq \Norm{\nabla \fL_{\lambda}^*(x_t)} - 2 {\BE[\Norm{g_{x,t} - \nabla \fL_{\lambda}^*(x_t)}\mid \fF_{t-1}]},
        \label{bilevel_high_prob_interm_2}
    \end{align}
    which implies
    \begin{align*}
        &\sum_{t=0}^{T-1}  \frac{\eta_x}{2} \Norm{\nabla \fL_{\lambda}^*(x_t)} - 2 \eta_x {\BE[\Norm{g_{x,t} - \nabla \fL_{\lambda}^*(x_t)} \mid \fF_{t-1}]} \\
       \leq & \sum_{t=0}^{T-1} \eta_x \left(\psi_t - \frac{1}{2}\right) \Norm{\nabla \fL_{\lambda}^*(x_t)} 
        \\
        \leq & \fL_{\lambda}^*(x_0) - \fL_{\lambda}^*(x_T) + \frac{D_3 \eta_x^2 T}{2} 
        + 6 (\eta_x \Norm{\nabla \fL_{\lambda}^*(x_0)} + \eta_x^2 T D_3)\ln\left(\frac{2}{\delta}\right),
    \end{align*}
    where the last step is based on (\ref{bilevel_high_prob_interm_1}).
    Rearranging above inequality, we have
    \begin{align*}
        & \frac{1}{2}\sum_{t=0}^{T-1} \eta_x \Norm{\nabla \fL_{\lambda}^*(x_t)} \\
        \leq & \fL_{\lambda}^*(x_0) - \fL_{\lambda}^*(x_T) + 2 \eta_x \sum_{t=0}^{T-1} {\BE[\Norm{g_{x, t} - \nabla \fL_{\lambda}^*(x_t)}\mid \fF_{t-1}]} + \frac{D_3 \eta_x^2 T}{2} \\ 
       & + 6 (\eta_x \Norm{\nabla \fL_{\lambda}^*(x_0)}+ \eta_x^2 T D_3)\ln\left(\frac{2}{\delta}\right) \\
        \leq & \fL_{\lambda}^*(x_0) - \fL_{\lambda}^*(x_T) + 2 \eta_x \sum_{t=0}^{T-1} {\BE\left[\Norm{g_{x, t} - \nabla_x \tilde{\fL}_{\lambda}(x_t, y_t, z_t)}\mid \fF_{t-1}\right] }\\
        & +  2 \eta_x \sum_{t=0}^{T-1} \Norm{\nabla_x \tilde{\fL}_{\lambda}(x_t, y_t, z_t) - \nabla \fL_{\lambda}^*(x_t)} + \frac{D_3 \eta_x^2 T}{2} + 6 (\eta_x \Norm{\nabla {\fL}_{\lambda}^*(x_0)}+ \eta_x^2 T D_3)\ln\left(\frac{2}{\delta}\right).
    \end{align*}
    
    where $\tilde{\fL}_{\lambda}(x,y,z) = f(x,y) + \lambda \left(g(x,y) - g(x,z)\right)$. Here, we use triangle inequality in the last step. 
    Dividing both sides by $\eta_x T/ 2$ leads to the desired result.
\end{proof}
Under an appropriate choice of $\eta_x$, we can derive the following  result.
\begin{lem}
Under Assumption \ref{asm:bilevel} and \ref{asm:pBCM}, we run N$^2$SBA (Algorithm~\ref{alg:nsgd_bilevel}) with $\lambda \geq 2 L_f / \mu$ and $\eta_x = \sqrt{\Delta / D_3 T}$. 
Then with probability at least $1 - \delta/2$ for all $\delta \in (0,1)$, it holds
     \begin{align*}
        &\frac{1}{T} \sum_{t=0}^{T-1} \Norm{\nabla \fL_{\lambda}^*(x_t)}\\
        \leq & \frac{\left( 7 + 50 \ln\left({2}/{\delta}\right) \right) \sqrt{\Delta D_3} }{ \sqrt{T}} +  \frac{8 \sigma_f + 16 \lambda \sigma_g}{M^{\frac{p-1}{p}}} +  \frac{4}{T}\sum_{t=0}^{T-1} \Norm{\nabla_x \tilde{\fL}_{\lambda}(x_t, y_t, z_t) - \nabla \fL_{\lambda}^*(x_t)}.
    \end{align*}
    \label{lemma:bilevel_core_high_prob}
\end{lem}
\begin{proof}
    We follow the notation $\tilde{\fL}_{\lambda}(x,y,z) = f(x,y) + \lambda \left(g(x,y) - g(x,z)\right)$ in the proof of Lemma \ref{high_prob_core_bound}.
    Note that $x_t$ and~$y_t$ are $\fF_{t-1}$ measurable and $\xi_{t,1}^t, \dots, \xi_{t,M}^t$, $\zeta_{t,1}^t, \dots, \zeta_{t,M}^t$ are independent of $\fF_{t-1}$.
    Hence, we apply Lemma~\ref{lemma:bound_var} to get
    \begin{align*}
        &\BE\left[\Norm{g_{x,t} -\nabla_x \tilde{\fL}_{\lambda} (x_t, y_t, z_t)} \mid \fF_{t-1}\right]  =  \BE\left[\Norm{g_{x,t} -\nabla_x \tilde{\fL}_{\lambda} (x_t, y_t, z_t)} \mid x_t, y_t, z_t\right] \\
        \leq & \BE\left[ \frac{1}{M} \sum_{i=1}^M\Norm{\nabla_x F(x_t, y_t; \xi_{t,i})  - \nabla_x f(x_t, y_t)}| x_t, y_t, z_t\right] \\
        &+ \lambda \BE\left[ \frac{1}{M} \sum_{i=1}^M \Norm{\nabla_x G(x_t, y_t; \zeta_{t,i})  - \nabla_x g(x_t, y_t)}| x_t, y_t, z_t\right] \\
        & + \lambda \BE \left[ \frac{1}{M} \sum_{i=1}^M \Norm{\nabla_x G(x_t, z_t; \zeta_{t,i})  - \nabla_x g(x_t, z_t)}| x_t, y_t, z_t\right] \\
        \leq & \frac{2 \sigma_f}{M^{\frac{p-1}{p}}} + \frac{4 \lambda \sigma_g}{M^{\frac{p-1}{p}}},
    \end{align*}
    where the first inequality follows the triangle inequality and the second inequality is due to Lemma \ref{lemma:bound_var} and Assumption \ref{asm:pBCM}.
    Plugging the above inequality into (\ref{high_prob_core_bound}) yields
    \begin{equation}
    \begin{split}
        \frac{1}{T} \sum_{t=0}^{T-1} \Norm{\nabla \fL_{\lambda}^*(x_t)} \leq & \frac{ 2 (\fL_{\lambda}^*(x_0) - \fL_{\lambda}^*(x_T))}{\eta_x T} +  \frac{8 \sigma_f + 16 \lambda \sigma_g}{M^{\frac{p-1}{p}}} +  4\sum_{t=0}^{T-1} \frac{\Norm{\nabla_x \tilde{\fL}_{\lambda}(x_t, y_t, z_t) - \nabla \fL_{\lambda}^*(x_t)}}{T}  \\
        &   + D_3 \eta_x  + 12 \left( \frac{\Norm{\nabla \fL_{\lambda}^*(x_0)}}{T} + \eta_x D_3\right)\ln\left(\frac{2}{\delta}\right) \\
        \leq & \frac{7 \sqrt{\Delta D_3} }{ \sqrt{T}} +  \frac{8 \sigma_f + 16 \lambda \sigma_g}{M^{\frac{p-1}{p}}} +  \sum_{t=0}^{T-1} \frac{4\Norm{\nabla_x \tilde{\fL}_{\lambda}(x_t, y_t, z_t) - \nabla \fL_{\lambda}^*(x_t)}}{T}  \\
        &    + 12 \left( \frac{\Norm{\nabla \fL_{\lambda}^*(x_0)}}{T} + \frac{\sqrt{\Delta D_3}}{\sqrt{T}}\right)\ln\left(\frac{2}{\delta}\right) \\
        \leq & \left( 7 + 50 \ln\left(\frac{2}{\delta}\right) \right)\frac{ \sqrt{\Delta D_3} }{ \sqrt{T}} +  \frac{8 \sigma_f + 16 \lambda \sigma_g}{M^{\frac{p-1}{p}}} \\
         & +  \sum_{t=0}^{T-1} \frac{4\Norm{\nabla_x \tilde{\fL}_{\lambda}(x_t, y_t, z_t) - \nabla \fL_{\lambda}^*(x_t)}}{T}
    \end{split}
    \label{bilevel_high_prob_main_1}
    \end{equation}
    holds with probability at least $1 - \delta / 2$, where the second inequality is based on the setting $\eta_x = \sqrt{\Delta/D_3 T}$ and the fact $\fL_{\lambda}^*(x_0) - \fL_{\lambda}^*(x_T) \leq 3 \Delta$ achieved by Lemma \ref{lemma:bilevel_relationship}; the last step is based on the fact $\Norm{\nabla\fL_\lambda^*(x_0)}\leq\sqrt{2\Delta D_3}$.
\end{proof}
We now provide the high-probability result to bound $\Norm{y_t - y_\lambda^*(x_t)}$ and $\Norm{{z}_{t} - y^*(x_{t})}$.
\begin{lem}
    Following the setting of Theorem \ref{thm:bilevel_high_prob},  it holds 
    \begin{align*}
    \Norm{{y}_{t} - y_\lambda^*(x_{t})} = \fO\left(\frac{\epsilon^2}{\ell^2 \kappa^3}\right)~~~{\rm and}~~~\Norm{{z}_{t} - y^*(x_{t})} 
   = \fO\left( \frac{\epsilon^2}{\ell^2 \kappa^3}\right)
\end{align*}
with probability at least $1 - \delta/ 2$ for all $t=0,\dots,T-1$ and $\delta\in(0,1)$.
\label{lemma:high_prob_y_z_bound}
\end{lem}
\begin{proof}
     We specify the parameter settings for N$^2$SBA as
    \begin{align*}
        & \eta_x = \sqrt{\frac{\Delta}{\ell \kappa^3 T}},~~~\eta_{y,t} = \min \left\{\frac{1}{400 (L_f + \lambda L_g) \ln\left({16 T K}/{\delta}\right)}, \frac{2 \ln (B_K)}{\lambda \mu K}\right\},\\
        &\eta_{z, t} = \min \left\{\frac{1}{400 \lambda L_g \ln\left({16 T K}/{\delta}\right)}, \frac{\ln (B_K)}{\lambda \mu K}\right\},~~~K = \tilde\fO\left(\frac{\ell^{\frac{p}{p-1}} \kappa^{\frac{4p}{p-1}} \sigma^{\frac{p}{p-1}}}{\epsilon^{\frac{2p}{p-1}}}\right),\\ 
        &  \lambda = \max\left\{\frac{ \kappa}{ \sqrt{R}}, \frac{\ell \kappa^2}{ \Delta}, \frac{\ell \kappa^3}{ \epsilon}\right\},  ~~~\tau_{t,k} = \frac{\exp(- \eta_{z,t} \lambda \mu (1 + k/ 2) ){ R_z}}{120 \eta_{z,t} \ln\left({16 T K}/{\delta}\right)},\\
        & \tau_{t,k}' = \frac{\exp(- \eta_{y,t} \lambda \mu (1/2 + k/ 4) ){ R_y}}{120 \eta_{y,t} \ln\left({16 T  K}/{\delta}\right)},
    \end{align*}
     where we set 
    \begin{align*}
       &B_{y, t} = \max\left\{2, \frac{\lambda^2\mu^2 K^{ \frac{2 (p - 1)}{p}} R_y^2}{5400^{\frac{2}{p}} (\sigma_f +\lambda\sigma_g)^2 (\ln (B_{y, t}))^2 \left(\ln\left({16 T K }/{\delta}\right)\right)^{\frac{2(p-1)}{p}}}\right\},\\
       &B_{z, t} = \max\left\{2, \frac{\mu^2 K^{ \frac{2 (p - 1)}{p}} R_z^2}{5400^{\frac{2}{p}} \sigma_g^2 (\ln (B_{z, t}))^2 \left(\ln\left({16 T K }/{\delta}\right)\right)^{\frac{2(p-1)}{p}}}\right\},\\
       & R_y^2 \coloneqq 4R_0^2 + \frac{2 \epsilon^4}{\ell^4 \kappa^6} + \frac{32 \epsilon^2}{\ell^2 \kappa^4},~~~ R_{z}^2 \coloneqq R_0^2 +  \frac{2 \epsilon^4}{\ell^4 \kappa^6} + \frac{2 \epsilon^2}{\ell^2 \kappa^4},\\
        & R_0 \geq \Norm{\hat{y}_{0,0} - y^*(x_0)}^2,~~~\Delta = \varphi(x_0) - \inf_{x \in \BR^{d_x}} \varphi(x),
    \end{align*}

 We first consider the term $\Norm{y_t - y_{\lambda}^*(x_t)}$.    
    Recall that Algorithm \ref{alg:nsgd_bilevel} set $y_t=\hat{y}_{t, K}$ in line~\ref{line:y_t}, where $\hat{y}_{t, K}$ can be regarded as the output of applying clipped stochastic gradient descent (lines \ref{bilevel_y_start}--\ref{bilevel_y_end}) to minimize the function $\fL_{\lambda} (x_t, y)=f(x_t, y) + \lambda(g(x_t, y) - g^*(x_t))$ with respect to~$y$. The updates use the stochastic gradient estimate 
    $\nabla_y F(x_t, y; \xi'_t) + \lambda \nabla_y G(x_t, y; \zeta'_t)$ for~$\fL_{\lambda} (x_t, y)$. 
    Since the setting of $\lambda$ and Lemma~\ref{lemma:lambda_strongly_convex} implies the function $\fL_{\lambda} (x_t, y)$ is $(\lambda \mu / 2)$-strongly convex in~$y$, we apply Lemma \ref{lemma:high_prob_strongly_convex_converg} with $h(y) = \fL_\lambda(x_t,y)$, $\sigma_h = 2(\sigma_f + \lambda \sigma_g)$, $\ell_h= 2 \lambda L_g$, $\mu_h= \lambda \mu/2$, and $\hat\delta = \delta/(4T)$ to obtain
    \begin{align}\label{eq:dis_t_1}
    \begin{split}    
          \Norm{\hat{y}_{t,K} - y_{\lambda}^*(x_t)}^2
    \leq & 2 {\hat R}_{y,t}^2 \exp\left(- \frac{\mu K}{1600 L_g  \ln\left({16 T K}/{\delta}\right)}\right)  \\
    & + \frac{32 \cdot 5400^{\frac{2}{p}} (\sigma_f + \lambda \sigma_g)^2 (\ln (B_{y,t}))^2 \left(\ln\left({16 T K }/{\delta}\right)\right)^{\frac{2(p-1)}{p}} }{\lambda^2\mu^2 K^{ \frac{2 (p - 1)}{p}}} ,
    \end{split}
    \end{align}
    for all given ${\hat R}_{y,t}^2 \geq \Norm{\hat{y}_{t, 0} - y_\lambda^*(x_{t})}^2$ with probability $1-\delta / (4T)$ for all $t=0, 1, \dots, T-1$. Here, the choice of $\sigma_h = 2(\sigma_f + \lambda \sigma_g)$ that corresponds to the upper bound for the $p$-th central moment shown in Lemma \ref{lemma:var_bound}.

    We then use induction to show that $\Norm{\hat{y}_{0, 0} - y_\lambda^*(x_{0})}^2,\dots,\Norm{\hat{y}_{0, t} - y_\lambda^*(x_{t})}^2$ have a uniform upper bound for all $t=0,\dots,T-1$ with high-probability.
    In particular, we define the event~$E_t$ as
    \begin{equation}
    \Norm{\hat{y}_{l,0} - y^*_{\lambda}(x_l)}^2 \leq R_y^2 \coloneqq 4R_0^2 + \frac{2 \epsilon^4}{\ell^4 \kappa^6} + \frac{32 \epsilon^2}{\ell^2 \kappa^4}
    \label{y_induction_1}
    \end{equation}
    holds for all $l = 0, 1, \dots, t$, and target to show that $\BP(E_t) \geq 1 - t \delta / (4T)$ holds for all $t = 0, \ldots, T-1$ by induction. 
For the induction base $t=0$, we follow the derivation of 
\eqref{eq:induction-dist-000} to achieve
\begin{align*}
\Norm{\hat{y}_{0,0} - y^*_{\lambda}(x_0)}^2 \leq 2 \Norm{\hat{y}_{0,0} - y^*(x_0)}^2 + 2\Norm{y^*(x_0) - y_\lambda^*(x_0)}^2 \leq 2 R_0^2 + 2 R_0^2 = 4 R_0^2,   
\end{align*}
i.e., $\BP(E_0) = 1$.
For the induction step, we suppose that $\BP(E_{t}) \geq 1 - t \delta / (4 T)$ holds for all $t=0,1,\dots,T'-1$. 
In the case of $t= T'$, we follow the derivation of \eqref{eq:induction-dist-00} to achieve
\begin{align}\label{eq:induction-dist-001}
\begin{split}
    \Norm{\hat{y}_{T',0} - y_{\lambda}^*(x_{T'})}^2 \leq  2 \Norm{\hat{y}_{T'-1,K} - y_{\lambda}^*(x_{T'-1})}^2 +  \frac{32 L_g^2 \eta_x^2}{\mu^2}.
\end{split}    
\end{align}
Following the derivation of \eqref{eq:dis_t_1}, it holds
\begin{align}\label{eq:induction-T-11}
\begin{split}
    \!\!\!\Norm{\hat{y}_{T'-1,K} - y_{\lambda}^*(x_{T'-1})}^2
    \leq & 2 {\hat R}_{y,T'-1}^2 \exp\left(- \frac{\mu K}{1600 L_g  \ln\left({16 T K}/{\delta}\right)}\right)  \\
    & + \frac{32 \cdot 5400^{\frac{2}{p}} (\sigma_f + \lambda \sigma_g)^2 (\ln (B_{y,T'-1}))^2 \left(\ln\left({16 T K }/{\delta}\right)\right)^{\frac{2(p-1)}{p}} }{\lambda^2\mu^2 K^{ \frac{2 (p - 1)}{p}}},
\end{split}
\end{align}
with probability at least $1 - \delta/ (4T)$ for all given ${\hat R}_{y,T'-1}^2 \geq \Norm{\hat{y}_{T'-1, 0} - y_\lambda^*(x_{T'-1})}^2$. 
In addition, the induction hypothesis $\BP(E_{T'-1}) \geq 1 - (T'-1) \delta / (4 T)$ implies we can 
apply \eqref{eq:induction-T-11} by
taking ${\hat R}_{y,T'-1}^2=R_y^2$ 
and the union bound on events $E_{T'-1}$ and (\ref{eq:induction-T-11}) to guarantee
\begin{align*}
       \Norm{\hat{y}_{T'-1,K} - y_{\lambda}^*(x_{T'-1})}^2
    \leq & 2 R_y^2 \exp\left(- \frac{\mu K}{1600 L_g  \ln\left({16 T K}/{\delta}\right)}\right)  \\
    & + \frac{32 \cdot 5400^{\frac{2}{p}} (\sigma_f + \lambda \sigma_g)^2 (\ln (B_{y,T'-1}))^2 \left(\ln\left({16 T K }/{\delta}\right)\right)^{\frac{2(p-1)}{p}} }{\lambda^2\mu^2 K^{ \frac{2 (p - 1)}{p}}} 
\end{align*}
holds with probability at least $1 - T'\delta / (4 T)$.
Combining above inequality with \eqref{eq:induction-dist-001}, we guarantee
\begin{align*}
 \Norm{\hat{y}_{T',0} - y_{\lambda}^*(x_{T'})}^2
\leq & 4 R_y^2 \exp\left(- \frac{\mu K}{1600 L_g  \ln\left({16 T K}/{\delta}\right)}\right)  \\
    & + \frac{64 \cdot 5400^{\frac{2}{p}} (\sigma_f + \lambda \sigma_g)^2 (\ln (B_{y,T'-1}))^2 \left(\ln\left({64 T K }/{\delta}\right)\right)^{\frac{2(p-1)}{p}} }{\lambda^2\mu^2 K^{ \frac{2 (p - 1)}{p}}}  +  \frac{32 L_g^2 \eta_x^2}{\mu^2} \\
\leq & \frac{2 \epsilon^4}{\ell^4 \kappa^6}  + \frac{32 L_g^2 \eta_x^2}{\mu^2} 
\leq  \frac{2 \epsilon^4}{\ell^4 \kappa^6} + \frac{32 \epsilon^2}{\ell^2 \kappa^4} 
    \leq  R_y^2, 
\end{align*}
holds with probability at least $1 - T' \delta/(4T)$, i.e., $\BP(E_{T'}) \geq 1 - T' \delta/ (4 T) $,
where the second inequality is based on the setting of $K$,
the third inequality is based on the setting of $\eta_x$,
and the last inequality is based on the definition of $R_y^2$.
This completes the induction.

Therefore, we can apply Lemma \ref{lemma:high_prob_strongly_convex_converg} on $\Norm{\hat{y}_{l, K} - y_\lambda^*(x_{l})}^2$ with $\hat{R}_{y,l} = R_y$ 
and the union bound on events $E_t$ and \eqref{eq:dis_t_1}
to guarantee
\begin{align*}
\begin{split}    
   & \Norm{{y}_{l} - y_\lambda^*(x_{l})}^2 = \Norm{\hat{y}_{l, K} - y_\lambda^*(x_{l})}^2 \\ 
   \leq & 2 R_y^2 \exp\left(- \frac{\mu K}{1600 L_g  \ln\left({16 T K}/{\delta}\right)}\right)  \\
    & + \frac{32 \cdot 5400^{\frac{2}{p}} (\sigma_f + \lambda \sigma_g)^2 (\ln (B_{y,l}))^2 \left(\ln\left({16 T K }/{\delta}\right)\right)^{\frac{2(p-1)}{p}} }{\lambda^2\mu^2 K^{ \frac{2 (p - 1)}{p}}} 
   = \fO\left(\frac{\epsilon^4}{\ell^4 \kappa^6}\right),
\end{split}
\end{align*}
holds for all $l = 0, \dots, t$ with probability at least $1 - (t+1) \delta / (4 T)$,
which implies 
\begin{align}\label{high_prob_y_bound}
    \Norm{{y}_{t} - y_\lambda^*(x_{t})} = \fO\left(\frac{\epsilon^2}{\ell^2 \kappa^3}\right),
\end{align}
holds for all $t = 0,1, \dots, T-1$ with probability at least $1 - \delta/ 4$.

    Next, we consider the term $\Norm{z_t - y^*(x_t)}$.    
    Recall that Algorithm \ref{alg:nsgd_bilevel} set $z_t=\hat{z}_{t, K}$ in line~\ref{line:z_t}, where $\hat{z}_{t, K}$ can be regarded as the output of applying clipped stochastic gradient descent (lines \ref{bilevel_z_start}--\ref{bilevel_z_end}) to minimize the function $\lambda g(x_t, y)$ with respect to~$y$. The updates use the stochastic gradient estimate 
    $\lambda \nabla_y G(x_t, y; \zeta_t)$ for $\lambda g(x_t, y)$. 
    Since the function $\lambda g(x_t, y)$ is $\lambda \mu$-strongly convex in~$y$, we apply Lemma \ref{lemma:high_prob_strongly_convex_converg} with $h(y) = \lambda g(x_t,y)$, $\sigma_h = \lambda \sigma_g$, $\ell_h=  \lambda L_g$, $\mu_h= \lambda \mu$, and $\hat{\delta} = \delta / (4T)$ to obtain
\begin{align}\label{eq:dis_zt_00}
\begin{split}
     \Norm{\hat{z}_{t,K} - y^*(x_{t})}^2 
    \leq & 2 \hat{R}_{z,t}^2 \exp\left(- \frac{\mu K}{400 L_g \ln\left({16 T  K }/{\delta}\right)}\right)  \\
     & +\frac{2\cdot5400^{\frac{2}{p}} \sigma_g^2 (\ln (B_{z,t}))^2 \left(\ln\left({16 T K  }/{\delta}\right)\right)^{\frac{2(p-1)}{p}}}{\mu^2 K^{ \frac{2 (p - 1)}{p}} }
     \end{split}
\end{align}
with probability at least $1 - \delta/(4T)$ for all given ${\hat R}_{z,t}^2 \geq \Norm{\hat{z}_{t, 0} - y^*(x_{t})}^2$.

We then use induction to show that 
$\Norm{\hat{z}_{0, 0} - y^*(x_{0})}^2,\dots,\Norm{\hat{z}_{0, t} - y^*(x_{t})}^2$ have a uniform upper bound for all $t=0,\dots,T-1$ with high-probability.
In particular, we define the event~$E_t'$ as
    \begin{equation}
    \Norm{\hat{z}_{l,0} - y^*(x_l)}^2 \leq R_z^2 \coloneqq R_0^2 +  \frac{2 \epsilon^4}{\ell^4 \kappa^6} + \frac{2 \epsilon^2}{\ell^2 \kappa^4}
    \label{z_induction_1}
    \end{equation}
    holds for all $l = 0, 1, \dots, t$, and target to show that $\BP(E_t') \geq 1 - t \delta / (4 T)$ holds for all $t = 0, \ldots, T-1$ by induction.     
    For the induction base $t=0$, we follow the derivation of 
\eqref{z_induction_base_case} to achieve 
\begin{align*}
\Norm{\hat{z}_{0, 0} - y^*(x_0)}^2= \Norm{\hat{y}_{0, 0} - y^*(x_0)}^2\leq R_0^2 \leq R_{z}^2,   
\end{align*}
i.e., $\BP(E_0') = 1$.
For the induction step, we suppose that $\BP(E'_t) \geq 1 - t \delta/(4T)$ holds for all $t=0,\dots,T-1$. In the case of $t={T'}$, we follow the derivation of \eqref{eq:induction-dist-z0} to achieve
\begin{align}\label{eq:induction-dist-z0_1}
\begin{split}
   \Norm{\hat{z}_{T',0} - y^*(x_{T'})}^2 \leq  2 \Norm{\hat{z}_{T'-1, K} - y^*(x_{T'-1})}^2 + \frac{2 L_g^2 \eta_x^2}{\mu^2} ,
\end{split}    
\end{align}
Following the derivation of \eqref{eq:dis_zt_00}, it holds
\begin{align}\label{eq:induction-z_T-2}
\begin{split}
      \Norm{\hat{z}_{T'-1,K} - y^*(x_{T'-1})}^2 
    \leq & 2 \hat{R}_{z,T'-1}^2 \exp\left(- \frac{\mu K}{400 L_g \ln\left({16 T  K }/{\delta}\right)}\right)  \\
     & +\frac{2\cdot5400^{\frac{2}{p}} \sigma_g^2 (\ln (B_{z,T'-1}))^2 \left(\ln\left({16 T K  }/{\delta}\right)\right)^{\frac{2(p-1)}{p}}}{\mu^2 K^{ \frac{2 (p - 1)}{p}} }
\end{split}
\end{align}
with probability at least $1 - \delta/(4T)$ for all given ${\hat R}_{z,T'-1}^2 \geq \Norm{\hat{z}_{T'-1, 0} - y^*(x_{T'-1})}^2$. 
In addition, the induction hypothesis $\BP(E'_{T'-1}) \geq 1 - (T'-1) \delta/(4T)$ implies we can 
apply \eqref{eq:induction-z_T-2} by
taking ${\hat R}_{z,T'-1}^2=R_z^2$ and the union bound on events $E'_{T'-1}$ and (\ref{eq:induction-z_T-2}) to guarantee
\begin{align*}
   \Norm{\hat{z}_{T'-1,K} - y^*(x_{T'-1})}^2 
    \leq & 2 R_z^2 \exp\left(- \frac{\mu K}{400 L_g \ln\left({16 T  K }/{\delta}\right)}\right)  \\
     & +\frac{2\cdot5400^{\frac{2}{p}} \sigma_g^2 (\ln (B_{z,T'-1}))^2 \left(\ln\left({16 T K  }/{\delta}\right)\right)^{\frac{2(p-1)}{p}}}{\mu^2 K^{ \frac{2 (p - 1)}{p}} }
\end{align*}
holds with probability at least $1 - T' \delta/ (4T)$.
Combining above inequality with \eqref{eq:induction-dist-z0_1}, we guarantee
\begin{align*}
\Norm{\hat{z}_{T',0} - y^*(x_{T'})}^2
    \leq &  4 R_z^2 \exp\left(- \frac{\mu K}{400 L_g \ln\left({16 T  K }/{\delta}\right)}\right)  \\
     & +\frac{4\cdot5400^{\frac{2}{p}} \sigma_g^2 (\ln (B_{z,T'-1}))^2 \left(\ln\left({16 T K  }/{\delta}\right)\right)^{\frac{2(p-1)}{p}}}{\mu^2 K^{ \frac{2 (p - 1)}{p}} } + \frac{2 L_g^2 \eta_x^2}{\mu^2}\\
\leq &  \frac{2 \epsilon^4}{\ell^4 \kappa^6} + \frac{2 L_g^2 \eta_x^2}{\mu^2}
\leq  \frac{2 \epsilon^4}{\ell^4 \kappa^6} + \frac{2 \epsilon^2}{\ell^2 \kappa^4} 
    \leq  R_z^2, 
\end{align*}
holds with probability at least $1 - T' \delta / (4 T)$, i.e., $\BP(E'_{T'}) \geq 1 - T' \delta/ (4 T) $.
Here, the second inequality is based on the setting of~$K$,
the third inequality is based on the setting of $\eta_x$,
and the last inequality is based on the definition of $R_z^2$.
This completes the induction.

Therefore, we can apply Lemma \ref{lemma:high_prob_strongly_convex_converg} on $\Norm{\hat{z}_{l, K} - y^*(x_{l})}^2$ with $\hat{R}_{z,l} = R_z$ 
and the union bound on events $E'_t$ and (\ref{eq:dis_zt_00})
to guarantee
\begin{align*}
\begin{split}    
   & \Norm{{z}_{l} - y^*(x_{l})}^2 = \Norm{\hat{z}_{l, K} - y^*(x_{l})}^2 \\ 
   \leq & 2 R_z^2 \exp\left(- \frac{\mu K}{400 L_g \ln\left({16 T  K }/{\delta}\right)}\right)  \\
     & +\frac{2\cdot5400^{\frac{2}{p}} \sigma_g^2 (\ln (B_{z,l}))^2 \left(\ln\left({16 T K  }/{\delta}\right)\right)^{\frac{2(p-1)}{p}}}{\mu^2 K^{ \frac{2 (p - 1)}{p}} }
   = \fO\left( \frac{\epsilon^4}{\ell^4 \kappa^6}\right).
\end{split}
\end{align*}
holds for all $l=0,\dots, t$ with probability at least $1 - (t+1) \delta / (4 T)$, which implies
\begin{align}\label{high_prob_z_bound}   
   & \Norm{{z}_{t} - y^*(x_{t})} 
   = \fO\left( \frac{\epsilon^2}{\ell^2 \kappa^3}\right)
\end{align}   
holds  for all $t=0,\dots,T-1$ with probability at least $1 -\delta / 4 $.

Applying the union bound on events (\ref{high_prob_y_bound}) and (\ref{high_prob_z_bound}) leads to the desired result.

\end{proof}

\subsubsection{The Proof of Theorem \ref{thm:bilevel_high_prob}}


\begin{proof}
    We follow the parameter settings in the proof of Lemma \ref{lemma:high_prob_y_z_bound} and additionally set
    \begin{align*}
        T = \tilde\fO\left(\frac{\Delta \ell \kappa^3}{ \epsilon^2}\right) ~~~ {\rm and} ~~~{ M} = \fO\left(\frac{\ell^{\frac{p}{p-1}} \sigma^{\frac{p}{p-1}} \kappa^{\frac{3p}{p-1}}}{\epsilon^{\frac{2p}{p-1}}}\right).
    \end{align*}
    According to Lemma \ref{lemma:bilevel_core_high_prob}, we have
    \begin{equation}
    \begin{split}
        \!\!\!&\frac{1}{T} \sum_{t=0}^{T-1} \Norm{\nabla \fL_{\lambda}^*(x_t)} \\
        \!\!\!\leq & \frac{\left( 7 + 50 \ln\left({2}/{\delta}\right) \right)\sqrt{\Delta D_3} }{ \sqrt{T}} +  \frac{8 \sigma_f + 16 \lambda \sigma_g}{M^{\frac{p-1}{p}}} +  \frac{4}{T}\sum_{t=0}^{T-1} \Norm{\nabla_x \tilde{\fL}_{\lambda}(x_t, y_t, z_t) - \nabla_x \fL_{\lambda}^*(x_t)} \\
        \!\!\!\leq & \frac{\left( 7 + 50 \ln\left({2}/{\delta}\right) \right)\sqrt{\Delta D_3} }{ \sqrt{T}} +  \frac{8 \sigma_f + 16 \lambda \sigma_g}{M^{\frac{p-1}{p}}} +  \frac{4\lambda L_g }{T}\sum_{t=0}^{T-1}(2\Norm{y_t - y_\lambda^*(x_t)} + \Norm{z_t - y^*(x_t)})
        \end{split}\label{high_prob_converg_final}
     \end{equation}
    with probability $1 - \delta / 2$, where $\tilde{\fL}_{\lambda}(x,y,z) = f(x,y) + \lambda \left(g(x,y) - g(x,z)\right)$.
    Here, we achieve the last inequality in the derivation \eqref{high_prob_converg_final} as follows
    \begin{align*}
        & \Norm{\nabla_x \tilde{\fL}_{\lambda}(x_t, y_t, z_t) - \nabla_x \fL_{\lambda}^*(x_t)} \\
        \leq & \Norm{\nabla_x f(x_t,y_t) - \nabla_x f(x_t,y_\lambda^*(x_t))}
        + \lambda\Norm{\nabla_x g(x_t,y_t) - \nabla_x f(x_t,y_\lambda^*(x_t))} \\
        & + \lambda \Norm{\nabla_x g(x_t,z_t) - \nabla_x g(x_t,y^*_\lambda(x_t))} \\
        \leq & L_f\Norm{y_t - y_\lambda^*(x_t)}
        + \lambda L_g\Norm{y_t - y_\lambda^*(x_t)} 
         + \lambda L_g \Norm{z_t - y^*_\lambda(x_t)} \\
         \leq & 2\lambda L_g\Norm{y_t - y_\lambda^*(x_t)} + \lambda L_g \Norm{z_t - y^*_\lambda(x_t)},
    \end{align*}
    where the first step is based on triangle inequality, the second step is based on the smoothness of $f$ and $g$, and the last step is based on the setting $\lambda \geq 2L_f/\mu$.

    Similar to the analysis in the proof of Theorem \ref{thm:bilevel_expect}, 
    we only need to consider the task of finding an $\fO(\epsilon)$-stationary point of $\fL_{\lambda}^*$. 
    For the last line in \eqref{high_prob_converg_final},
    our settings of $T$ and $M$ guarantee 
    \begin{align}\label{high_prob_inter_1}
        \frac{\left( 7 + 50 \ln\left({2}/{\delta}\right) \right)\sqrt{\Delta D_3} }{ \sqrt{T}} = \fO(\epsilon)
        \qquad\text{and}\qquad 
        \frac{8 \sigma_f + 16 \lambda \sigma_g}{M^{\frac{p-1}{p}}} = \fO(\epsilon).
    \end{align}
    According to Lemma \ref{lemma:high_prob_y_z_bound}, we have
  \begin{align}\label{high_prob_inter_2}
    \lambda L_g\Norm{{y}_{t} - y_\lambda^*(x_{t})} = \fO\left(\epsilon\right)\qquad{\rm and}\qquad\lambda L_g\Norm{{z}_{t} - y^*(x_{t})} 
   = \fO\left( \epsilon\right)
\end{align}
 holds for all $t=0,1,\dots, T-1$  with probability at least $1 - \delta/ 2$.
 Applying the union bound on (\ref{high_prob_converg_final}), (\ref{high_prob_inter_1}), and (\ref{high_prob_inter_2}), it holds
 \begin{align*}
     \frac{1}{T} \sum_{t=0}^{T-1} \Norm{\nabla \fL_{\lambda}^*(x_t)} = \fO(\epsilon)
 \end{align*}
 with probability at least $1 - \delta$.
 In addition, the total SFO complexity is 
\begin{align*}
    TM + TK = \tilde\fO\left(\frac{\Delta\ell^{\frac{2p-1}{p-1}} \kappa^{\frac{7p-3}{p-1}}\sigma^{\frac{p}{p-1}}}{\epsilon^{\frac{4p-2}{p-1}}}\right).
\end{align*}

\end{proof}

\section{Proofs for Results of Stochastic Minimax Optimization}\label{sec:nonconvex_strongly_concave_proof}

In this section, we establish the convergence guarantee of the N$^2$SGDA algorithm for solving the nonconvex–strongly-concave minimax optimization problem. 
The analysis in this section follows the counterpart for bilevel optimization, 
while we can omit the update on $\hat z_{t,k}$ to achieve the sharper complexity bounds for minimax problem.

\subsection{The In-Expectation Convergence Guarantee}

In this subsection, we present the proofs for the in-expectation convergence guarantee of our N$^2$SGDA shown in Section \ref{sec:minimax_opt}.


\subsubsection{The Proof of Lemma \ref{lemma:core_recur}}
\begin{proof}
    According to Lemma~\ref{lemma:nonconvex_strongly_concave_prop}, the function $\Phi(x)$ is $(\ell + \kappa \ell)$-smooth and we have
    \begin{equation}
        \Phi(x_{t+1}) - \Phi(x_{t}) - \langle x_{t+1} - x_{t}, \nabla \Phi(x_{t}) \rangle \leq (\kappa + 1) \ell \Norm{x_{t+1} - x_{t}}^2.
    \end{equation}
    Plugging update $x_{t+1} - x_{t} = - \eta_x \frac{g_{x, t}}{\Norm{g_{x, t}}}$ into the above inequality yields
    \begin{align*}
        \Phi(x_{t+1}) \leq & \Phi(x_{t}) - \eta_x \left\langle \frac{g_{x, t}}{\Norm{g_{x, t}}}, \nabla \Phi(x_{t}) \right \rangle + \eta_{x}^2 (\kappa + 1) \ell \\
        \leq & \Phi(x_{t}) - \eta_x \Norm{\nabla \Phi(x_{t})} + 2 \eta_x \Norm{g_{x, t} - \nabla \Phi(x_{t})} + \eta_{x}^2 (\kappa + 1) \ell,
    \end{align*}
    where the second inequality follows from Lemma \ref{lemma:lower_bound}.
    By the triangle inequality, one has
    \begin{align*}
        \Norm{g_{x, t} - \nabla \Phi(x_{t})} \leq &\Norm{g_{x, t} - \nabla_x f(x_{t}, y_{t})} + \Norm{\nabla_x f(x_{t}, y_{t}) - \nabla \Phi(x_{t})} .
    \end{align*}
    Taking the expectation on both sides of the inequality, we have
    \begin{align*}
        \BE[\Norm{g_{x, t} - \nabla \Phi(x_{t})}] \leq \frac{2 \sigma}{M^{\frac{p-1}{p}}}  + \BE[\Norm{\nabla_x f(x_{t}, y_{t}) - \nabla \Phi(x_{t})}],
    \end{align*}
    where the above inequality is due to Lemma~\ref{lemma:bound_var}.
    Then it follows that
    \begin{align*}
    \begin{split}
        \BE[\Phi(x_{t+1})] 
        \leq & \BE[\Phi(x_{t})] - \eta_x \BE[\Norm{\nabla \Phi(x_{t})}] + 2\eta_x \BE[\Norm{g_{x, t} - \nabla \Phi(x_{t})}] + \eta_{x}^2 (\kappa + 1) \ell \\
        \leq & \BE[\Phi(x_{t})] - \eta_x \BE[\Norm{\nabla \Phi(x_{t})}] + \frac{4 \sigma \eta_x}{M^{\frac{p-1}{p}}} \\
        &+ 2 \eta_x \BE[\Norm{\nabla_x f(x_{t}, y_{t}) - \nabla \Phi(x_{t})}] + \eta_{x}^2 (\kappa + 1) \ell.
    \end{split}
    \end{align*}
    Summing the above inequality over $t=0,\dots,T-1$, we have
    \begin{align*}
    \begin{split}
        \BE[\Phi(x_{T})] \leq & \BE[\Phi(x_{0})] - \eta_x \sum_{t=0}^{T-1}\BE[\Norm{\nabla \Phi(x_{t})}] + \frac{4 \sigma \eta_x T}{M^{\frac{p-1}{p}}}\\
        &+ 2\eta_x \sum_{t=0}^{T-1} \BE[\Norm{\nabla_x f(x_{t}, y_{t}) - \nabla \Phi(x_{t})}] + \eta_{x}^2 (\kappa + 1)  \ell T \\
         \leq & \BE[\Phi(x_{0})] - \eta_x \sum_{t=0}^{T-1}\BE[\Norm{\nabla \Phi(x_{t})}] + \frac{4 \sigma \eta_x T}{M^{\frac{p-1}{p}}}\\
        &+ 2\eta_x \ell \sum_{t=0}^{T-1} \BE[\Norm{ y_{t} - y^*(x_t)}] + \eta_{x}^2 (\kappa + 1)  \ell T,
    \end{split}
    \end{align*}
    where the second inequality follows from the $\ell$-smoothness of the function $f$.
    Rearranging the above inequality and dividing both sides by $\eta_x T$, it follows that
    \begin{equation*}
    \begin{split}
        \frac{1}{T}\sum_{t=0}^{T-1}\BE[\Norm{\nabla \Phi(x_{t})}] \leq & \frac{\BE[\Phi(x_0) - \Phi(x_{T})]}{\eta_x T} + \frac{4 \sigma}{M^{\frac{p-1}{p}}} + \frac{2\ell \sum_{t=0}^{T-1} \BE[\Norm{ y_{t} - y^*(x_t)}]}{T}   + \eta_{x} (\kappa + 1)  \ell.
        \end{split}
    \end{equation*}
\end{proof}


\subsubsection{The Proof of Theorem \ref{thm:minimax_expectation}}

We first provide the upper bound of $\Norm{y_t - y^*(x_t)}$.
\begin{lem}
    Following the setting of Theorem \ref{thm:minimax_expectation}, for all $t= 0, \dots, T- 1$, we have
    \begin{align*}
        \BE[\Norm{y_t - y^*(x_t)}] = \fO \left(\frac{\epsilon}{\ell}\right).
    \end{align*}
    \label{lemma:y_bound_minimax}
\end{lem}
\begin{proof}
     We first specify the parameter setting for N$^2$SGDA as
\begin{align*}
       & \eta_x =  \frac{\epsilon}{\kappa \ell}, ~~~ \eta_{y,t} =  \min \left\{\frac{1}{400 \ell}, \frac{\ln (B_{t})}{\mu K}\right\},\\
        & \tau_{t,k} = \frac{\exp(- \eta_{y,t} \mu (1 + k/ 2) )R_y}{120 \eta_{y,t}},  ~~~ K = \tilde{\fO}\left(\kappa + \left( \frac{\ell^2 \sigma^2 }{\mu^2 \epsilon^2} \right)^{\frac{p}{2(p-1)}} {+ \frac{R_y \ell}{\epsilon}} \right),
    \end{align*}
    where we set
    \begin{align*}
        B_{t} =  \max\left\{2, \frac{\mu^2 K^{ \frac{2 (p - 1)}{p}} R_y^2}{5400^{\frac{2}{p}} \sigma^2 (\ln (B_{t}))^2 }\right\}, ~~~ \Delta_\Phi = \Phi(x_0) - \min_{x \in \BR^{d_x}} \Phi(x), ~~~ R_y^2 \coloneqq \fO\left(R_0^2 + \frac{ \epsilon^2}{\ell^2}\right).
    \end{align*}

    Recall that Algorithm \ref{alg:nsgd_minimax} set $y_t=\hat{y}_{t, K}$ in line~\ref{minimax_update_y}, where $\hat{y}_{t, K}$ can be regarded as the output of applying clipped stochastic gradient descent (lines \ref{inner_iter_start}--\ref{inner_iter_end}) to minimize the function $-f(x_t, y)$ with respect to~$y$. The updates use the stochastic gradient estimate 
    $-\nabla_y F(x_t, y; \xi'_t)$ for $-f(x_t, y)$. 
    We apply Lemma \ref{lemma:y_recur_strongly} with $h(y) = -f(x_t,y)$, $\sigma_h = \sigma$, $\ell_h= \ell$, and $\mu_h= \mu$ to obtain
    \begin{align}\label{eq:dis_y_t}\small
    \begin{split}        
      \BE\left[\Norm{\hat{y}_{t, K} - y^*(x_{t})}^2\right] \leq & {\fO \left( \hat{R}_{y,t}^2 \exp\left(- \frac{K}{ \kappa \ln(K \kappa)} \right)+ \frac{\sigma^2 (\ln (B_K))^2 (\ln(K \kappa))^{\frac{2(p-1)}{p}}}{\mu^2 K^{ \frac{2 (p - 1)}{p}}} + \frac{\hat{R}_{y,t}^2 }{K^2}\right),}    
    \end{split}
    \end{align}
    for all ${\hat R}_{y,t}^2 \geq \BE[\Norm{\hat{y}_{t, 0} - y^*(x_{t})}^2]$.

    We then use induction to show $\BE[\Norm{\hat{y}_{t, 0} - y^*(x_{t})}^2]$ has the uniform upper bound, i.e.,
    \begin{equation}
         \BE[\Norm{\hat{y}_{t,0} - y^*(x_t)}^2] \leq R_y^2 \coloneqq \fO \left(\frac{ \epsilon^2}{\ell^2} + R_0^2\right).
        \label{minimax_y_induction}
    \end{equation} 
    holds for all $t = 0, \ldots, T-1$. 
For the induction base $t=0$, it follows from the setting of $R_0$ and $R_y$ that 
\begin{align*}
\BE[\Norm{\hat{y}_{0,0} - y^*(x_0)}^2] \leq R_0^2 \leq R_y^2.
\end{align*}
For the induction step, we suppose \eqref{y_induction} holds for all~$t = 0, \dots, T'-1$. In the case of $t = T'$, we have
\begin{align}\label{eq:induction-dist-01}
\begin{split}
     \BE\left[\Norm{\hat{y}_{T', 0} - y^*(x_{T'}) }^2\right] 
     \leq & 2 \BE\left[\Norm{\hat{y}_{T'-1, K} - y^*(x_{T'-1})}^2\right] + 2 \BE\left[\Norm{y^*(x_{T'-1}) - y^*(x_{T'})}^2\right] \\
        \leq & 2 \BE\left[\Norm{\hat{y}_{T'-1, K} - y^*(x_{T'-1})}^2\right] + 2 \kappa^2 \Norm{x_{T'-1} - x_{T'}}^2 \\
        = &  2 \BE\left[\Norm{\hat{y}_{T'-1, K} - y^*(x_{T'-1})}^2\right] + 2 \kappa^2 \eta_x^2. 
\end{split}    
\end{align}
where the first inequality follows the Young's inequality, the second inequality is due to Lemma~\ref{lemma:nonconvex_strongly_concave_prop}, and the last inequality is based on line \ref{line:minimax_outer-end} of Algorithm \ref{alg:nsgd_minimax}.

Following the derivation of \eqref{eq:dis_t}, we have

\begin{align}\label{eq:induction-T-2}
\begin{split}
    &\BE\left[\Norm{\hat{y}_{T'-1, K} - y^*(x_{T'-1})}^2\right] \\
    = & {\fO \left( \hat{R}_{y,T'-1}^2 \exp\left(- \frac{K}{ \kappa \ln(K \kappa)} \right)+ \frac{ \sigma^2 (\ln (B_{T'-1}))^2 (\ln(K \kappa))^{\frac{2(p-1)}{p}}}{\mu^2 K^{ \frac{2 (p - 1)}{p}}} + \frac{\hat{R}_{y,T'-1}^2}{K^2}\right) ,}
\end{split}
\end{align}
for all ${\hat R}_{y,T'-1}^2 \geq \BE[\Norm{\hat{y}_{T'-1, 0} - y^*(x_{T'-1})}^2]$. 
In addition, the induction hypothesis (\ref{minimax_y_induction}) implies we can 
apply \eqref{eq:induction-T-2} by
taking ${\hat R}_{y,T'-1}^2=R_y^2$ to achieve
\begin{align*}\small
\begin{split}    
   \BE\left[\Norm{\hat{y}_{T'-1, K} - y^*(x_{T'-1})}^2\right] = & {\fO \left( R_y^2 \exp\left(- \frac{K}{ \kappa \ln(K \kappa)} \right)+ \frac{ \sigma^2 (\ln (B_{T'-1}))^2 (\ln(K \kappa))^{\frac{2(p-1)}{p}}}{\mu^2 K^{ \frac{2 (p - 1)}{p}}} + \frac{R_y^2}{K^2}\right)  .}
\end{split}
\end{align*}
Combining above inequality with \eqref{eq:induction-dist-01}, we have
\begin{align*}
    &\BE\left[\Norm{\hat{y}_{T',0} - y_{\lambda}^*(x_{T'})}^2\right] \\
\leq & {\fO \left( R_y^2 \exp\left(- \frac{K}{ \kappa \ln(K \kappa)} \right)+ \frac{ \sigma^2 (\ln (B_{T'-1}))^2 (\ln(K \kappa))^{\frac{2(p-1)}{p}}}{\mu^2 K^{ \frac{2 (p - 1)}{p}}} + \frac{R_y^2}{K^2}\right)  + 2 \kappa^2 \eta_x^2} \\
\leq  &{\fO\left(\frac{ \epsilon^2}{\ell^2} \right)}
    \leq  R_y^2, 
\end{align*}
where the second inequality is based on the setting of $K$,
the third inequality is based on the setting of $\eta_x$,
and the last inequality is based on the definition of $R_y^2$.
This completes the induction.

Therefore, we can apply Lemma \ref{lemma:y_recur_strongly} on $\Norm{\hat{y}_{t, K} - y^*(x_{t})}^2$ with $\hat{R}_{y,t} = R_y$ to obtain
\begin{align*}
\begin{split}    
   & \BE\left[\Norm{{y}_{t} - y^*(x_{t})}^2\right] = \BE\left[\Norm{\hat{y}_{t, K} - y^*(x_{t})}^2\right] \\ 
   \leq &  {\fO \left( R_y^2 \exp\left(- \frac{K}{ \kappa \ln(K \kappa)} \right)+ \frac{ \sigma^2 (\ln (B_t))^2 (\ln(K \kappa))^{\frac{2(p-1)}{p}}}{\mu^2 K^{ \frac{2 (p - 1)}{p}}} + \frac{R_y^2}{K^2}\right)}
   = \fO\left(\frac{\epsilon^2}{\ell^2 }\right).
\end{split}
\end{align*}

Consequently, we use Jensen's inequality to achieve
\begin{align*}   
   & \BE\left[\Norm{{y}_{t} - y^*(x_{t})}\right] 
   \leq \sqrt{\BE\left[\Norm{{y}_{t} - y^*(x_{t})}^2\right]}   
    = \fO \left(\frac{\epsilon}{\ell}\right).
\end{align*}    
\end{proof}

We then provide the proof of Theorem \ref{thm:minimax_expectation}.

\begin{proof}
We follow the parameter settings in the proof of Lemma \ref{lemma:y_bound_minimax} and additionally set
\begin{align*}
    T =  \fO\left(\frac{\Delta_\Phi \kappa \ell}{\epsilon^2}\right) \qquad\text{and}\qquad  M = \left(\frac{\sigma}{\epsilon}\right)^{\frac{p}{p-1}},
\end{align*}
Recall that Lemma \ref{lemma:core_recur} implies that
 \begin{align}
 \begin{split}
    \frac{1}{T}\sum_{t=0}^{T-1}\BE[\Norm{\nabla \Phi(x_{t})}] \leq & \frac{\BE[\Phi(x_0) - \Phi(x_{T})]}{\eta_x T} + \frac{4 \sigma}{M^{\frac{p-1}{p}}} \\
         & + \frac{2\ell \sum_{t=0}^{T-1} \BE[\Norm{ y_{t} - y^*(x_t)}]}{T} + \eta_{x} (\kappa + 1)  \ell.
         \label{minimax_main_grad}
    \end{split}
 \end{align}   
In the remainder of this proof, we will show that each term on the right-hand side of the \eqref{minimax_main_grad} can be upper bounded by $\fO(\epsilon)$ in expectation, which means that the output~$\hat{x}$ is an $\fO(\epsilon)$-stationary point of $\Phi(x)$.


    Based on the settings of $\eta_x$, $T$ and $M$, we have
    \begin{align}\label{minimax_2}
        \frac{\BE[\Phi(x_0) - \Phi(x_{T})]}{\eta_x T} = \fO(\epsilon), ~~~ \frac{4 \sigma}{M^{\frac{p-1}{p}}} = \fO(\epsilon), ~~~ {\rm and} ~~~ \eta_{x} (\kappa + 1)  \ell= \fO(\epsilon).
    \end{align}
   According to Lemma \ref{lemma:y_bound_minimax}, we have
   \begin{align}\label{minimax_3}
       \BE[\Norm{y_t - y^*(x_t)}] = \fO\left(\frac{\epsilon}{\ell}\right).
   \end{align}
   for all $t=0, \dots, T-1$.
    Combining equations (\ref{minimax_main_grad}), (\ref{minimax_2}), and (\ref{minimax_3}), we have
    \begin{align*}
        \BE[\Norm{\nabla \Phi(\hat{x})}] = \frac{1}{T}\sum_{t=0}^{T-1}\BE[\Norm{\nabla \Phi(x_{t})}] = \fO(\epsilon),
    \end{align*}
    where the first step is due to $\hat{x}$ is uniformly sampled from $\{x_t\}_{t=1}^T$. In addition, the total SFO complexity is 
    \begin{align*}
        T M + T K = \tilde\fO\left( \frac{\Delta_\Phi \kappa^{\frac{2p-1}{p-1}} \ell \sigma^{\frac{p}{p-1}}}{\epsilon^\frac{3p-2}{p-1}}  + \frac{\Delta_\Phi \kappa^2 \ell}{\epsilon^2} {+ \frac{R_0 \Delta_\Phi \kappa \ell^2}{\epsilon^3}}\right). 
    \end{align*}
\end{proof}

\subsection{The High-Probability Convergence Guarantee}

In this subsection, we present the proofs for the high-probability convergence guarantee of the  N$^2$SGDA shown in Section \ref{sec:minimax_opt}.
We denote the natural filtration of our method by $\fF_t \coloneqq \sigma(\{g_{x,s}\}_{s \leq t}, \{g_{y,s,k}\}_{s \leq {t+1}, k \leq K})$.

\subsubsection{Technical Lemmas}

We first establish the following lemma to bound the gradient norm of $\Phi$.

\begin{lem}
Let $\delta \in (0, 1)$. Under Assumption \ref{asm:minimax}, the iterates generated by  N$^2$SGDA satisfy
   \begin{align*}
          \frac{1}{T} \sum_{t=0}^{T-1}  \Norm{\nabla \Phi(x_t)} \leq & \frac{2\Delta_\Phi}{\eta_x T} +  \sum_{t=1}^T \frac{4\BE[\Norm{g_{x, t} - \nabla_x f(x_t, y_t)} \mid \fF_{t-1}] }{T} +  \sum_{t=1}^T \frac{4 \Norm{ \nabla_x f(x_t, y_t) - \nabla \Phi(x_t)}}{T}   \\
         & + 2 \eta_x (\kappa + 1) \ell  + 12 \left( \frac{\Norm{\nabla \Phi(x_0)}}{T}  + 2 \eta_x (\kappa + 1) \ell \right) \ln\left(\frac{2}{\delta}\right)
     \end{align*}
     with probability at least $1 - \delta / 2$.
     \label{lemma:high_prob_main_recur}
\end{lem}
\begin{proof}
     Let $\nu_t \coloneqq \frac{\langle \nabla \Phi(x_t), g_{x, t} \rangle}{\Norm{\nabla \Phi(x_t)} \Norm{g_{x, t}}}$.
     According to Lemma~\ref{lemma:core_recur}, we have
     \begin{align*}
          \Phi(x_{t+1}) \leq & \Phi(x_{t}) - \eta_x \Norm{\nabla \Phi(x_{t})} \nu_{t} + \eta_{x}^2 (\kappa + 1) \ell.
     \end{align*}
     Summing over above inequality with $t=0$ to $T-1$, we obtain
     \begin{align}\label{nu_upper_1}
         \sum_{t=0}^{T-1} \eta_x \nu_t \Norm{\nabla \Phi(x_t)} \leq \Delta_{\Phi} + \eta_x^2 (\kappa + 1) \ell T.
     \end{align}
     Let $\psi_t \coloneqq \BE[\nu_t \mid \fF_{t-1}]$ and denote $\{\Theta_t \coloneqq -\eta_x (\nu_t - \psi_t)\Norm{\nabla \Phi(x_t)}\}_{t \in \BN}$ as the martingale difference sequence with respect to $\{\fF_t\}_{t \in \BN}$.
     Note that $|\nu_t|\leq 1$, then we have
     \begin{align*}
         \exp\left( \frac{\Theta_t^2}{4 \eta_x^2 \Norm{\nabla \Phi(x_t)}^2}\right) = \exp\left( \frac{(\nu_t -\psi_t)^2}{4}\right) \leq {\rm e},
     \end{align*}

 Consequently, we apply Lemma \ref{lemma:hp_concentrate} with $\Theta_t= -\eta_x (\nu_t - \psi_t)\Norm{\nabla \Phi(x_t)}$, $\sigma_t = 2 \eta_x \Norm{\nabla \Phi(x_t)}$, and $\hat\delta=\delta/2$ to obtain that
    \begin{align*} 
        \sum_{t=1}^T \eta_x (\psi_t - \nu_t)\Norm{\nabla \Phi(x_t)} \leq & 3 \chi \eta_x^2 \sum_{t=1}^T  \Norm{\nabla \Phi(x_t)}^2 + \frac{1}{\chi} \ln \left(\frac{2}{\delta}\right)
    \end{align*}
    holds with probability at least $1 - \delta / 2$  for all $\chi > 0$.
    Rearranging above inequality, we have
    \begin{align}\label{eq:high-L-lambda_1}
    \begin{split}        
         & \sum_{t=0}^{T-1} \eta_x (\psi_t - 3 \chi \eta_x \Norm{\nabla \Phi(x_t)}) \Norm{\nabla \Phi(x_t)} \\
         \leq & \sum_{t=1}^T \eta_x \nu_t\Norm{\nabla \Phi(x_t)} + \frac{1}{\chi} \ln \left(\frac{2}{\delta}\right) 
         \leq \Delta_\Phi + \eta_x^2 (\kappa + 1) \ell T + \frac{1}{\chi}\ln\left(\frac{2}{\delta}\right)
    \end{split}
     \end{align}
    holds with probability at least $1 - \delta / 2$, where the last step is based on \eqref{nu_upper_1}. In addition, the smoothness of $\Phi$ shown in Lemma \ref{lemma:nonconvex_strongly_concave_prop} implies
     
     
     \begin{equation}\label{Phi_smooth}
      \Norm{\nabla \Phi(x_t)} \leq \Norm{\nabla \Phi(x_0)} + (\kappa + 1) \ell t \eta_x.    
     \end{equation}

    Then with probability at least $1 - \delta / 2$, we have
    \begin{align}\label{minimax_high_prob_interm_1}
    \begin{split}        
        & \sum_{t=0}^{T-1} \eta_x \left(\psi_t - \frac{1}{2}\right) \Norm{\nabla \Phi(x_t)} 
        \\
        \leq & \sum_{t=0}^{T-1} \eta_x \left(\psi_t - \frac{\Norm{\nabla \Phi(x_t)}}{2(\Norm{\nabla \Phi(x_0)} + (\kappa + 1) \ell \eta_x T)}\right) \Norm{\nabla \Phi(x_t)} \\ 
        \leq & \Delta_\Phi + \eta_x^2 (\kappa + 1) \ell T + 6 (\eta_x \Norm{\nabla \Phi(x_0)} + \eta_x^2 T (\kappa + 1) \ell)\ln\left(\frac{2}{\delta}\right),
    \end{split}
    \end{align}
    where the first inequality follows from (\ref{Phi_smooth}) and the second inequality is due to \eqref{eq:high-L-lambda_1} with
     $\chi \coloneqq {1}/{(6 (\eta_x \Norm{\nabla \Phi(x_0)} + \eta_x^2 T (\kappa + 1) \ell))}$.

     By applying Lemma~\ref{lemma:lower_bound} with $a=\nabla \Phi(x_t)$ and $b=g_{x,t}$, we get
     \begin{align}
         \psi_t \Norm{\nabla \Phi(x_t)} \geq \Norm{\nabla \Phi(x_t)} - 2 \BE[\Norm{g_{x, t} - \nabla \Phi(x_t)}\mid \fF_{t-1}].
         \label{minimax_interm_2}
     \end{align}
     which implies
    \begin{align*}
        &\sum_{t=0}^{T-1}  \frac{\eta_x}{2} \Norm{\nabla \Phi(x_t)} - 2 \eta_x \BE[\Norm{g_{x,t} - \nabla \Phi(x_t)}\mid \fF_{t-1}] \\
       \leq & \sum_{t=0}^{T-1} \eta_x \left(\psi_t - \frac{1}{2}\right) \Norm{\nabla \Phi(x_t)} 
        \\
        \leq & \Delta_\Phi + \eta_x^2 (\kappa + 1) \ell T + 6 (\eta_x \Norm{\nabla \Phi(x_0)} + \eta_x^2 T (\kappa + 1) \ell)\ln\left(\frac{2}{\delta}\right),
    \end{align*}
     where the last step is based on (\ref{minimax_high_prob_interm_1}).
    Rearranging the above inequality, we have
     \begin{align*}
         \frac{1}{2} \sum_{t=0}^{T-1} \eta_{x}  \Norm{\nabla \Phi(x_t)} \leq & \Delta_\Phi + 2 \eta_x \sum_{t=0}^{T-1} \BE[\Norm{g_{x, t} - \nabla \Phi(x_t)}\mid \fF_{t-1}] + \eta_x^2 (\kappa + 1) \ell T \\
         & + 6 ( \eta_x \Norm{\nabla \Phi(x_0)} + \eta_x^2 T (\kappa + 1) \ell) \ln\left(\frac{2}{\delta}\right) .
     \end{align*}
     with probability at least $1 - \delta/ 2$.
     Dividing both sides by ${\eta_x T}/{2}$ finishes the proof.
\end{proof}
Under an appropriate choice of $\eta_x$, we can derive the following high-probability lemma.
\begin{lem}
Let $\delta \in (0, 1)$. Under Assumption \ref{asm:minimax}, we choose $\eta_x = \sqrt{{\Delta_\Phi}/{(\varkappa + 1) \ell T}}$ for N$^2$SGDA.
Then, with probability at least $1 - \delta / 2$, the iterates generated by the algorithm satisfy
     \begin{align*}
         \frac{1}{T} \sum_{t=0}^{T-1}   \Norm{\nabla \Phi(x_t)} 
         \leq & \left(4 + 50  \ln\left(\frac{2}{\delta}\right) \right) \frac{ \sqrt{\Delta_\Phi (\kappa + 1) \ell}}{\sqrt{T}}  +  \sum_{t=1}^T \frac{4 \ell \Norm{ y_t - y^*(x_t)}}{T} +  \frac{8 \sigma}{M^{\frac{p-1}{p}}}.
    \end{align*}
    \label{lemma:high_prob_main_res}
\end{lem}
\begin{proof}
    Note that $x_t$ and $y_t$ are $\fF_{t-1}$ measurable and $\xi_1^t, \ldots, \xi_M^t$ are independent of $\fF_{t-1}$,
    then we have $\BE[\Norm{g_{x, t} - \nabla_x f(x_t, y_t)} \mid \fF_{t-1}] = \BE[\Norm{g_{x, t} - \nabla_x f(x_t, y_t)} \mid x_t, y_t]$. 
    Hence, applying Lemma~\ref{lemma:bound_var} implies
    \begin{align*}
        \BE[\Norm{g_{x, t} - \nabla_x f(x_t, y_t)} \mid \fF_{t-1}] = \BE[\Norm{g_{x, t} - \nabla_x f(x_t, y_t)} \mid x_t, y_t] \leq \frac{2 \sigma}{M^{\frac{p-1}{p}}}.
    \end{align*}
    Plugging above result into Lemma~\ref{lemma:high_prob_main_recur} yields
    \begin{align*}
          \frac{1}{T} \sum_{t=0}^{T-1}   \Norm{\nabla \Phi(x_t)} \leq & \frac{2\Delta_\Phi}{\eta_x T} +  \frac{8 \sigma}{M^{\frac{p-1}{p}}} +  \sum_{t=1}^T \frac{4 \Norm{ \nabla_x f(x_t, y_t) - \nabla \Phi(x_t)}}{T}   \\
         & + 2 \eta_x (\kappa + 1) \ell  + 12 \left( \frac{\Norm{\nabla \Phi(x_0)}}{T}  + 2 \eta_x (\kappa + 1) \ell \right) \ln\left(\frac{2}{\delta}\right).
    \end{align*}
    We choose $\eta_x = \sqrt{{\Delta_\Phi}/{(\kappa + 1) \ell T}}$, then with probability at least $1 - \delta / 2$, it holds
    \begin{align*}
         \frac{1}{T} \sum_{t=0}^{T-1}   \Norm{\nabla \Phi(x_t)} \leq & \frac{4 \sqrt{\Delta_\Phi (\kappa + 1) \ell}}{\sqrt{T}} +  \frac{8 \sigma}{M^{\frac{p-1}{p}}} +  \sum_{t=0}^{T-1} \frac{4 \Norm{ \nabla_x f(x_t, y_t) - \nabla \Phi(x_t)}}{T}   \\
         &   + 12 \left( \frac{\Norm{\nabla \Phi(x_0)}}{T}  + \frac{2 \sqrt{\Delta_\Phi (\kappa + 1) \ell}}{\sqrt{T}} \right) \ln\left(\frac{2}{\delta}\right) \\
         \leq & \left(4 + 50  \ln\left(\frac{2}{\delta}\right) \right) \frac{ \sqrt{\Delta_\Phi (\kappa + 1) \ell}}{\sqrt{T}}  +  \sum_{t=0}^{T-1} \frac{4 \ell \Norm{ y_t - y^*(x_t)}}{T}  +  \frac{8 \sigma}{M^{\frac{p-1}{p}}},
    \end{align*}
    where we use $\Norm{\nabla \Phi(x_0)} \leq \sqrt{2 (\kappa+1) \ell \Delta_\Phi}$ to derive the last inequality.
\end{proof}

\subsubsection{The Proof of Theorem \ref{thm:minimax_high_prob}}
We first provide lemmas to upper bound the term $\Norm{y_t- y^*(x_t)}$ with high probability.
\begin{lem}
Following the setting of Theorem \ref{thm:minimax_high_prob}, for all $t=0, \dots, T-1$, we have
\begin{align}
    \Norm{{y}_{t} - y^*(x_{t})} = \fO\left(\frac{\epsilon}{\ell}\right),
\end{align}    
with probability at least $1 - \delta / 2$.
\label{lemma:minimax_y_bound_high_prob}
\end{lem}
\begin{proof}
We first specify the following hyperparameters for N$2$SGDA:
\begin{align*}
       & \eta_x =  \sqrt{\frac{\Delta_\Phi}{(\kappa + 1) \ell T}},~~~\eta_{y,t} =  \min \left\{\frac{1}{400 \ell \ln\left({8 T K}/{\delta}\right)}, \frac{\ln (B_{t})}{\mu K}\right\},\\
        & \tau_{t,k} = \frac{\exp(- \eta_{y,t} \mu (1 + k/ 2) )R_y}{120 \eta_{y,t} \ln\left({8 T K}/{\delta}\right)},~~~ K = \tilde{\fO}\left(\kappa + \left( \frac{\ell^2 \sigma^2 }{\mu^2 \epsilon^2} \right)^{\frac{p}{2(p-1)}} \right),
    \end{align*}
    where we set
    \begin{align*}
        & B_{t} =  \max\left\{2, \frac{\mu^2 K^{ \frac{2 (p - 1)}{p}} R_y^2}{5400^{\frac{2}{p}} \sigma^2 (\ln (B_{t}))^2 \left(\ln\left({8 T K}/{\delta}\right)\right)^{\frac{2(p-1)}{p}} }\right\},  \\
        &  \Delta_\Phi = \Phi(x_0) - \min_{x \in \BR^{d_x}} \Phi(x),~~~R_y^2 = R_0^2 + \frac{10 \epsilon^2}{\ell^2},
    \end{align*}

  Recall that Algorithm \ref{alg:nsgd_minimax} set $y_t=\hat{y}_{t, K}$ in line~\ref{minimax_update_y}, where $\hat{y}_{t, K}$ can be regarded as the output of applying clipped stochastic gradient descent (lines \ref{inner_iter_start}--\ref{inner_iter_end}) to minimize the function $-f(x_t, y)$ with respect to~$y$. The updates use the stochastic gradient estimate 
    $-\nabla_y F(x_t, y; \xi'_t)$ for $-f(x_t, y)$. 
    We apply Lemma \ref{lemma:high_prob_strongly_convex_converg} with $h(y) = -f(x_t,y)$, $\sigma_h = \sigma$, $\ell_h= \ell$,  $\mu_h= \mu$, and $\hat\delta = \delta/(2T)$ to obtain
    \begin{align}\label{eq:minimax_dis_t_1}
    \begin{split}    
          \Norm{\hat{y}_{t,K} - y^*(x_t)}^2
    \leq & 2 {\hat R}_{y,t}^2 \exp\left(- \frac{\mu K}{400 \ell  \ln\left({8 T K}/{\delta}\right)}\right)  \\
    & + \frac{2 \cdot 5400^{\frac{2}{p}} \sigma^2 (\ln (B_{y,t}))^2 \left(\ln\left({8 T K }/{\delta}\right)\right)^{\frac{2(p-1)}{p}} }{ \mu^2 K^{ \frac{2 (p - 1)}{p}}} ,
    \end{split}
    \end{align}
    for all given ${\hat R}_{y,t}^2 \geq \Norm{\hat{y}_{t, 0} - y^*(x_{t})}^2$ with probability $1-\delta / (2T)$ for all $t=0, 1, \dots, T-1$. 

    We then use induction to show that $\Norm{\hat{y}_{0, 0} - y^*(x_{0})}^2, \dots, \Norm{\hat{y}_{t, 0} - y^*(x_{t})}^2$ for all $t=0,\dots, T-1$ have a uniform upper bound with high probability.
    In particular, we define the event $E_t$ as
    \begin{equation}
    \Norm{\hat{y}_{l,0} - y^*(x_l)}^2 \leq R_y^2 \coloneqq R_0^2 + \frac{10 \epsilon^2}{\ell^2}
    \label{minimax_y_induction_1}
    \end{equation}
    holds for all $l = 0, 1, \dots, t$ and target to show that $\BP(E_t) \geq 1 - t \delta / (2T)$  for all $t = 0, \ldots, T-1$. 

For the induction base $t=0$, we follow the choice of $R_0$ to achieve
\begin{align*}
\Norm{\hat{y}_{0,0} - y^*(x_0)}^2 \leq  R_0^2   
\end{align*}
that is, $\BP(E_0) = 1$.
For the induction step, we suppose that $\BP(E_{t}) \geq 1 - t \delta / (2 T)$ holds for all $t = 0, \dots,  T' -1$. In the case of $t= T'$, we follow the derivation of \eqref{eq:induction-dist-01} to achieve
\begin{align}\label{eq:minimax_induction-dist-001}
\begin{split}
    \Norm{\hat{y}_{T',0} - y^*(x_{T'})}^2 \leq & 2 \Norm{\hat{y}_{T'-1, K} - y^*(x_{T'-1})}^2 + 2 \kappa^2\eta_x^2,
\end{split}    
\end{align}
Following the derivation of \eqref{eq:minimax_dis_t_1}, it holds
\begin{align}\label{eq:minimax_induction-T-11}
\begin{split}
    \Norm{\hat{y}_{T'-1,K} - y^*(x_{T'-1})}^2
    \leq & 2 {\hat R}_{y,T'-1}^2 \exp\left(- \frac{\mu K}{400 \ell  \ln\left({8 T K}/{\delta}\right)}\right)  \\
    & + \frac{2 \cdot 5400^{\frac{2}{p}} \sigma^2 (\ln (B_{y,T'-1}))^2 \left(\ln\left({8 T K }/{\delta}\right)\right)^{\frac{2(p-1)}{p}} }{ \mu^2 K^{ \frac{2 (p - 1)}{p}}} ,
\end{split}
\end{align}
with probability at least $1 - \delta/ (2T)$ for all given ${\hat R}_{y,T'-1}^2 \geq \Norm{\hat{y}_{T'-1, 0} - y^*(x_{T'-1})}^2$. 
In addition, the induction hypothesis $\BP(E_{T-1}) \geq 1 - (T-1) \delta / (2T)$ implies we can 
apply \eqref{eq:minimax_induction-T-11} by
taking ${\hat R}_{y,T'-1}^2=R_y^2$ and union bound on $E_{T'-1}$ and (\ref{eq:minimax_induction-T-11}) to guarantee
\begin{align*}
\begin{split}    
      \Norm{\hat{y}_{T'-1,K} - y^*(x_{T'-1})}^2
    \leq & 2 R_y^2 \exp\left(- \frac{\mu K}{400 \ell  \ln\left({8 T K}/{\delta}\right)}\right)  \\
    & + \frac{2 \cdot 5400^{\frac{2}{p}} \sigma^2 (\ln (B_{y,T'-1}))^2 \left(\ln\left({8 T K }/{\delta}\right)\right)^{\frac{2(p-1)}{p}} }{ \mu^2 K^{ \frac{2 (p - 1)}{p}}},
\end{split}
\end{align*}
holds with probability at least $1 - T'\delta / (2 T)$.
Combining above inequality with \eqref{eq:minimax_induction-dist-001}, we guarantee
\begin{align*}
 \Norm{\hat{y}_{T',0} - y^*(x_{T'})}^2
    \leq & 4 R_y^2 \exp\left(- \frac{\mu K}{400 \ell  \ln\left({8 T K}/{\delta}\right)}\right)  \\
    & + \frac{4 \cdot 5400^{\frac{2}{p}} \sigma^2 (\ln (B_{y,T'-1}))^2 \left(\ln\left({8 T K }/{\delta}\right)\right)^{\frac{2(p-1)}{p}} }{ \mu^2 K^{ \frac{2 (p - 1)}{p}}}  + 2 \kappa^2 \eta_x^2 \\
    \leq & \frac{8 \epsilon^2}{\ell^2} + 2 \kappa^2 \eta_x^2 \leq \frac{10 \epsilon^2}{\ell^2} \leq R_y^2.
\end{align*}
holds with probability at least $1 - T' \delta/(2T)$, i.e., $\BP(E'_{T'}) \geq 1 - T' \delta/ (2 T) $,
where the second inequality is based on the setting of $K$,
the third inequality is based on the setting of $\eta_x$,
and the last inequality is based on the definition of $R_y^2$.
This completes the induction.

Therefore, we can apply Lemma \ref{lemma:high_prob_strongly_convex_converg} on $\Norm{\hat{y}_{l, K} - y^*(x_{l})}^2$ with $\hat{R}_{y,l} = R_y$ and the union bound on $E_t$ and (\ref{eq:minimax_dis_t_1}) to guarantee
\begin{align*}
\begin{split}    
   & \Norm{{y}_{l} - y^*(x_{l})}^2 = \Norm{\hat{y}_{l, K} - y^*(x_{l})}^2 \\ 
   \leq & 2 R_y^2 \exp\left(- \frac{\mu K}{400 \ell  \ln\left({8 T K}/{\delta}\right)}\right)  \\
    & + \frac{2 \cdot 5400^{\frac{2}{p}} \sigma^2 (\ln (B_{y,l}))^2 \left(\ln\left({8 T K }/{\delta}\right)\right)^{\frac{2(p-1)}{p}} }{ \mu^2 K^{ \frac{2 (p - 1)}{p}}}
   = \fO\left(\frac{\epsilon^2}{\ell^2}\right),
\end{split}
\end{align*}
holds for all $l = 0, \dots, t$ with probability at least $1 - (t+1) \delta / (2 T)$,
which implies 
\begin{align}
    \Norm{{y}_{t} - y^*(x_{t})} = \fO\left(\frac{\epsilon}{\ell}\right),
\end{align}
holds for all $t = 0,1, \dots, T-1$ with probability at least $1 - \delta/ 2$.

\end{proof}
Then we provide the proof of Theorem \ref{thm:minimax_high_prob}.
\begin{proof}
We follow the parameter settings in the proof of Lemma \ref{lemma:minimax_y_bound_high_prob} and additionally set
\begin{align*}
     T =  \Tilde\fO\left(\frac{\Delta_\Phi \kappa \ell}{\epsilon^2}\right)~~~\text{and}~~~
        M = \left(\frac{\sigma}{\epsilon}\right)^{\frac{p}{p-1}},
\end{align*}
    According to Lemma \ref{lemma:high_prob_main_res}, we have 
    \begin{equation}
         \frac{1}{T} \sum_{t=0}^{T-1}   \Norm{\nabla \Phi(x_t)} 
         \leq   \frac{\left(4 + 50  \ln\left({2}/{\delta}\right) \right) \sqrt{\Delta_\Phi (\kappa + 1) \ell}}{\sqrt{T}}  +  \sum_{t=1}^T \frac{4 \ell \Norm{ y_t - y^*(x_t)}}{T} +  \frac{8 \sigma}{M^{\frac{p-1}{p}}},
         \label{nonconvex_high_prob_1}
    \end{equation}
    with probability $ 1- \delta / 2$.
    Our choices of $T$ and $M$ guarantee
    \begin{align}\label{minimax_inter_1}
         \frac{\left(4 + 50  \ln\left({2}/{\delta}\right) \right) \sqrt{\Delta_\Phi (\kappa + 1) \ell}}{\sqrt{T}} = \fO(\epsilon)\quad{\rm and}\quad  \frac{8 \sigma}{M^{\frac{p-1}{p}}} = \fO(\epsilon),
    \end{align}
    According to Lemma \ref{lemma:minimax_y_bound_high_prob}, we have
    \begin{align}\label{minimax_inter_2}
        \sum_{t=1}^T \frac{4 \ell \Norm{ y_t - y^*(x_t)}}{T} = \fO(\epsilon)
    \end{align}
    with probability at least $1 - \delta / 2$.
    Combining equations (\ref{nonconvex_high_prob_1}), (\ref{minimax_inter_1}) and (\ref{minimax_inter_2}), it holds
    \begin{align*}
         \frac{1}{T} \sum_{t=0}^{T-1}   \Norm{\nabla \Phi(x_t)} = \fO(\epsilon)
    \end{align*}
    with probability at least $1 - \delta$.
     In addition, the total SFO complexity is 
    \begin{align*}
        T M + T K 
        = & \Tilde{O}\left(  \frac{\Delta_\Phi \kappa^{\frac{2p -1 }{p-1}} \ell \sigma^{\frac{p}{p-1}}}{\epsilon^\frac{3p-2}{p-1}}  + \frac{\Delta_\Phi  \kappa^2 \ell}{ \epsilon^2} \right).
    \end{align*}

\end{proof}

\section{Theoretical Results for Strongly Convex Problem}\label{appendix:strongly_convex}

In this section, we present both in-expectation and high-probability guarantees of clipped-SGD for minimizing strongly convex functions under heavy-tailed noise, 
which serves as our analysis for bilevel and minimax problems.

\subsection{The High-Probability Convergence Analysis}
The following result states the high-probability convergence guarantee of clipped SGD for a strongly convex function, as established in Lemma 3.1 in \citet{sadiev2023high}. 

\begin{lem}[{\citet[Lemma 3.1]{sadiev2023high}}]
Assume the function ${h}(y) = \BE [{H}(y; \nu)]$ is $\ell_h$-smooth and $\mu_h$-strongly convex with the stochastic index $\nu\sim\fD_h$, and the stochastic gradient estimator $\nabla H(y; \nu)$ satisfies the conditions $\BE[\Norm{\nabla {H}(y; \nu) - \nabla h(y)}^p] \leq \sigma_h^p$ and $\BE[\nabla H(y; \nu)] = \nabla h(y)$ for some $\sigma_h>0$. 
We run the clipped stochastic gradient descent 
    \begin{equation}
    \begin{cases}    
        \nu_k \sim \fD_h \\
        \hat{y}_{k+1} = \hat{y}_{k} - \eta_y {\rm clip}(\nabla H(\hat{y}_k; \nu_k), \tau_k)
    \end{cases}
    \end{equation}
    for $k = 0, 1, \dots, K$ with %
\begin{align*}
    \tau_k =& \frac{\exp(- \eta_y \mu_h (1 + k/ 2) )\hat{R}}{120 \eta_y \ln\left({4( K+1 )}/{\hat\delta}\right)}, \quad
    \eta_y =  \min \left\{\frac{1}{400 \ell_h \ln\left({4 ( K+1 )}/{\hat\delta}\right)}, \frac{\ln (B_K)}{\mu_h (K+1)}\right\},\\
    B_K= & \Theta\left(\max\left\{2, \frac{\mu_h^2 K^{ \frac{2 (p - 1)}{p}} \hat{R}^2}{ \sigma_h^2 \left(\ln \left(\max\left\{2, \frac{\mu_h^2 K^{\frac{2(p-1)}{p}} \hat{R}^2}{\sigma_h^2 (\ln(K / \hat{\delta}))^{\frac{2(p-1)}{p}}}\right\}\right)\right)^2 \left(\ln\left({4 (K + 1)}/{\hat\delta}\right)\right)^{\frac{2(p-1)}{p}} }\right\}\right),
\end{align*}    
Then for all $\hat{R} \geq \Norm{\hat{y}_0 - y^*}$ and $\hat{\delta} \in (0,1)$ satisfying $\ln\big({4 ( K+1 )}/{\hat\delta}\big) \geq 1$, 
we have
\begin{align*}
        \Norm{\hat{y}_{K+1} - y^*}^2 = \fO \left( \max \left\{ \hat{R}^2 \exp\left( -\frac{K\mu_h }{\ell_h \ln(K / \hat\delta)}\right), \frac{\sigma_h^2 (\ln (B_K))^2 \big(\ln\big({    K  }/{\hat\delta}\big)\big)^{\frac{2(p-1)}{p}}}{\mu_h^2 K^{ \frac{2 (p - 1)}{p}}} \right\} \right).
    \end{align*}
with probability at least $1 - \hat\delta$.
\label{lemma:high_prob_strongly_convex_converg}
\end{lem}
{
For ease of the presentation, Lemma \ref{lemma:high_prob_strongly_convex_converg} shows that, by choosing parameters 

\begin{align*}
    \tau_k =& \tilde{\fO}\left(\frac{\exp(- \eta_y \mu_h k )\hat{R}}{\eta_y }\right),~\eta_y =  \tilde{\fO} \left(\min \left\{\frac{1}{ \ell_h }, \frac{1}{\mu_h K}\right\}\right),~B_K= \tilde\fO\left(\max\left\{2, \frac{\mu_h^2 K^{ \frac{2 (p - 1)}{p}} \hat{R}^2}{ \sigma_h^2   }\right\}\right).
\end{align*}    
we have
\begin{align*}
        \Norm{\hat{y}_{K+1} - y^*}^2 = \tilde\fO \left( \max \left\{ \hat{R}^2 \exp\left( -\frac{K\mu_h }{\ell_h }\right), \frac{\sigma_h^2  }{\mu_h^2 K^{ \frac{2 (p - 1)}{p}}} \right\} \right).
    \end{align*}
with probability at least $1 - \hat\delta$ for all $\hat{R} \geq \Norm{\hat{y}_0 - y^*}$ and $\hat{\delta} \in (0,1)$ satisfying \mbox{$\ln({4 ( K+1 )}/{\hat\delta}) \geq 1$}.
}

{
\subsection{The In-Expectation Convergence Guarantee (The Proof of Lemma \ref{lemma:y_recur_strongly})}
We restate Lemma \ref{lemma:y_recur_strongly} below with the complete parameter specification.
\begin{lem}[Lemma \ref{lemma:y_recur_strongly} Restated]
    Assume the function ${h}(y) = \BE [{H}(y; \nu)]$ is $\ell_h$-smooth and $\mu_h$-strongly convex with the stochastic index $\nu\sim\fD_h$, and the stochastic gradient estimator $\nabla H(y; \nu)$ satisfies the conditions~$\BE[\nabla H(y; \nu)] = \nabla h(y)$
and $\BE[\Norm{\nabla {H}(y; \nu) - \nabla h(y)}^p] \leq \sigma_h^p$ for some $\sigma_h>0$. 
We run the clipped stochastic gradient descent 
    \begin{equation}
    \begin{cases}    
        \nu_k \sim \fD_h \\
        \hat{y}_{k+1} = \hat{y}_{k} - \eta_y {\rm clip}(\nabla H(\hat{y}_k; \nu_k), \tau_k)
    \end{cases}
    \end{equation}
    for $k = 0, 1, \dots, K$ with %
   \begin{align*}
    \tau_k =& \frac{\exp(- \eta_y \mu_h (1 + k/ 2) )\hat{R}}{120 \eta_y \ln\left({4( K+1 )}/{\hat\delta}\right)}, \quad
    \eta_y =  \min \left\{\frac{1}{400 \ell_h \ln\left({4 ( K+1 )}/{\hat\delta}\right)}, \frac{\ln (B_K)}{\mu_h (K+1)}\right\},\\
    B_K= & \Theta\left(\max\left\{2, \frac{\mu_h^2 K^{ \frac{2 (p - 1)}{p}} \hat{R}^2}{ \sigma_h^2 \left(\ln \left(\max\left\{2, \frac{\mu_h^2 K^{\frac{2(p-1)}{p}} \hat{R}^2}{\sigma_h^2 (\ln(K / \hat{\delta}))^{\frac{2(p-1)}{p}}}\right\}\right)\right)^2 \left(\ln\left({4 (K + 1)}/{\hat\delta}\right)\right)^{\frac{2(p-1)}{p}} }\right\}\right),
\end{align*}    
where $\hat{\delta} = \mu_h^2   \ell_h^{-2}  (K+1)^{-4}$.
Then for all $\hat{R}^2 \geq \Norm{\hat{y}_{0} - y^*}^2$,
we have
\begin{align*}
    \BE[\Norm{\hat{y}_{K} - y^*}^2] = \tilde\fO \left(  \hat{R}^2 \exp\left( -\frac{K\mu_h }{\ell_h }\right)+ \frac{\sigma_h^2 }{\mu_h^2 K^{ \frac{2 (p - 1)}{p}}} + \frac{\hat{R}^2}{K^2} \right).
\end{align*}
\end{lem}
\begin{proof}
    We apply the high-probability result established in Lemma~\ref{lemma:high_prob_strongly_convex_converg} to derive the in-expectation result as follows. 
    We first consider the upper bound of $\Norm{\hat{y}_{K+1} - y^*}^2$.
    The update rule of $\hat{y}_{k+1}$ implies that
    \begin{align}\label{single_update_bound}
        \Norm{\hat{y}_{k+1} - \hat{y}_k} = \eta_y \Norm{{\rm clip}(\nabla H(\hat{y}_k; \nu_k), \tau_k)} \leq \eta_y \tau_k.
    \end{align}
    Combining equation (\ref{single_update_bound}) with the choice $\hat{R} \geq \Norm{\hat{y}_0 - y^*}$, we achieve
    \begin{align*}
        \Norm{\hat{y}_{K+1} - y^*} \leq \sum_{k=0}^K\Norm{\hat{y}_{k+1} - \hat{y}_k} + \Norm{\hat{y}_{0} -y^*} \leq \eta_y \sum_{k=0}^K \tau_k + \hat{R}.
    \end{align*}
    According to the parameter settings of $\eta_y$ and $\tau_k$, we have
    \begin{align*}
        \Norm{\hat{y}_{K+1} - y^*} \leq & \sum_{k=0}^K \frac{\hat{R}\exp(- \eta_y \mu_h (1 + k/ 2) )}{120 \ln({4( K+1 )}/{\hat\delta})} + \hat{R} \\
        \leq & \frac{\hat{R}\exp(-\eta_y \mu_h  )}{120 \ln({4( K+1 )}/{\hat\delta})} \cdot \frac{1}{1 - \exp(- \eta_y \mu_h / 2)} + \hat{R} \\
        \leq & \frac{\hat{R} }{30 \eta_y \mu_h \ln({4( K+1 )}/{\hat\delta})}  + \hat{R}  
    \end{align*}
    where the second inequality is due to the summation of a geometric series, and the third inequality follows from $0 < \eta_y \mu_h \leq \ln(B_K)/(K+1)<1$ and  $1 - \exp(-x) \geq x / 2$ for $x \in (0,1)$.
    Following the choice of $\eta_y$, it holds
    \begin{align*}
        \Norm{\hat{y}_{K+1} - y^*} \leq & \fO \left(\frac{\ell_h \hat{R} }{ \mu_h } + \frac{\hat{R} K}{\ln(B_K)\ln({K}/{\hat\delta})} + \hat{R}\right) \\
        = & \fO \left(\frac{\ell_h \hat{R} }{ \mu_h } + \frac{\hat{R} K}{\ln(B_K) \ln({K}/{\hat\delta})}\right).
    \end{align*}
    Using the fact $(a + b)^2 \leq 2a^2 + 2b^2$, we have
    \begin{align}\label{universal_bound}
         \Norm{\hat{y}_{K+1} - y^*}^2 = \fO\left(\frac{  \ell_h^2 \hat{R}^2 }{ \mu_h^2 } + \frac{ \hat{R}^2 K^2}{(\ln(B_K))^2 (\ln({K}/{\hat\delta}))^2}\right).
    \end{align}
    Lemma \ref{lemma:high_prob_strongly_convex_converg} indicates that with probability at least $1 - \hat{\delta}$, one has
    \begin{align}\label{hp_bound_1}
        \Norm{\hat{y}_{K+1} - y^*}^2 =  \!\fO \!\left( \!\max\! \left\{\! \hat{R}^2 \exp\left( -\frac{K\mu_h }{\ell_h \ln(K / \hat\delta)}\right), \frac{\sigma_h^2 (\ln (B_K))^2 \big(\ln\big({    K  }/{\hat\delta}\big)\big)^{\frac{2(p-1)}{p}}}{\mu_h^2 K^{ \frac{2 (p - 1)}{p}}} \!\right\} \!\right).
    \end{align}
    Combining equations (\ref{universal_bound}) with (\ref{hp_bound_1}) by choosing $\hat{\delta} = \mu_h^2   \ell_h^{-2}  (K+1)^{-4}$, we can obtain the following in-expectation result
    \begin{align*}
       & \BE[\Norm{\hat{y}_{K+1} - y^*}^2]\\ \leq & (1 - \hat{\delta})  \fO \left( \max \left\{ \hat{R}^2 \exp\left( -\frac{K\mu_h }{\ell_h \ln(K / \hat\delta)}\right), \frac{\sigma_h^2 (\ln (B_K))^2 \big(\ln\big({    K  }/{\hat\delta}\big)\big)^{\frac{2(p-1)}{p}}}{\mu_h^2 K^{ \frac{2 (p - 1)}{p}}} \right\} \right) \\
        &+ \hat{\delta} \fO\left(\frac{  \ell_h^2 \hat{R}^2}{ \mu_h^2 } + \frac{ \hat{R}^2 K^2}{(\ln(B_K))^2 (\ln({K}/{\hat\delta}))^2}\right) \\
        \leq &\fO \left( \max \left\{ \hat{R}^2 \exp\left( -\frac{K\mu_h }{\ell_h \ln(K \ell_h/ \mu_h)}\right), \frac{\sigma_h^2 (\ln (B_K))^2 \big(\ln\big({    K  }\ell_h/ \mu_h\big)\big)^{\frac{2(p-1)}{p}}}{\mu_h^2 K^{ \frac{2 (p - 1)}{p}}} \right\} \right) \\
        &+  \fO\left(\frac{\hat{R}^2}{K^4} + \frac{\hat{R}^2 \mu_h^2}{\ell_h^2 K^2}\right) \\
        \leq &\fO \left(  \hat{R}^2 \exp\left( -\frac{K\mu_h }{\ell_h \ln(K \ell_h/ \mu_h)}\right) + \frac{\sigma_h^2 (\ln (B_K))^2 \big(\ln\big({    K  }\ell_h/ \mu_h\big)\big)^{\frac{2(p-1)}{p}}}{\mu_h^2 K^{ \frac{2 (p - 1)}{p}}} + \frac{\hat{R}^2 }{K^2} \right) \\
         = &\tilde\fO \left(  \hat{R}^2 \exp\left( -\frac{K\mu_h }{\ell_h }\right)+ \frac{\sigma_h^2 }{\mu_h^2 K^{ \frac{2 (p - 1)}{p}}} + \frac{\hat{R}^2}{K^2} \right).
    \end{align*}
\end{proof}
}

\section{The Improved SFO Bound for \texttt{F$^2$BSA}}\label{appendix:F2BDA}
In this section, we improve the upper bound by applying stochastic gradient descent to solve the stochastic strongly convex problem provided by
\citet[Theorem D.2]{chen2025near}. 
By replacing Theorem D.2 of \citet{chen2025near} with the following theorem, we directly obtain an improved SFO complexity of $\tilde\fO(\kappa^{11} \sigma^2 \epsilon^{-6})$ for the stochastic bilevel optimization under the bounded variance.

\begin{thm}
Suppose $h(x) \colon \BR^d \to \BR$ is $\beta$-smooth and $\alpha$-strongly convex.
Consider the following update of stochastic gradient descent
\begin{align*}
    x_{t+1} = x_t - \frac{1}{\beta} \nabla h(x_t; \fB_t),
\end{align*}
for $t=0,\dots,T-1$, where the gradient estimator satisfies
\begin{align*}
    \BE[\nabla h(x_t; \fB_t)] = \nabla h(x_t)
    \qquad\text{and}\qquad
    \BE[\Norm{\nabla h(x_t; \fB_t) - \nabla h(x_t)}^2] \leq \frac{\sigma^2}{B}.
\end{align*}
Then it holds that
 \begin{align*}
        \BE[ \Norm{x_{T} - x^*}^2] \leq \left(1 - \frac{\alpha}{\beta}\right)^T\BE[ \Norm{ x_0 - x^*}^2] + \frac{\sigma^2}{\alpha \beta B}.
    \end{align*}
where $x^* = \argmin_{x \in \BR^d} h(x)$.
\end{thm}

\begin{proof}
    According to the update rule, we have
    \begin{align*}
        \Norm{x_{t+1} - x^*}^2 = & \Norm{x_{t} - x^*}^2  - \frac{2}{\beta} \langle x_t - x^*, \nabla h(x_t; \fB_t) \rangle + \frac{1}{\beta^2}\Norm{\nabla h(x_t; \fB_t)}^2.
    \end{align*}
    Taking the expectation on both sides of above inequality, we have
    \begin{align*}
       \BE[ \Norm{x_{t+1} - x^*}^2] = & \BE[\Norm{x_{t} - x^*}^2]  - \frac{2}{\beta} \BE[\langle x_t - x^*, \nabla h(x_t) \rangle] + \frac{1}{\beta^2}\BE[\Norm{\nabla h(x_t; \fB_t)}^2] \\
       \leq & \BE[\Norm{x_{t} - x^*}^2]  - \frac{2}{\beta} \BE[\langle x_t - x^*, \nabla h(x_t) \rangle] + \frac{\sigma^2}{\beta^2 B} \\
       \leq & \BE[\Norm{x_{t} - x^*}^2]  - \frac{\alpha}{\beta}\BE[ \Norm{ x_t - x^*}^2] + \frac{2}{\beta}\BE[h(x^*) - h(x_t)] + \frac{\sigma^2}{\beta^2 B} \\
       \leq & \left(1 - \frac{\alpha}{\beta}\right)\BE[ \Norm{ x_t - x^*}^2] + \frac{\sigma^2}{\beta^2 B},
    \end{align*}
    where the first inequality is due to the bounded variance assumption, the second inequality follows from the strong convexity assumption.
    The last inequality comes from the optimality of $x^*$.
    Unrolling the above inequality for $T$ iterations, we have
    \begin{align*}
        \BE[ \Norm{x_{T} - x^*}^2] \leq \left(1 - \frac{\alpha}{\beta}\right)^T\BE[ \Norm{ x_0 - x^*}^2] + \frac{\sigma^2}{\alpha \beta B}.
    \end{align*}
\end{proof}

{\section{Lower Bound for Stochastic Bilevel Optimization}\label{appendix:proof_lower_bound}

In this section, we present the proof of the lower bound for the stochastic bilevel optimization problem shown in Section \ref{sec:lower_bound_bilevel}.
The construction of our lower bound is based on the zero-chain property of the hard instance. 
In particular, we first introduce the notion of ``progress function'' 
\citep[Section 3.1]{arjevani2023lower} as 
\begin{align*}
    {\rm prog}_\alpha(x) \coloneqq \max \{i \geq 0 \mid |x_{(i)}| > \alpha\}
\end{align*}
for some $\alpha\in[0,1)$, where $x_{(i)}$ is the $i$-th coordinate of the vector $x$ for $i\geq 1$ and we let $x_{(0)}\equiv 1$ in convention.
We then introduce the probabilistic zero-chain property as follows.
\begin{dfn}[{\citet[Definition~2]{arjevani2023lower}}]
    We say a stochastic gradient function $h(x; \xi)$ is probability-$q$ zero-chain if it holds 
    \begin{align*}
       & \BP\big(\forall x: {\rm prog}_0 (h(x; \xi)) \leq {\rm prog}_{\frac{1}{4}}(x) ~\text{and}~ h(x; \xi) = h(x_{\leq {\rm prog}_{\frac{1}{4}}(x)}; \xi)\big) \geq 1 - q, \\
       & \BP\big(\forall x: {\rm prog}_0 (h(x; \xi)) \leq 1 + {\rm prog}_{\frac{1}{4}}(x) ~\text{and}~ h(x; \xi) = h(x_{\leq 1 + {\rm prog}_{\frac{1}{4}}(x)}; \xi)\big) = 1 .
    \end{align*}
    where the $j$-th coordinate of the vector $x_{\leq i}$ is 
    {$(x_{\leq i})_{(j)} = x_j {\mathbbm 1}\{i\leq j\}$}.
\end{dfn}
We then introduce the following stochastic nonconvex function together with its stochastic gradient estimator.

\begin{dfn}[{\citet[Section 3.2]{arjevani2023lower}}]\label{def:zero-chain}
    \!Define the function \mbox{$\Bar{f}^{\rm \textnormal{nc-s}}\!:\BR^T\!\to \BR$} and its stochastic gradient estimator $\hat{\nabla}\Bar{F}^{\textnormal{nc-s}}(x;\xi)$ as
    \begin{align*}
        \Bar{f}^{\textnormal{nc-s}}(x) =  
        -\Psi(1)\Phi(x_{(1)}) + \sum_{i=2}^T\left(\Psi(-x_{(i-1)}) \Phi(- x_{(i)}) - \Psi(x_{(i-1)}) \Phi(x_{(i)})\right)
    \end{align*}
    and
    \begin{align*}
        (\hat{\nabla}\Bar{F}^{\textnormal{nc-s}}(x;\xi))_{(j)} = (\nabla\Bar{f}^{\textnormal{nc-s}}(x))_{(j)} \left(1 + {\mathbbm 1}\left\{j>{\rm prog}_{\frac{1}{4}}(x) \right\}\cdot\left(\frac{\xi}{q} - 1\right)\right), 
    \end{align*}
    respectively, where $x=[x_{(1)},\dots,x_{(T)}]^\top\in\BR^T$ and $\xi \sim {\rm Bernoulli}(q)$ with some $q\in[0,1]$.  
    Here, the component functions $\Psi: \BR \to \BR$ and $\Phi : \BR \to \BR$ are defined as
    \begin{align*}
        \Psi(t) = \begin{cases}
        0, & t \leq \dfrac{1}{2}, \\[0.3cm]
        \exp\left(1 - \dfrac{1}{(2t - 1)^2}\right)  & t > \dfrac{1}{2},
        \end{cases}
        \qquad {\rm and} \qquad
        \Phi(t) = \sqrt{\rm e} \int_{-\infty}^t \exp\left(-\frac{\tau^2}{2}\right) {\rm d} \tau,
    \end{align*}
    respectively.
\end{dfn}

The stochastic first-order oracle $\hat{\nabla}\Bar{F}^{\textnormal{nc-s}}(x;\xi)$ has the following property.

\begin{lem}[{\citet[Lemma~3]{arjevani2023lower}}]\label{lemma:zero-chain}
Following Definition \ref{def:zero-chain},
the stochastic first-order oracle $\hat{\nabla}\bar{F}^{\textnormal{nc-s}}(x; \xi)$ is
an unbiased estimator of  $\nabla \bar{f}^{\textnormal{nc-s}}(x)$ and it is
a probability-$q$ zero chain.
\end{lem}


Consequently, we define the following hard instance for our lower bound construction.

\begin{dfn}[{\citet[Section F.2]{cutkosky2023optimal}}]\label{dfn:hard-nonconvex-func}
    For all $U\in{\rm Orth}(d, T) \coloneqq \{ U \in \BR^{d \times T} \mid U^\top U = I_T\}$,
    we define the functions
    \begin{align*}
        \Bar{f}_U^{\textnormal{nc-rs}}(x) =& \Bar{f}^{\textnormal{nc-s}}(U^{\top} \rho(x)) + \frac{1}{10} \langle x, \rho(x) \rangle
    \end{align*}
    and
    \begin{align*}    
        \hat{\nabla} \Bar{F}_U^{\textnormal{nc-rs}}(x;\xi) = &J(x)^{\top} U \hat{\nabla}\Bar{F}^{\textnormal{nc-s}}(U^\top \rho(x);\xi) + \frac{1}{10} (\rho(x) + J(x)^\top x),
    \end{align*}
    where $\xi \sim {\rm Bernoulli}(q)$ with some $q\in[0,1]$, $\rho(x) = x / \sqrt{1 + \Norm{x}^2 / R^2}$ with $R = 230 \sqrt{T}$, and~$J(x)$ is the Jacobian of $\rho$.
\end{dfn}

We present several useful properties for  $\Bar{F}^{\textnormal{nc-rs}}$ 
and 
$\hat{\nabla} \Bar{F}^{\textnormal{nc-rs}})$ as follows

\begin{lem}[{\citet[Lemmas 25 and 26]{cutkosky2023optimal}}]\label{lemma:hard_inst_prop} Following definitions of $\Bar{f}_U^{\textnormal{nc-rs}}$ and $\hat{\nabla} \Bar{F}_U^{\textnormal{nc-rs}}$, we have:
    \begin{itemize}
        \item It holds $\Bar{f}_U^{\textnormal{nc-rs}}(0) - \inf_{x \in \BR^T} \Bar{f}_U^{\textnormal{nc-rs}}(x) \leq 12 T$;
        \item  It holds $\Norm{\hat\nabla \bar{f}_U^{\textnormal{nc-rs}}(x)} \leq 92 \sqrt{T}$ for all $x\in\BR^T$;
        \item The gradient of $\Bar{f}_U^{\textnormal{nc-rs}}(x)$  is 156-Lipschitz continuous;
        \item Let $\{x_t\}_{t \in \BN}$ be the sequence generated by a randomized algorithm $A \in \fA^{\rm sfo}$.
        For any deterministic bijective mapping $\fM: \BR^T \to \BR^T$, with probability at least $1 - \delta$, it holds $\Norm{\nabla \Bar{f}_U^{\textnormal{nc-rs}}(\fM(x_t))} \geq 1 / 2$ for all $t \leq (T - \log(2 / \delta)) / (2q)$.
    \end{itemize}
\end{lem}

In addition, we can show the following property of the stochastic gradient $\hat{\nabla} \Bar{F}_U^{\textnormal{nc-rs}}(x; \xi)$.
\begin{lem}\label{lemma:bilevel_p_bcm}
For given $p\in(1,2]$ and $U\in{\rm Orth}(d, T)$,
it holds 
\begin{align*}
\BE[\hat{\nabla} \Bar{F}_U^{\textnormal{nc-rs}}(x; \xi)]=\nabla \Bar{f}_U^{\textnormal{nc-rs}}(x)
\quad\text{and}\quad
\BE\left[\Norm{\hat{\nabla} \bar{F}_U^{\textnormal{nc-rs}}(x; \xi) - \nabla \bar{f}_U^{\textnormal{nc-rs}}(x)}^p\right] 
\leq \frac{2 \cdot 23^p}{q^{p-1}}
\end{align*}
for all $x\in\BR^T$.
\end{lem}
\begin{proof}
    We first show that $\hat{\nabla} \Bar{F}^{\textnormal{nc-s}}(x; \xi)$ is unbiased estimator of $\nabla \bar{f}^{\textnormal{nc-s}}(x)$.
    By the definition of $\hat{\nabla} \Bar{F}^{\textnormal{nc-s}}(x; \xi)$ and Lemma \ref{lemma:zero-chain}, it holds
    \begin{align*}
        \BE[(\hat{\nabla} \Bar{F}^{\textnormal{nc-s}}(x; \xi))_{(j) }] = (\nabla \Bar{f}^{\textnormal{nc-s}}(x))_{(j) } \quad \text{for all}\quad  j \leq {\rm prog}_{\frac{1}{4}}(x);
    \end{align*}
    and
    \begin{align*}
        \BE[(\hat{\nabla} \Bar{F}^{\textnormal{nc-s}}(x;\xi))_{(j) }] = (\nabla \Bar{f}^{\textnormal{nc-s}}(x))_{(j) }\left(1 + \frac{\BE[\xi]}{q} - 1\right) = (\nabla \Bar{f}^{\textnormal{nc-s}}(x))_{(j)}~~{\rm for}~~j = {\rm prog}_{\frac{1}{4}}(x) + 1.
    \end{align*}
    Therefore, we know that $\hat{\nabla} \Bar{F}^{\textnormal{nc-s}}(x; \xi)$ is an unbiased estimator of $\nabla \Bar{f}^{\textnormal{nc-s}}(x)$.
    
    We then bound the difference of $\hat{\nabla} \bar{F}^{\textnormal{nc-s}}(x; \xi)$ and $\nabla \bar{f}^{\textnormal{nc-s}}(x)$ as follows
    \begin{align}\label{p-var-grad}
    \begin{split}        
        & \BE\left[\Norm{\hat{\nabla} \bar{F}^{\textnormal{nc-s}}(x; \xi) - \nabla \bar{f}^{\textnormal{nc-s}}(x)}^p\right] \\
        \leq & \Norm{\nabla \bar{f}^{\textnormal{nc-s}}(x)}_\infty^p \BE\left[\left|\frac{\xi}{q} - 1\right|^p\right] \\
        = & \Norm{\nabla \bar{f}^{\textnormal{nc-s}}(x)}_\infty^p \left(1 - q + \frac{(1 - q)^p}{q^{p - 1}}\right) \\
        \leq & \frac{ 2 \Norm{\nabla \bar{f}^{\textnormal{nc-s}}(x)}_\infty^p}{q^{p-1}} \\
        \leq & \frac{2 \cdot 23^p}{q^{p-1}},
    \end{split}
    \end{align}
    where the first step is based on definitions of $\hat{\nabla} \bar{F}^{\textnormal{nc-s}}(x; \xi)$ and $\nabla \bar{f}^{\textnormal{nc-s}}(x)$, the second step is based on $\xi \sim {\rm Bernoulli}(q)$, and the last step is based on the result $\Norm{\nabla \bar{f}^{\textnormal{nc-s}}(x)}_\infty \leq 23$ from Lemma 2 of \citet{arjevani2023lower}.
    
    Now, we focus on the properties $\hat{\nabla} \Bar{f}^{\textnormal{nc-rs}}$. According to the definition of $\hat{\nabla} \Bar{f}^{\textnormal{nc-rs}}$ and the unbiasedness of $\hat{\nabla}\Bar{F}^{\textnormal{nc-s}}$, it holds 
    \begin{align*}
       & \BE[ \hat{\nabla} \Bar{F}_U^{\textnormal{nc-rs}}(x;\xi) ]  \\
       = &J(x)^{\top} U \BE[\hat{\nabla}\Bar{F}^{\textnormal{nc-s}}(U^\top \rho(x); \xi)] + \frac{1}{10} (\rho(x) + J(x)^\top x) \\
       = & J(x)^{\top} U \nabla\Bar{f}^{\textnormal{nc-s}}(U^\top \rho(x)) + \frac{1}{10} (\rho(x) + J(x)^\top x) \\
       = & \nabla \bar{f}_U^{\textnormal{nc-rs}}(x).
    \end{align*}
    This implies that $\hat{\nabla} \Bar{F}_U^{\textnormal{nc-rs}}(x;  \xi)$ is an unbiased estimator of $\nabla \Bar{f}_U^{\textnormal{nc-rs}}(x)$.
    
    Recall that we have defined $J(x)$ as the Jacobian of the function $\rho(x) = x/ \sqrt{1 + \Norm{x}^2 / R^2}$.
    We let $s = 1 + \Norm{x}^2 / R^2$, then it holds
    \begin{align*}
         J(x)_{i, i} = \frac{\partial \rho_i}{\partial x_i} = \frac{1}{\sqrt{s}} - \frac{x_i^2 / R^2}{s^{3/2}}
         \qquad\text{and}\qquad
         J(x)_{i,j} = \frac{\partial \rho_i}{\partial x_j} = - \frac{x_i x_j / R^2}{s^{3/2}}.
    \end{align*}
    which indicates
    \begin{align*}
        J(x) = \frac{1}{\sqrt{s}} I - \frac{1}{R^2 s^{3/2}} x x^{\top} = \frac{I-\rho(x)\rho(x)^\top/R^2}{\sqrt{1 + \Norm{x}^2 / R^2}}.
    \end{align*}
    According to Lemma 13 of \cite{arjevani2023lower}, we know that
    \begin{align}\label{eq:J}
        \Norm{J(x)} \leq  1.
    \end{align}
    
    Finally, bound the $p$-th central moment of $\hat{\nabla} \bar{F}_U^{\textnormal{nc-rs}}(x;  \xi)$ as
    \begin{align*}
       & \BE\left[\Norm{\hat{\nabla} \bar{F}_U^{\textnormal{nc-rs}}(x; \xi) - \nabla \bar{F}_U^{\textnormal{nc-rs}}(x)}^p\right] \\
       = & \BE\left[ \Norm{J(x)^{\top} U \hat{\nabla} \bar{F}^{\textnormal{nc-s}}(U^\top \rho(x); \xi) - J(x)^{\top} U \nabla \bar{f}^{\textnormal{nc-s}}(U^\top \rho(x))}^p  \right] \\
       \leq &   \BE\left[\Norm{J(x)}^p \Norm{U}^p \Norm{\hat{\nabla} \bar{F}^{\textnormal{nc-s}}(U^\top \rho(x); \xi) - \nabla \bar{f}^{\textnormal{nc-s}}(U^\top \rho(x))}^p \right] \\
       \leq & \frac{2 \cdot 23^p}{q^{p-1}},
    \end{align*}
    where the first inequality is due to the submultiplicativity of the matrix norm, 
    the second inequality is based on 
    the fact $\Norm{U} \leq 1$ and equations (\ref{p-var-grad}) and (\ref{eq:J}).
\end{proof}

We now present our hard instance for the stochastic bilevel optimization problem.


\subsection{Proof of Theorem \ref{thm:bilevel_lower_bound}}
\begin{proof}
Following Definition \ref{dfn:hard-nonconvex-func}, we consider the 
bilevel optimization problem
\begin{align}\label{prob:lower}
\begin{split}    
   & \min_{x\in\BR^{d_x}} \varphi(x) \coloneqq  f(x,y^*(x)), \\
   &~ {\rm s.t.}~~y^*(x) \coloneqq \argmin_{y \in \BR^{d_y}} g(x,y),
\end{split}
\end{align}
where
\begin{align}\label{bilevel_hard_instance-1}
\begin{split}    
    & f(x,y) \coloneqq \frac{L_f \beta^2}{156} \bar{f}_U^{\textnormal{nc-rs}}\left(\frac{y}{ \beta}\right),~~
    g(x, y) \coloneqq \frac{\mu}{2} \Norm{y}^2 - L_g \langle x, y \rangle, \\
    & \beta = \frac{624 \epsilon}{L_f \kappa_g}, \quad~~ \kappa_g = \frac{L_g}{\mu},
    \quad~~ \text{and} \quad~~ U\in{\rm Orth}(d, T).
\end{split}
\end{align}
Furthermore, we introduce the the stochastic first-order oracles for the upper- and lower-level functions $f:\BR^T\times\BR^T\to\BR$ and $g:\BR^T\times\BR^T\to\BR$ as 
\begin{align}\label{bilevel_hard_instance-3}
\hat\nabla F_U(x, y; \xi) = \left(0_T, \frac{L_f \beta}{156} \hat{\nabla} \bar{F}_U^{\textnormal{nc-rs}}\left(\frac{y}{\beta};\xi\right)\right) \qquad {\rm and} \qquad \hat{\nabla} G(x,y) = \nabla g(x,y),
\end{align}
respectively .
By appropriate parameter settings,
we verify that the bilevel problem (\ref{prob:lower}) belongs to $\fF^{\textnormal{nc-scq}}(C_f, L_f, L_g, \mu, \Delta)$ and the stochastic gradient estimator $\hat\nabla F_U(x, y; \xi)$ is unbiased and its $p$-th central moment is bounded by~$\sigma^p$ as follows. 


\begin{itemize}[leftmargin=1cm,topsep=0.05cm,itemsep=0.03cm]

\item 
We take
\begin{align}\label{eq:T-lower-bound}
    T = \floor{\frac{156 \Delta}{12 L_f \beta^2}}, 
\end{align}
then the definition of $f(x,y)$ in equation (\ref{bilevel_hard_instance-1}) implies that 
\begin{align*}
    \Norm{\nabla f(x,y)} = \frac{L_f \beta}{156}\Norm{\nabla \bar{f}_U^{\textnormal{nc-rs}}\left(\frac{y}{\beta}\right)} \leq L_f \beta \sqrt{T} \leq  \sqrt{13 L_f \Delta} \leq C_f,
\end{align*}
where the first inequality is based on Lemma \ref{lemma:hard_inst_prop}, the second equality is based on the setting of $T$, and the last step is base on the condition $C_f \geq \sqrt{13 L_f \Delta}$ in the statement of this theorem.
Hence, we know that the function $f(x,y)$ is $C_f$-Lipschitz, also  $C_f$-Lipschitz in $y$.

\item The definition of $f(x,y)$ in equation (\ref{bilevel_hard_instance-1}) implies that
\begin{align*}
    \Norm{\nabla^2 f(x, y)} = \frac{L_f }{156} \Norm{\nabla^2 \bar{f}_U^{\textnormal{nc-rs}}\left(\frac{y}{\beta}\right)} \leq L_f,
\end{align*}
where the last inequality is due to the third statement of Lemma \ref{lemma:hard_inst_prop}.
Hence, the function $f(x, y)$ is $L_f$-gradient Lipschitz.   

\item The definition of $g(x,y)$ in equation (\ref{bilevel_hard_instance-1}) directly implies it is $L_g$-gradient Lipschitz and $\mu$-strongly convex in $y$.

\item Note that the definition 
\begin{align*}
     \quad g(x, y) \coloneqq \frac{\mu}{2} \Norm{y}^2 - L_g \langle x, y \rangle, 
\end{align*}
implies that $y^*(x) = L_g x / \mu$.
Hence, the definition 
\begin{align*}
f(x, y) \coloneqq \frac{L_f \beta^2}{156} \bar{f}_U^{\textnormal{nc-rs}}\left(\frac{y}{\beta}\right)    
\end{align*}
implies
\begin{align}\label{eq:phi}
    \varphi(x) = f(x, y^*(x)) = \frac{L_f \beta^2}{ 156} \bar{f}_U^{\textnormal{nc-rs}} \left(\frac{\kappa_g x}{\beta}\right),
\end{align}
where $\kappa_g \coloneqq L_g / \mu$.
This implies
\begin{align*}
   & \varphi(0) - \inf_{x \in \BR^T} \varphi(x) 
\overset{(\ref{eq:phi})}= \frac{L_f \beta^2}{ 156} \left(\bar{f}_U^{\textnormal{nc-rs}} \left(0\right) - \inf_{x \in \BR^T}
\bar{f}_U^{\textnormal{nc-rs}} \left(\frac{\kappa_g x}{\beta}\right)\right) \\
\leq & \frac{L_f \beta^2}{ 156} \cdot 12 T \leq \Delta,
\end{align*}
where the last line is due to  first statement in Lemma \ref{lemma:hard_inst_prop} and the setting of $T$ in (\ref{eq:T-lower-bound}).

\item We take 
    \begin{align*}
        q =  2^{\frac{1}{p-1}} \left(\frac{23 L_f \beta}{156 \sigma}\right)^{\frac{p}{p-1}},
    \end{align*}
    then the $p$-th central moment of $\hat\nabla F$ is bounded as
    \begin{align*}
    \begin{split}        
        & \BE\left[\Norm{\hat\nabla F_U(x, y; \xi) - \nabla f(x,y)}^p\right] \\
        = & \left(\frac{L_f \beta}{156}\right)^p  \BE\left[\Norm{\hat{\nabla} \bar{F}_U^{\textnormal{nc-rs}}\left(\frac{y}{\beta}; \xi\right) - \nabla \bar{f}_U^{\textnormal{nc-rs}}\left(\frac{y}{\beta}\right)}^p\right] \\
        \leq &  \left(\frac{L_f \beta}{156}\right)^p  \cdot \frac{2 \cdot 23^p}{q^{p-1}} = \sigma^p,
    \end{split}
    \end{align*}
    where the last step is based on   
    Lemma \ref{lemma:bilevel_p_bcm} and the setting of $q$.
\end{itemize}

Following Definition \ref{dfn:random_algo}, we consider the randomized algorithm $A$ consists of the mappings $\{A_t\}_{t \in \BN}$ recursively defined as
\begin{align*}
    (x_t, y_t) = A_t(r, \hat{\nabla} \bar{F}_U^{\textnormal{nc-rs}}(y_0;  \xi),  \hat\nabla G(x_0, y_0), \dots, \hat{\nabla}\bar{F}_U^{\textnormal{nc-rs}}(y_{t-1};  \xi), \hat\nabla G(x_{t-1}, y_{t-1})),
\end{align*}
where $r$ is the random seed.
Note that $\hat\nabla G(x, y)=\nabla g(x,y)$ is a deterministic mapping of~$(x, y)$, then there exists mappings $\{A^1_{t}\}_{t \in \BN}$ satisfying the recursion
\begin{align*}
    (x_t, y_t) = A^1_{t} (r, \hat{\nabla} \bar{F}_U^{\textnormal{nc-rs}}(y_0; \xi), x_0, y_0, \dots, \hat{\nabla} \bar{F}_U^{\textnormal{nc-rs}}(y_{t-1}; \xi), x_{t-1}, y_{t-1}).
\end{align*}
This implies there exist mappings $\{A^2_t\}_{t \in \BN}$ satisfying
\begin{align*}
    (x_t, y_t) = A^2_t(r, \hat{\nabla} \bar{F}_U^{\textnormal{nc-rs}}(y_0;  \xi), \dots, \hat{\nabla} \bar{F}_U^{\textnormal{nc-rs}}(y_{t-1};  \xi)).
\end{align*}
Recall that equation (\ref{eq:phi})
implies
\begin{align}\label{eq:grad-phi}
     \nabla\varphi(x) = \frac{L_f \beta\kappa_g}{ 156} \nabla\bar{f}_U^{\textnormal{nc-rs}} \left(\frac{\kappa_g x}{\beta}\right).
\end{align}
We take
\begin{align*}
    \beta = \frac{624 \epsilon}{L_f \kappa_g},
\end{align*}
then using Lemma \ref{lemma:hard_inst_prop} with $\delta = 1/2$ implies
\begin{align}\label{eq:grad-phi-lower}
   \BE[ \Norm{\nabla \varphi(x_t)}] \overset{(\ref{eq:grad-phi})}= \frac{L_f \beta \kappa_g}{156} \BE\left[\Norm{\nabla \bar{f}_U^{\textnormal{nc-rs}}\left(\frac{\kappa_g x_t}{\beta}\right)}\right] \geq \frac{L_f \beta \kappa_g}{624} = \epsilon
\end{align}
for all $t \leq (T - \log(2 / \delta)) / (2q)$.
Therefore, achieving an $\epsilon$-stationary point of the function~$\varphi(x)$ requires at least the SFO complexity of 
\begin{align*}
        \Omega\left(\frac{T - \log(2 / \delta)}{2q}\right) =\Omega\left(\frac{T}{q}\right) =\Omega\left(\frac{\Delta}{L_f \beta^2} \cdot \frac{\sigma^{\frac{p}{p-1}}}{L_f^{\frac{p}{p-1}} \beta^{\frac{p}{p-1}}}\right) = \Omega\left(\frac{\Delta L_f \sigma^{\frac{p}{p-1}} \kappa_g^{\frac{3p-2}{p-1}}}{ \epsilon^{\frac{3p-2}{p-1}}} \right).
    \end{align*}
\end{proof}
}

\section{Discussion on the Concurrent Work}\label{sec:concurrent}

In a concurrent work, \cite{zhang2025nonconvex} studied decentralized stochastic bilevel optimization under heavy-tailed noise, while their pre-print version on arXiv does not provide the proofs for theoretical results. 
We clarify the difference between our setting and that of related work as follows.
\begin{itemize}[leftmargin=1cm,topsep=0.05cm,itemsep=0.03cm]
    \item The concurrent work proposed a normalized stochastic gradient descent algorithm based on momentum-based variance reduction \citep{cutkosky2019momentum}, which requires all stochastic components $F(x,y;\xi)$ and $G(x,y;\zeta)$ have the Lipschitz continuous gradient \citep[Assumption 3.1]{zhang2025nonconvex}. In contrast, our algorithm design and analysis only require the upper-level function $f(x,y)$ and the lower-level function $g(x,y)$ have the Lipschitz continuous gradient, which is weaker than that of \cite{zhang2025nonconvex}.
    \item The concurrent work assumes both the lower-level function $g(x,y)$ and the penalized function $\delta f(x,y)+g(x,y)$ satisfies Polyak--Lojasiewicz (PL) condition with respect to $y$ \citep[Assumption 3.1]{zhang2025nonconvex}, where $\delta>0$ is the penalized parameter which is similar to the role of $1/\lambda$ in our analysis.
    Although the PL condition on~$g(x,y)$ is weaker than the strong convexity used in our work, the PL condition on $\delta f(x,y)+g(x,y)$ is difficult to verify since the choice of their $\delta$ depends on the iteration number of the algorithm \cite[Theorem 4.5]{zhang2025nonconvex}.
\end{itemize}

\bibliographystyle{plainnat}
\bibliography{reference}

@article{chen2025near,
  title={Near-optimal nonconvex-strongly-convex bilevel optimization with fully first-order oracles},
  author={Chen, Lesi and Ma, Yaohua and Zhang, Jingzhao},
  journal={Journal of Machine Learning Research},
  volume={26},
  number={109},
  pages={1--56},
  year={2025}
}

@article{hong2020two,
  title={A two-timescale stochastic algorithm framework for bilevel optimization: Complexity analysis and application to actor-critic},
  author={Hong, Mingyi and Wai, Hoi-To and Wang, Zhaoran and Yang, Zhuoran},
  journal={SIAM Journal on Optimization},
  volume={33},
  number={1},
  pages={147--180},
  year={2023},
  publisher={SIAM}
}

@inproceedings{franceschi2018bilevel,
  title={Bilevel programming for hyperparameter optimization and meta-learning},
  author={Franceschi, Luca and Frasconi, Paolo and Salzo, Saverio and Grazzi, Riccardo and Pontil, Massimiliano},
  booktitle={International Conference on Machine Learning},
  pages={1568--1577},
  year={2018}
}

@inproceedings{pedregosa2016hyperparameter,
  title={Hyperparameter optimization with approximate gradient},
  author={Pedregosa, Fabian},
  booktitle={International Conference on Machine Learning},
  pages={737--746},
  year={2016}
}

@incollection{feurer2019hyperparameter,
  title={Hyperparameter optimization},
  author={Feurer, Matthias and Hutter, Frank},
  booktitle={Automated machine learning: Methods, systems, challenges},
  pages={3--33},
  year={2019},
  publisher={Springer International Publishing Cham}
}

@book{sutton1998introduction,
  title={Introduction to Reinforcement Learning},
  author={Sutton, Richard S. and Barto, Andrew G.},
  volume={135},
  year={1998},
  publisher={MIT press Cambridge}
}

@article{li2017deep,
  title={Deep reinforcement learning: An overview},
  author={Li, Yuxi},
  journal={arXiv preprint arXiv:1701.07274},
  year={2017}
}

@inproceedings{finn2017model,
  title={Model-agnostic meta-learning for fast adaptation of deep networks},
  author={Finn, Chelsea and Abbeel, Pieter and Levine, Sergey},
  booktitle={International Conference on Machine Learning},
  pages={1126--1135},
  year={2017}
}

@inproceedings{fallah2020convergence,
  title={On the convergence theory of gradient-based model-agnostic meta-learning algorithms},
  author={Fallah, Alireza and Mokhtari, Aryan and Ozdaglar, Asuman},
  booktitle={International Conference on Artificial Intelligence and Statistics},
  pages={1082--1092},
  year={2020}
}

@inproceedings{rajeswaran2019meta,
  title={Meta-learning with implicit gradients},
  author={Rajeswaran, Aravind and Finn, Chelsea and Kakade, Sham M. and Levine, Sergey},
  booktitle={Advances in Neural Information Processing systems},
  volume={32},
  pages={113--124},
  year={2019}
}

@article{ji2022theoretical,
  title={Theoretical convergence of multi-step model-agnostic meta-learning},
  author={Ji, Kaiyi and Yang, Junjie and Liang, Yingbin},
  journal={Journal of Machine Learning Research},
  volume={23},
  number={29},
  pages={1--41},
  year={2022}
}

@inproceedings{liu2018darts,
  title={{DARTS}: Differentiable architecture search},
  author={Liu, Hanxiao and Simonyan, Karen and Yang, Yiming},
  booktitle={International Conference on Learning Representations},
  year={2019}
}

@article{zoph2016neural,
  title={Neural architecture search with reinforcement learning},
  author={Zoph, Barret and Le, Quoc V.},
  journal={arXiv preprint arXiv:1611.01578},
  year={2016}
}

@inproceedings{cutkosky2019momentum,
  title={Momentum-based variance reduction in non-convex {SGD}},
  author={Cutkosky, Ashok and Orabona, Francesco},
  booktitle={Advances in Neural Information Processing Systems},
  pages={15236--15245},
  volume={32},
  year={2019}
}

@article{zhang2025nonconvex,
  title={Nonconvex Decentralized Stochastic Bilevel Optimization under Heavy-Tailed Noises},
  author={Zhang, Xinwen and Zhang, Yihan and Gao, Hongchang},
  journal={arXiv preprint arXiv:2509.15543},
  year={2025}
}

@inproceedings{zhang2021idarts,
  title={{iDARTS}: Differentiable architecture search with stochastic implicit gradients},
  author={Zhang, Miao and Su, Steven W. and Pan, Shirui and Chang, Xiaojun and Abbasnejad, Ehsan M. and Haffari, Reza},
  booktitle={International Conference on Machine Learning},
  pages={12557--12566},
  year={2021}
}

@article{ghadimi2018approximation,
  title={Approximation methods for bilevel programming},
  author={Ghadimi, Saeed and Wang, Mengdi},
  journal={arXiv preprint arXiv:1802.02246},
  year={2018}
}

@inproceedings{chen2021closing,
  title={Closing the gap: Tighter analysis of alternating stochastic gradient methods for bilevel problems},
  author={Chen, Tianyi and Sun, Yuejiao and Yin, Wotao},
  booktitle={Advances in Neural Information Processing Systems},
  volume={34},
  pages={25294--25307},
  year={2021}
}

@inproceedings{chen2022single,
  title={A single-timescale method for stochastic bilevel optimization},
  author={Chen, Tianyi and Sun, Yuejiao and Xiao, Quan and Yin, Wotao},
  booktitle={International Conference on Artificial Intelligence and Statistics},
  pages={2466--2488},
  year={2022}
}

@article{khanduri2021near,
  title={A near-optimal algorithm for stochastic bilevel optimization via double-momentum},
  author={Khanduri, Prashant and Zeng, Siliang and Hong, Mingyi and Wai, Hoi-To and Wang, Zhaoran and Yang, Zhuoran},
  journal={Advances in Neural Information Processing Systems},
  volume={34},
  pages={30271--30283},
  year={2021}
}

@inproceedings{ji2021bilevel,
  title={Bilevel optimization: Convergence analysis and enhanced design},
  author={Ji, Kaiyi and Yang, Junjie and Liang, Yingbin},
  booktitle={International Conference on Machine Learning},
  pages={4882--4892},
  year={2021}
}

@inproceedings{yang2021provably,
  title={Provably faster algorithms for bilevel optimization},
  author={Yang, Junjie and Ji, Kaiyi and Liang, Yingbin},
  booktitle={Advances in Neural Information Processing Systems},
  volume={34},
  pages={13670--13682},
  year={2021}
}

@inproceedings{kwon2023fully,
  title={A fully first-order method for stochastic bilevel optimization},
  author={Kwon, Jeongyeol and Kwon, Dohyun and Wright, Stephen and Nowak, Robert D.},
  booktitle={International Conference on Machine Learning},
  pages={18083--18113},
  year={2023}
}

@inproceedings{shen2023penalty,
  title={On penalty-based bilevel gradient descent method},
  author={Shen, Han and Chen, Tianyi},
  booktitle={International Conference on Machine Learning},
  pages={30992--31015},
  year={2023}
}

@inproceedings{simsekli2019tail,
  title={A tail-index analysis of stochastic gradient noise in deep neural networks},
  author={Simsekli, Umut and Sagun, Levent and Gurbuzbalaban, Mert},
  booktitle={International Conference on Machine Learning},
  pages={5827--5837},
  year={2019}
}

@inproceedings{zhang2020adaptive,
  title={Why are adaptive methods good for attention models?},
  author={Zhang, Jingzhao and Karimireddy, Sai Praneeth and Veit, Andreas and Kim, Seungyeon and Reddi, Sashank and Kumar, Sanjiv and Sra, Suvrit},
  booktitle={Advances in Neural Information Processing Systems},
  volume={33},
  pages={15383--15393},
  year={2020}
}

@inproceedings{ahn2023linear,
  title={Linear attention is (maybe) all you need (to understand transformer optimization)},
  author={Ahn, Kwangjun and Cheng, Xiang and Song, Minhak and Yun, Chulhee and Jadbabaie, Ali and Sra, Suvrit},
  booktitle={International Conference on Learning Representations},
  year={2024}
}

@inproceedings{garg2021proximal,
  title={On proximal policy optimization’s heavy-tailed gradients},
  author={Garg, Saurabh and Zhanson, Joshua and Parisotto, Emilio and Prasad, Adarsh and Kolter, Zico and Lipton, Zachary and Balakrishnan, Sivaraman and Salakhutdinov, Ruslan and Ravikumar, Pradeep},
  booktitle={International Conference on Machine Learning},
  pages={3610--3619},
  year={2021}
}

@article{nemirovski2009robust,
  title={Robust stochastic approximation approach to stochastic programming},
  author={Nemirovski, Arkadi and Juditsky, Anatoli and Lan, Guanghui and Shapiro, Alexander},
  journal={SIAM Journal on optimization},
  volume={19},
  number={4},
  pages={1574--1609},
  year={2009},
  publisher={SIAM}
}

@article{ghadimi2013stochastic,
  title={Stochastic first-and zeroth-order methods for nonconvex stochastic programming},
  author={Ghadimi, Saeed and Lan, Guanghui},
  journal={SIAM journal on optimization},
  volume={23},
  number={4},
  pages={2341--2368},
  year={2013},
  publisher={SIAM}
}

@inproceedings{gorbunov2022clipped,
  title={Clipped stochastic methods for variational inequalities with heavy-tailed noise},
  author={Gorbunov, Eduard and Danilova, Marina and Dobre, David and Dvurechenskii, Pavel and Gasnikov, Alexander and Gidel, Gauthier},
  booktitle={Advances in Neural Information Processing Systems},
  volume={35},
  pages={31319--31332},
  year={2022}
}

@inproceedings{gorbunov2023high,
  title={High-probability convergence for composite and distributed stochastic minimization and variational inequalities with heavy-tailed noise},
  author={Gorbunov, Eduard and Sadiev, Abdurakhmon and Danilova, Marina and Horv{\'a}th, Samuel and Gidel, Gauthier and Dvurechensky, Pavel and Gasnikov, Alexander and Richt{\'a}rik, Peter},
  booktitle={International Conference on Machine Learning},
  pages={15951--16070},
  year={2023}
}

@book{levy1925calcul,
  title={Calcul des probabilit{\'e}s},
  author={L{\'e}vy, Paul},
  year={1925},
  publisher={Gauthier-Villars}
}

@article{mandelbrot1960pareto,
  title={The {Pareto}-{Levy} law and the distribution of income},
  author={Mandelbrot, Benoit},
  journal={International Economic Review},
  volume={1},
  number={2},
  pages={79--106},
  year={1960},
  publisher={JSTOR}
}

@inproceedings{sadiev2023high,
  title={High-probability bounds for stochastic optimization and variational inequalities: the case of unbounded variance},
  author={Sadiev, Abdurakhmon and Danilova, Marina and Gorbunov, Eduard and Horv{\'a}th, Samuel and Gidel, Gauthier and Dvurechensky, Pavel and Gasnikov, Alexander and Richt{\'a}rik, Peter},
  booktitle={International Conference on Machine Learning},
  pages={29563--29648},
  year={2023}
}

@inproceedings{gorbunov2020stochastic,
  title={Stochastic optimization with heavy-tailed noise via accelerated gradient clipping},
  author={Gorbunov, Eduard and Danilova, Marina and Gasnikov, Alexander},
  booktitle={Advances in Neural Information Processing Systems},
  volume={33},
  pages={15042--15053},
  year={2020}
}

@article{gorbunov2024high,
  title={High-Probability Complexity Bounds for Non-smooth Stochastic Convex Optimization with Heavy-Tailed Noise},
  author={Gorbunov, Eduard and Danilova, Marina and Shibaev, Innokentiy and Dvurechensky, Pavel and Gasnikov, Alexander},
  journal={Journal of Optimization Theory and Applications},
  pages={1--60},
  year={2024},
  publisher={Springer}
}

@article{nguyen2023high,
  title={High probability convergence of clipped-{SGD} under heavy-tailed noise},
  author={Nguyen, Ta Duy and Nguyen, Thien Hang and Ene, Alina and Nguyen, Huy Le},
  journal={arXiv preprint arXiv:2302.05437},
  year={2023}
}

@inproceedings{cutkosky2020momentum,
  title={Momentum improves normalized {SGD}},
  author={Cutkosky, Ashok and Mehta, Harsh},
  booktitle={International Conference on Machine Learning},
  pages={2260--2268},
  year={2020}
}

@inproceedings{liu2023stochastic,
  title={Stochastic nonsmooth convex optimization with heavy-tailed noises: High-probability bound, in-expectation rate and initial distance adaptation},
  author={Liu, Zijian and Zhou, Zhengyuan},
  booktitle={International Conference on Machine Learning},
  pages={21884--21914},
  year={2023}
}

@inproceedings{li2021complexity,
  title={Complexity lower bounds for nonconvex-strongly-concave min-max optimization},
  author={Li, Haochuan and Tian, Yi and Zhang, Jingzhao and Jadbabaie, Ali},
  booktitle={Advances in Neural Information Processing Systems},
  volume={34},
  pages={1792--1804},
  year={2021}
}

@inproceedings{cutkosky2023optimal,
  title={Optimal stochastic non-smooth non-convex optimization through online-to-non-convex conversion},
  author={Cutkosky, Ashok and Mehta, Harsh and Orabona, Francesco},
  booktitle={International Conference on Machine Learning},
  pages={6643--6670},
  year={2023}
}

@article{liang2023lower,
  title={Lower bounds and accelerated algorithms for bilevel optimization},
  author={Ji, Kaiyi and Liang, Yingbin},
  journal={Journal of Machine Learning Research},
  volume={24},
  number={22},
  pages={1--56},
  year={2023}
}

@article{chen2025condition,
  title={On the Condition Number Dependency in Bilevel Optimization},
  author={Chen, Lesi and Zhang, Jingzhao},
  journal={arXiv preprint arXiv:2511.22331},
  year={2025}
}

@inproceedings{laguel2024high,
  title={High-probability complexity guarantees for nonconvex minimax problems},
  author={Laguel, Yassine and Syed, Yasa and Aybat, Necdet Serhat and G{\"u}rb{\"u}zbalaban, Mert},
  booktitle={Advances in Neural Information Processing System},
  volume={38},
  pages={140937--140969},
  year={2024}
}

@inproceedings{hubler2024gradient,
  title={From gradient clipping to normalization for heavy tailed {SGD}},
  author={H{\"u}bler, Florian and Fatkhullin, Ilyas and He, Niao},
  booktitle={International Conference on Artificial Intelligence and Statistics},
  pages={2413--2421},
  year={2024}
}

@inproceedings{liu2022bome,
  title={{BOME}! bilevel optimization made easy: A simple first-order approach},
  author={Liu, Bo and Ye, Mao and Wright, Stephen and Stone, Peter and Liu, Qiang},
  booktitle={Advances in Neural Information Processing Systems},
  volume={35},
  pages={17248--17262},
  year={2022}
}

@inproceedings{chen2024finding,
  title={On finding small hyper-gradients in bilevel optimization: Hardness results and improved analysis},
  author={Chen, Lesi and Xu, Jing and Zhang, Jingzhao},
  booktitle={Conference on Learning Theory},
  pages={947--980},
  year={2024}
}

@inproceedings{lin2020gradient,
  title={On gradient descent ascent for nonconvex-concave minimax problems},
  author={Lin, Tianyi and Jin, Chi and Jordan, Michael},
  booktitle={International Conference on Machine Learning},
  pages={6083--6093},
  year={2020}
}

@inproceedings{jin2020local,
  title={What is local optimality in nonconvex-nonconcave minimax optimization?},
  author={Jin, Chi and Netrapalli, Praneeth and Jordan, Michael I.},
  booktitle={International Conference on Machine Learning},
  pages={4880--4889},
  year={2020}
}

@inproceedings{nouiehed2019solving,
  title={Solving a class of non-convex min-max games using iterative first order methods},
  author={Nouiehed, Maher and Sanjabi, Maziar and Huang, Tianjian and Lee, Jason D. and Razaviyayn, Meisam},
  booktitle={Advances in Neural Information Processing Systems},
  volume={32},
  pages={14934--14942},
  year={2019}
}

@inproceedings{hao2024bilevel,
  title={Bilevel optimization under unbounded smoothness: A new algorithm and convergence analysis},
  author={Hao, Jie and Gong, Xiaochuan and Liu, Mingrui},
  booktitle={International Conference on Learning Representations},
  year={2024}
}

@article{gong2024accelerated,
  title={An accelerated algorithm for stochastic bilevel optimization under unbounded smoothness},
  author={Gong, Xiaochuan and Hao, Jie and Liu, Mingrui},
  journal={Advances in Neural Information Processing Systems},
  volume={37},
  pages={78201--78243},
  year={2024}
}

@inproceedings{gong2024nearly,
  title={A nearly optimal single loop algorithm for stochastic bilevel optimization under unbounded smoothness},
  author={Gong, Xiaochuan and Hao, Jie and Liu, Mingrui},
  booktitle={International Conference on Machine Learning},
  pages={15854--15892},
  year={2024}
}

@article{laguel2023high,
  title={High probability and risk-averse guarantees for stochastic saddle point problems},
  author={Laguel, Yassine and Aybat, Necdet Serhat and G{\"u}rb{\"u}zbalaban, Mert},
  journal={arXiv preprint arXiv:2304.00444},
  year={2023}
}

@inproceedings{grazzi2020iteration,
  title={On the iteration complexity of hypergradient computation},
  author={Grazzi, Riccardo and Franceschi, Luca and Pontil, Massimiliano and Salzo, Saverio},
  booktitle={International Conference on Machine Learning},
  pages={3748--3758},
  year={2020}
}

@article{radford2019language,
  title={Language models are unsupervised multitask learners},
  author={Radford, Alec and Wu, Jeffrey and Child, Rewon and Luan, David and Amodei, Dario and Sutskever, Ilya and others},
  journal={OpenAI Blog 1.8 },
  year={2019}
}

@misc{taori2023stanford,
  title={Stanford {Alpaca}: An instruction-following {LLaMa} model},
  author={Taori, Rohan and Gulrajani, Ishaan and Zhang, Tianyi and Dubois, Yann and Li, Xuechen and Guestrin, Carlos and Liang, Percy and Hashimoto, Tatsunori B.},
  year={2023},
  publisher={Stanford, CA, USA}
}

@inproceedings{tarzanagh2022fednest,
  title={{FedNest}: Federated bilevel, minimax, and compositional optimization},
  author={Tarzanagh, Davoud Ataee and Li, Mingchen and Thrampoulidis, Christos and Oymak, Samet},
  booktitle={International Conference on Machine Learning},
  pages={21146--21179},
  year={2022}
}

@inproceedings{chen2023decentralized,
  title={Decentralized stochastic bilevel optimization with improved per-iteration complexity},
  author={Chen, Xuxing and Huang, Minhui and Ma, Shiqian and Balasubramanian, Krishna},
  booktitle={International Conference on Machine Learning},
  pages={4641--4671},
  year={2023}
}

@book{durrett2019probability,
  title={Probability: theory and examples},
  author={Durrett, Rick},
  volume={49},
  year={2019},
  publisher={Cambridge University Press}
}

@article{li2020high,
  title={A high probability analysis of adaptive {SGD} with momentum},
  author={Li, Xiaoyu and Orabona, Francesco},
  journal={arXiv preprint arXiv:2007.14294},
  year={2020}
}

@article{bennett1962probability,
  title={Probability inequalities for the sum of independent random variables},
  author={Bennett, George},
  journal={Journal of the American Statistical Association},
  volume={57},
  number={297},
  pages={33--45},
  year={1962},
  publisher={Taylor \& Francis}
}

@article{dzhaparidze2001bernstein,
  title={On {Bernstein}-type inequalities for {Martingales}},
  author={Dzhaparidze, Kacha and Van Zanten, J.H.},
  journal={Stochastic Processes and Their Applications},
  volume={93},
  number={1},
  pages={109--117},
  year={2001},
  publisher={Elsevier}
}

@article{freedman1975tail,
  title={On tail probabilities for {Martingales}},
  author={Freedman, David A.},
  journal={the Annals of Probability},
  pages={100--118},
  year={1975},
  publisher={JSTOR}
}

@inproceedings{fradin2025tight,
  title={Tight Lower Bounds and Optimal Algorithms for Stochastic Nonconvex Optimization with Heavy-Tailed Noise},
  author={Fradin, Adrien and Sadiev, Abdurakhmon and Condat, Laurent and Richt{\'a}rik, Peter},
  booktitle={International Conference on Artificial Intelligence and Statistics},
  pages={2548--2556},
  year={2026}
}

@inproceedings{liunonconvex2025,
  title={Nonconvex Stochastic Optimization under Heavy-Tailed Noises: Optimal Convergence without Gradient Clipping},
  author={Liu, Zijian and Zhou, Zhengyuan},
  year={2025},
  booktitle={International Conference on Learning Representations}
}

@article{goodfellow2020generative,
  title={Generative adversarial networks},
  author={Goodfellow, Ian and Pouget-Abadie, Jean and Mirza, Mehdi and Xu, Bing and Warde-Farley, David and Ozair, Sherjil and Courville, Aaron and Bengio, Yoshua},
  journal={Communications of the ACM},
  volume={63},
  number={11},
  pages={139--144},
  year={2020},
  publisher={ACM New York, NY, USA}
}

@article{xu2009robustness,
  title={Robustness and Regularization of Support Vector Machines.},
  author={Xu, Huan and Caramanis, Constantine and Mannor, Shie},
  journal={Journal of Machine Learning Research},
  volume={10},
  number={7},
  year={2009}
}

@book{cesa2006prediction,
  title={Prediction, learning, and games},
  author={Cesa-Bianchi, Nicolo and Lugosi, G{\'a}bor},
  year={2006},
  publisher={Cambridge University Press}
}

@article{he2025complexity,
  title={Complexity of normalized stochastic first-order methods with momentum under heavy-tailed noise},
  author={He, Chuan and Lu, Zhaosong and Sun, Defeng and Deng, Zhanwang},
  journal={arXiv preprint arXiv:2506.11214},
  year={2025}
}

@article{huang2025efficiently,
  title={Efficiently escaping saddle points in bilevel optimization},
  author={Huang, Minhui and Chen, Xuxing and Ji, Kaiyi and Ma, Shiqian and Lai, Lifeng},
  journal={Journal of Machine Learning Research},
  volume={26},
  number={1},
  pages={1--61},
  year={2025}
}

@article{wang2024efficient,
  title={Efficient first order method for saddle point problems with higher order smoothness},
  author={Wang, Nuozhou and Zhang, Junyu and Zhang, Shuzhong},
  journal={SIAM Journal on Optimization},
  volume={34},
  number={4},
  pages={3342--3370},
  year={2024},
  publisher={SIAM}
}

@article{nichol2018first,
  title={On first-order meta-learning algorithms},
  author={Nichol, Alex and Achiam, Joshua and Schulman, John},
  journal={arXiv preprint arXiv:1803.02999},
  year={2018}
}

@article{sun2024gradient,
  title={Gradient normalization provably benefits nonconvex {SGD} under heavy-tailed noise},
  author={Sun, Tao and Liu, Xinwang and Yuan, Kun},
  journal={arXiv preprint arXiv:2410.16561},
  year={2024}
}

@inproceedings{zhanggeneralized,
  title={Generalized Smooth Bilevel Optimization with Nonconvex Lower-Level},
  author={Zhang, Siqi and Huang, Xing and Huang, Feihu},
  booktitle={International Conference on Machine Learning},
  year={2025}
}

@inproceedings{yang2022decentralized,
  title={Decentralized gossip-based stochastic bilevel optimization over communication networks},
  author={Yang, Shuoguang and Zhang, Xuezhou and Wang, Mengdi},
  booktitle={Advances in Neural Information Processing Systems},
  volume={35},
  pages={238--252},
  year={2022}
}

@inproceedings{sow2022convergence,
  title={On the convergence theory for {Hessian}-free bilevel algorithms},
  author={Sow, Daouda and Ji, Kaiyi and Liang, Yingbin},
  booktitle={Advances in Neural Information Processing Systems},
  volume={35},
  pages={4136--4149},
  year={2022}
}

@inproceedings{shafieezadeh2015distributionally,
  title={Distributionally robust logistic regression},
  author={Shafieezadeh Abadeh, Soroosh and Mohajerin Esfahani, Peyman M. and Kuhn, Daniel},
  booktitle={Advances in Neural Information Processing Systems},
  volume={28},
  pages={1576--1584},
  year={2015}
}

@inproceedings{huangoptimal,
  title={Optimal {Hessian}/{Jacobian}-Free Nonconvex-{PL} Bilevel Optimization},
  author={Huang, Feihu},
  booktitle={International Conference on Machine Learning},
  pages={19598--1962},
  year={2024}
}

@inproceedings{madry2017towards,
  title={Towards deep learning models resistant to adversarial attacks},
  author={Madry, Aleksander and Makelov, Aleksandar and Schmidt, Ludwig and Tsipras, Dimitris and Vladu, Adrian},
  booktitle={International Conference on Learning Representations},
  year={2018}
}

@inproceedings{song2019maml,
  title={{ES}-{MAML}: Simple {Hessian}-free meta learning},
  author={Song, Xingyou and Gao, Wenbo and Yang, Yuxiang and Choromanski, Krzysztof and Pacchiano, Aldo and Tang, Yunhao},
  booktitle={International Conference on Learning Representations},
  year={2020}
}

@inproceedings{yaoovercoming,
  title={Overcoming Lower-Level Constraints in Bilevel Optimization: A Novel Approach with Regularized Gap Functions},
  author={Yao, Wei and Yin, Haian and Zeng, Shangzhi and Zhang, Jin},
  booktitle={International Conference on Learning Representations},
 year={2025}
}

@article{arjevani2023lower,
  title={Lower bounds for non-convex stochastic optimization},
  author={Arjevani, Yossi and Carmon, Yair and Duchi, John C. and Foster, Dylan J. and Srebro, Nathan and Woodworth, Blake},
  journal={Mathematical Programming},
  volume={199},
  number={1},
  pages={165--214},
  year={2023},
  publisher={Springer}
}
\end{document}